\documentclass[11pt]{amsart}

\usepackage[margin=1in]{geometry}
\usepackage{amsmath,amssymb,amsfonts,amsthm,amsopn}
\usepackage{mathrsfs}
\usepackage{graphicx}
\usepackage{xcolor}
\usepackage{overpic}
\usepackage{subfig}
\usepackage{adjustbox}
\usepackage{epstopdf}
\usepackage{enumitem}
\usepackage[font=small]{caption}
\usepackage{multirow}
\usepackage{siunitx}
\usepackage{algorithmic}
\usepackage[hidelinks]{hyperref}
\usepackage[capitalize,nameinlink]{cleveref}

\DeclareGraphicsExtensions{.pdf,.png,.jpg,.eps}

\numberwithin{equation}{section}

\theoremstyle{plain}
\newtheorem{theorem}{Theorem}[section]
\newtheorem{lemma}[theorem]{Lemma}
\newtheorem{proposition}[theorem]{Proposition}

\newtheorem{assumption}{Assumption}
\theoremstyle{definition}

\theoremstyle{remark}
\newtheorem{remark}{Remark}[section]

\newtheorem*{example*}{Example}

\crefname{hypothesis}{Hypothesis}{Hypotheses}
\crefname{fact}{Fact}{Facts}
\crefname{remark}{remark}{remarks}
\crefname{theorem}{theorem}{theorems}
\crefname{lemma}{lemma}{lemmas}
\crefname{assumption}{assumption}{assumptions}
\crefname{proposition}{proposition}{propositions}
\crefname{definition}{definition}{definitions}

\newcommand{\R}{\mathbb{R}}

\newcommand{\test}{{\text{test}}}

\renewcommand{\P}{\mathscr{P}}
\newcommand{\PP}{\mathcal{P}}
\newcommand{\PK}{\mathsf{P}}

\newcommand{\FF}{\mathcal{F}}

\newcommand{\HK}{\mathsf{H}}
\newcommand{\mS}{\mathcal{S}}
\newcommand{\mH}{\mathcal{H}}

\newcommand{\B}{\mathscr{B}}

\newcommand{\JJ}{\mathcal{J}}

\newcommand{\UU}{\mathcal{U}}
\newcommand{\UK}{\mathsf{U}}
\newcommand{\mY}{\mathcal{Y}}

\newcommand{\GG}{\mathcal{G}}

\newcommand{\wh}[1]{\widehat{#1}}

\newcommand{\whu}{\wh{u}}
\newcommand{\whP}{\wh{P}}

\newcommand{\mbf}[1]{\mathbf{#1}}

\DeclareMathOperator*{\argmin}{arg\,min}

\newcommand{\mR}{\mathcal{R}}

\newcommand{\editstart}{}
\newcommand{\editend}{}

\allowdisplaybreaks

\title[Data-Efficient Kernel Methods for Equation Learning]{Data-Efficient Kernel Methods for Learning
 Differential Equations and Their Solution Operators:
 Algorithms and Error Analysis}

\author[B. Hosseini]{Bamdad Hosseini}
\address{Department of Applied Mathematics, University of Washington, Seattle, WA, 98195}
\email{bamdadh@uw.edu}

\author[A. Hsu]{Alexander Hsu}
\address{Department of Applied Mathematics, University of Washington, Seattle, WA, 98195}
\email{owlx@uw.edu}

\author[Y. Jalalian]{Yasamin Jalalian}
\address{Department of Computing and Mathematical Sciences, California Institute of Technology, Pasadena, CA, 91125}
\email{yjalalia@caltech.edu}

\author[J.F. Osorio Ramirez]{Juan Felipe Osorio Ramirez}
\address{Department of Applied Mathematics, University of Washington, Seattle, WA, 98195}
\email{josorior@uw.edu}

\author[H. Owhadi]{Houman Owhadi}
\address{Department of Computing and Mathematical Sciences, California Institute of Technology, Pasadena, CA, 91125}
\email{owhadi@caltech.edu}

\thanks{BH, AH and JFOR were supported by the National Science Foundations grants
DMS-2208535 "Machine Learning for Inverse Problems" and DMS-2337678 "CAREER: Gaussian Processes for Scientific Machine Learning: Theoretical Analysis and Computational Algorithms".
YJ and HO acknowledge support from the Air Force Office of Scientific Research under MURI awards number FA9550-20-1-0358 (Machine Learning and Physics-Based Modeling and Simulation), FOA-AFRL-AFOSR-2023-0004 (Mathematics of Digital Twins), the Department of Energy under award number DE-SC0023163 (SEA-CROGS: Scalable, Efficient, and Accelerated Causal Reasoning Operators, Graphs and Spikes for Earth and Embedded Systems), the National Science Foundations under award number 2425909 (Discovering the Law of Stress Transfer and Earthquake Dynamics in a Fault Network using a Computational Graph Discovery Approach), and the Jet Propulsion Laboratory
PDRDF 24AW0133 (Jupiter’s Radiation Environment: Assimilating Data with ML-driven Approaches). Additionally, HO acknowledges support from the DoD Vannevar Bush Faculty Fellowship Program (ONR N000142512035).}



\hypersetup{
  pdftitle={Data-Efficient Kernel Methods for Learning Differential Equations and Their Solution Operators: Algorithms and Error Analysis},
  pdfauthor={B. Hosseini, A. Hsu, Y. Jalalian, J.F. Osorio Ramirez, H. Owhadi}
}

\begin{document}

\begin{abstract}
We introduce a novel kernel-based framework for learning differential equations and their solution maps. 
\editstart
The framework is efficient in its data requirements in terms of the number of observed states and the amount of data available per state. Our formulation also allows us to employ efficient Gauss-Newton-type optimization algorithms in the training procedure to improve 
accuracy and speed of the training procedure. 
We also support our algorithm with rigorous theoretical guarantees in the form of quantitative worst-case error bounds for the learned equation that can be extended to the subsequent solution operators under 
smoothness assumptions. 
\editend
Numerical benchmarks demonstrate significant improvements in computational complexity and robustness while  achieving  one to two orders of magnitude improvements in terms of accuracy compared to 
\editstart
certain 
\editend
state-of-the-art algorithms. 
% In comparison 
% to equivalent neural net methods, our approach is significantly more 
% robust to the choice of hyperparameters and does not require close 
% human supervision during training. 
\end{abstract}

\maketitle

\section{Introduction}\label{sec:introduction}

In recent years, machine learning (ML) has revolutionized the way data is combined with mathematical models to infer and predict the behavior of physical systems. This wide adoption of ML in science has given rise to a new area of computational 
science and engineering often referred to as physics-informed 
ML (PIML) \cite{karniadakis2021physics} 
\editstart
or more broadly, scientific machine learning (SciML) \cite{cuomo2022scientific}.
\editend
The goal here is to infer and simulate physical processes
driven by differential equations (DEs) by combining data and 
expert knowledge in an automated manner.
In this article we focus on three core problems of PIML following 
\cite{BOULLE202483}, i.e., 
the problems of learning or discovering DEs~\cite{brunton2016discovering, rudy2017data}; 
learning or approximating their solution operators~\cite{KOVACHKI2024419, batlle2023kernel};
and finally solving DEs~\cite{raissi2019physics, chen2021solving}. 
All three of these problems have a long history in control, scientific computing, 
and numerical analysis but the wide adoption of ML has given rise to new 
ideas and techniques for solving them and in particular blurred the 
lines between these three problems. In this article  we 
are primarily focused on solving the above problems in very scarce data 
regimes and we shall introduce a simple, accurate, robust, and efficient algorithm 
based on the theory of Reproducing Kernel Hilbert Spaces (RKHSs) and Gaussian Processes (GPs) leading to the unification of the three problems of PIML under the same umbrella. 
\editstart 
We emphasize that some of the  core ideas outlined here, especially those 
pertaining to  joint optimization algorithms, 
can go beyond the case of RKHS/GP models and be 
adapted to existing methods in the literature such as neural net or dictionary 
based methods; see for example \cite{chen2021physics,hsu2025joint}.
\editend

\subsection{Problem context and statement}\label{sec:problem-statement}
Consider a bounded open set $\mY \subset \R^d$.
For a function $u: \mY \to \R$ which 
describes the state of a physical system, we consider 
partial differential equations (PDEs)
of the form 
\begin{equation}\label{generic-DE}
    \P(u)(y)  = f(y) \qquad y \in \mY, \qquad \B(u)(y) = g(y) \qquad y \in \partial \mY,
\end{equation}
where $\P: \UU \to \FF$ denotes a differential operator that describes the PDE in  $\mY$, $\B: \UU \to \GG$ denotes the boundary operator, $f$ and $g$ denote the source term and boundary data/initial conditions, and 
$\UU, \FF, \GG$ denote appropriate function spaces for $u, f, g$.
We emphasize that 
in the above formulation, $y$ is considered as an abstract input variable 
that may be spatial only (in the case of steady state PDEs) or a space-time 
variable (in the case of dynamics). An important feature of PIML is the seamless fusion of noisy observations and prior knowledge about the physical process of interest, thereby unifying numerical 
PDEs with inverse problems. 

\editstart
To this end, for some integer $M$, consider tuples $(u^m, f^m, g^m)_{m=1}^M$ that satisfy \eqref{generic-DE}, 
as well as known observation operators $\varphi^m_\UU: \UU \to \R^{d^m_\UU}, \varphi^m_\FF: \FF \to \R^{d^m_\FF}, 
\varphi^m_\GG: \GG \to \R^{d^m_\GG}$; note that $m$ is an index and not an exponent.
Then, we assume access to a data set of the form 
$(\varphi_\UU(u^m) + \varepsilon^m_\UU, \varphi_\FF(f^m) + \varepsilon^m_\FF, \varphi_\GG(g^m) + \varepsilon^m_\GG)_{m=1}^M$,
where $\varepsilon^m_\UU, \varepsilon^m_\FF, \varepsilon^m_\GG$ represent 
possible observation noises. 
With this data set at hand, 
we define the following three core problems of PIML:
\begin{enumerate}[label=(P\arabic*)]
\item {\it PDE solvers:} Suppose $\P, \B$ are known and compute $u$ for a given $(f,g)$. In this case there are no measurement operators $\varphi$;
\item {\it Equation learning or discovery:} Infer the differential operator $\P: \UU \to \FF$ and/or $\B: \UU \to \GG$ from data. Typically, the measurement operators $\varphi$ are in the form of 
pointwise observations;
\item {\it Operator learning:} Infer the solution map
$(\P, \B)^{-1}: \FF \times \GG \to \UU$ from data \footnote{More broadly we may also consider parametric PDEs in which case 
the solution map takes as input an abstract parameter alongside 
the source or boundary data.}. 
Typically the $\varphi$ are either pointwise observations or coefficients 
over a basis such as Fourier or principal components.
\end{enumerate}
The above problems can be viewed as three related formulations of the same inverse problem:
inferring, from varying amounts of data and prior structure, the map that connects inputs, governing 
laws, and solution fields. They differ primarily in which components are assumed known and which must 
be learned. PDE solvers (P1) operate in the small-data with high-physics regime: 
the differential operator, domain, boundary and initial conditions are specified, and the task is to approximate the corresponding solution. 
PDE discovery (P2) lies in an intermediate regime, where partial physical 
structure is available but the governing operator, constitutive terms, coefficients, or closures must 
be inferred from observations. Operator learning (P3) is most relevant in the high-data with 
low-physics regime, where one seeks to learn the solution operator 
directly from examples, often without 
explicitly identifying the underlying PDE (see figure 1 of \cite{BOULLE202483}).
In particular, (P2) and (P3) are both problems of inferring operators from data, 
just with inputs and outputs reversed.
% between different spaces. 
At the same time, (P1) and (P3) target the same output,
% can be viewed as the same 
since computing 
$u$ amounts to evaluating $\P^{-1}$. However,
PDE solvers assume $\P$ is known without any observed data on the state 
while operator learning often does not utilize $\P$ directly but works with observed data. Hence, the two problems
arise in different target applications.

We primarily focus on problems in (P2) where $\P$ is unknown and further assume
Dirichlet boundary conditions so that $\B$
is the identity map. 
To be precise, we consider differential operators
of the form 
\begin{equation}\label{form-of-P}
    \P(u)(y) = P \circ \Phi(u, y),
    \quad \Phi: (u, y) \mapsto 
    (y,  \delta_y \circ L_1 u , \dots,  \delta_y \circ L_Q u ) \in \R^{Q + d},
\end{equation}
and write $\mS = \R^{Q + d}$ for brevity. 
In the above  $P: \mS \to \R$ is a possibly nonlinear function while 
$\delta_y$ denotes the pointwise evaluation functional at $y$,
and $( L_q )_{q=1}^{Q}$ are bounded  and linear differential operators that we assume are known
or chosen by the user \footnote{Note that with this setup $\Phi$ is linear in $u$
but it is generally nonlinear in $y$ since $u$ may be nonlinear.}.
For ease of presentation we always assume $L_1 = \text{Id}$ so that $P$
takes $y$ and the point values $u(y)$ as input even though it may 
not depend on these quantities. Beyond this assumption, we do not pose any 
restrictions on the $L_q$'s.

\begin{example*}
The above assumptions on the form of $\P$ 
encompass most scalar PDEs of interest in physics and engineering. 
For example consider the nonlinear PDE:
\begin{equation}\label{nonlinear-heat-eqn}
\left\{
\begin{aligned}
   & \P(u) (t,x) = \partial_t u(t, x) -  \partial_{x}\left( a(x) \partial_x u (t,x)\right) - u^3(t, x) =  f(t,x),  \\ & \hspace{0.55 \textwidth}  \text{ for } (t, x) \in (0, T] \times (0,1),\\
& u(0, x) = u(t, 0) = u(t, 1)= 0, 
\end{aligned}
\right.
\end{equation}
with a smooth coefficient $a: [0,1] \to (0, + \infty)$.
Writing $y = (t,x)$ and letting $\mY = (0,T] \times (0,1)$
and introducing the differential operators
$L_1: u \mapsto u$, $L_2: u \mapsto \partial_t u$, 
$L_3: u \mapsto \partial_{x} u$, and $L_4: u \mapsto \partial_{x}^2 u$,
we can cast \eqref{nonlinear-heat-eqn} in the form \eqref{form-of-P}
by writing $\P(u)(y) = 
\editstart
P \circ \Phi(u, (t,x)) = 
\editend
    \,P( t, x,  u(t,x),  L_2 u (t, x),  
    L_3 u (t, x),  L_4 u (t,x) )$ 
    with the nonlinear map
\begin{equation*}
P: \R^6 \to \R, \qquad 
    P(s_1, s_2, s_3, s_4, s_5, s_6) = 
    s_4 -\partial_x a(s_2) s_5 - a(s_2) s_6 - s_3^3.
\end{equation*}
We emphasize that the example PDE above is precisely the 
type of equation that we have in mind, i.e., 
nonlinear equations with variable coefficients that 
may not have sparse representations in a known basis.
Additionally, if $a$ is known we can directly introduce the operators 
$L_1: u \mapsto u$, $L_2: u \mapsto \partial_t u - \partial_x a \partial_x u$ leading to the much simpler
nonlinear map  $P(s_1, s_2, s_3) = s_3 - s_2^3$.
\end{example*}

With the above assumptions we consider the {\it equation learning problem} of estimating $P$
with some function $\whP$ from a data set of the form 
$(u^m(Y^m), f^m, g^m)_{m=1}^M$ where 
we introduced the shorthand notation  $u^m(Y^m):= (u^m(y_1^m), \dots, u^m(y_N^m)) \in \R^N$
alongside the known sets of points
$Y^m = \{ y_1^m, \dots, y_N^m \} \subset \overline{\mY}$ for some integer $N \ge 1$, that we refer to as the {\it observation points}.
In other words, we take the earlier observation operators $\varphi^m_\UU$ to be  the evaluation maps on the $Y^m$. 
Also note that we assume the source and boundary data $f^m,g^m$ are known everywhere although 
our algorithms only rely on evaluations of these functions on a finite grid.
Along with the equation learning problem, we also consider the {\it state reconstruction problem}
of estimating the true $u^m$ with approximations $\whu^m$.
\Cref{fig:equation-learning-cartoon} shows a cartoon of our setup where the red elements are the unknowns of interest, 
the blue elements are the data, and the  black elements are assumed to be known.

% from a data set of the form 
% More precisely, 
% consider sets of points 
% $Y^m = \{ y_1^m, \dots, y_N^m \} \subset \mY$,
% which we refer to as the {\it observation points}
% \footnote{One may also let $N$ change with the index $m$ but we keep the size of the mesh fixed to keep our notation light.}, for some integer $N \ge 1$.
% Further introduce the compact notation
% $u^m(Y^m):= (u^m(y_1^m), \dots, u^m(y_N^m)) \in \R^N$. 
% Then our goal throughout the article is to learn $P$  from the 
% training data  $( u^m(Y^m), f^m)_{m=1}^M$, i.e., 
% $u^m$ is only observed on the $Y^m$ while $f^m$ is assumed to be 
% known everywhere; this is  a simplifying assumption for us and can be 
% relaxed to having finite information on the $f^m$ under some 
% circumstances but we will leave this as a future direction.
% \Cref{fig:computational-graphs}(a) shows a cartoon of our 
% problem where the red elements are the unknowns of interest, 
% the blue elements are the data, and the  black elements are assumed to be known. 

\begin{figure}[htp]
    \centering
    \raisebox{0.1 \height}{
    \begin{overpic}[width=0.15\linewidth]{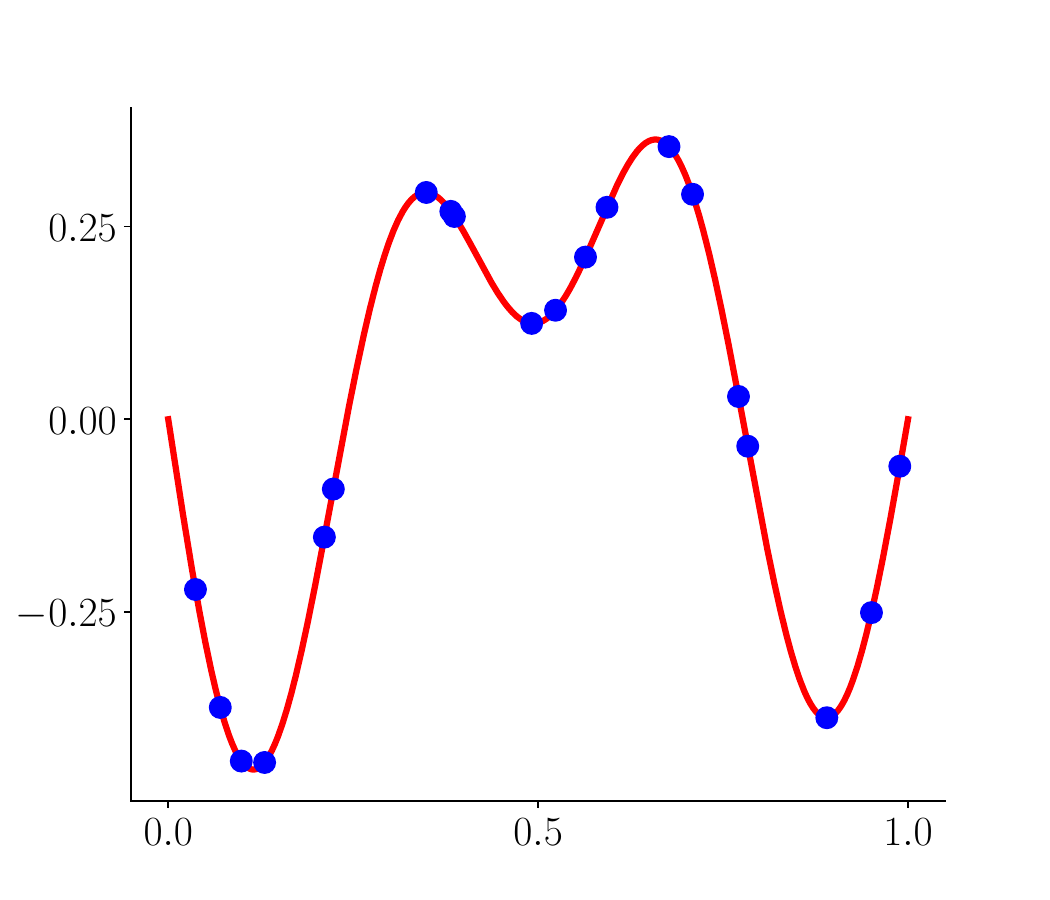}
        %\put(-15, 70){{\large \color{red} $u$}}
        \put(47, -12){$\mY$}
        \put(120, 40){\huge $\xrightarrow{ \: \: \Phi \: \:}$}
        \put(-30,42){\large \color{red} $u^m$}
    \end{overpic}
    }
    \hspace{0.1\textwidth}
    \begin{overpic}[width=0.25\linewidth]{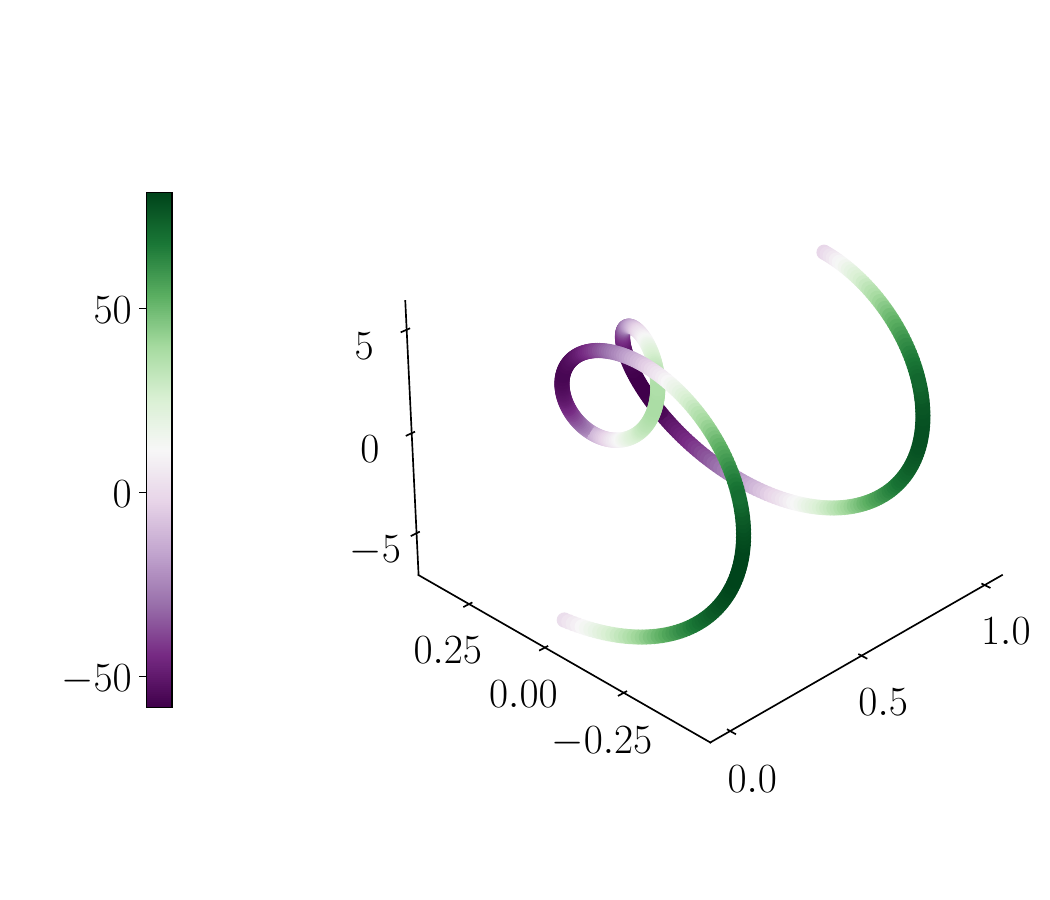}
        \put(88,12){\footnotesize $y$}
        \put(38,12){\footnotesize $u$}
        \put(20,42){\footnotesize $\partial_y u$}
        \put(7,10){\footnotesize $\partial_{y}^2 u$}
        \put(66, 0){$\mS$}
    \end{overpic}
    \hspace{0.1\textwidth}
    \raisebox{0.1 \height}{
    \begin{overpic}[align=c, width=0.15\linewidth]{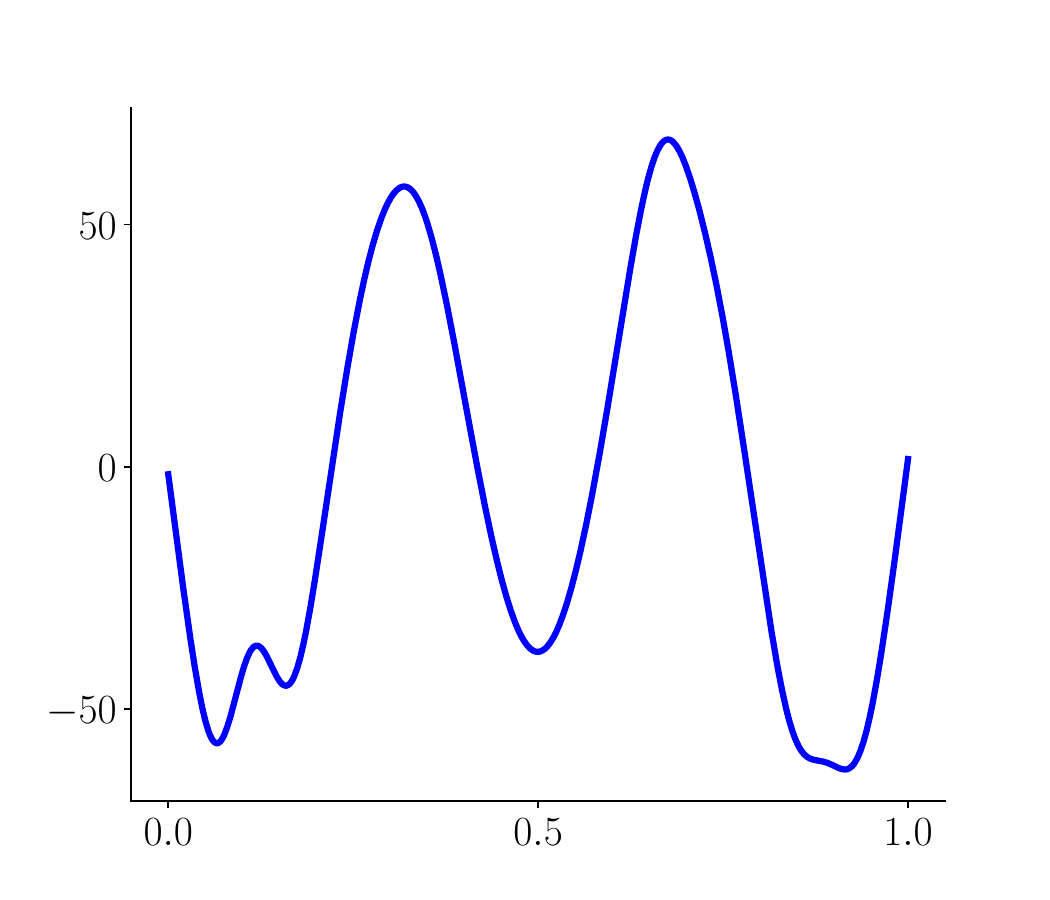}
        \put(-70,40){\color{red} \huge $\xrightarrow{\: \: P \: \:}$}
        \put(93, 42){{\Large \color{blue} $f^m$}}
        %\put(50, -3){\footnotesize $y$}
        \put(49, -8){$\mY$}
    \end{overpic}
    }
    \caption{
    Schematic depiction of our equation learning problem. 
    Red elements are the unknowns of interest, blue objects are the given data. The left and right 
    panels show the states $u^m$ and corresponding right-hand sides $f^m$ for an example second order PDE 
    depending on $y, u, \partial_y u$, and $\partial_{y}^2 u$ while the  middle panel 
    shows $\Phi(u, y)$. 
    }
    \label{fig:equation-learning-cartoon}
\end{figure}

\editend

 \subsection{Summary of contributions}\label{sec:contributions}
\editstart
Our main contribution in this article is to propose a computational framework based on the theory of RKHSs and GPs for
the equation learning and state reconstruction  problems of finding $\whP \approx P$ and  $\whu^m \approx u^m$,
called Kernel Equation Learning (KEqL). Two ideas are central to our method:
\begin{enumerate}
    
   \item \emph{Leveraging the fact that $(u^m, f^m, g^m)$ 
   collectively satisfy the PDE  
    defined by $\P$ improves the accuracy of
    equation learning in scarce data regimes.}
    
    \item  
    \editstart
    \emph{Under our assumptions, learning the function $P$ in scarce data regimes appears to be more
    data-efficient than directly learning the solution map $\P^{-1}$. This allows us to 
    approximate $\P^{-1}$ by ``inverting" the learned PDE, a method that 
    we refer to as equation-mediated operator learning.}
 \editend 
 % i.e., operator learning through equation learning 
 % is more effective.} \textcolor{red}{This is coherent with \cite{BOULLE202483}, and it actually exactly places our operator learning somewhere between equation learning and operator learning in the spectrum but to claim this, shouldn't we 1. in experiments, have a pure operator learning perspective and show in scarce data it doesn't work as well as our equation-mediated or 2. show that in theory, this improves the rates? Do we show any of these?}
\end{enumerate}
\editstart
Clearly, the above points rely on prior knowledge about the underlying PDE, i.e., the data comes from an underlying PDE and that PDE is local. In this light,
the performance of our method is tied to taking advantage of this additional prior knowledge. 
% At the same time 
When these assumptions are violated, the method can fail, for 
example, if $\P$ involves integral operators. 
We note that while operator learning is often formulated as learning general maps between function spaces, a substantial portion of the existing work on operator learning specifically seeks to approximate the solution maps of PDEs, with both method development and theoretical analyses paying special attention to these important cases \cite{li2020fourier,lanthaler2022deeponeterrors,boulle2023elliptic,boulle2022learning,marcati2023exponential}.

Our method, described in \Cref{sec:proposed-method},
relies on minimizing an objective function
% variational loss 
consisting of a PDE residual term and a data misfit term
to impose the constraint that $\whP\circ \Phi(\whu^m, \cdot) =f^m$ while matching the observed data, which is precisely the core idea (1) above. 
An interesting feature of our %method 
approach is that it naturally contains 
PDE solvers (P1) in the sense that if we assume $P$ is given and hence we do not need to estimate $\whP$, then 
our method for computing $\whu^m$ reduces to a %collocation PDE solver.
collocation method for solving PDEs. Building on this idea, our variational formulation
also allows us to define a notion of a {\it pseudo inverse} for the learned differential operator 
$\widehat{\P} = \whP \circ \Phi$ as an approximation to $\P^{-1}$. We consider this as an implicit 
operator learning method which we refer to as equation-mediated operator learning per core idea (2) above.
We highlight that this method does not necessarily lead to a fast emulator, as is often the 
case for other operator learning methods, but rather has better accuracy per our empirical results in \Cref{sec:numerics}.

Throughout the paper we:
(a) show that our method is accurate and data efficient, achieving superior performance across multiple benchmarks;
(b) give rigorous theoretical guarantees, providing a quantitative worst-case error bound for equation learning that can be extended to operator learning errors;
and (c) demonstrate that our method unlocks new capabilities in PIML, enabling one-shot operator learning and the discovery of variable-coefficient PDEs in severely data-scarce settings.

\editend

 \subsection{Review of the relevant literature}\label{sec:literature-review}
\editstart
We now give a brief review of the related works in the context of problems 
(P1)-(P3)
and connect them to previous works in control, scientific computing, and 
applied mathematics. 
% Readers who are primarily interested in our methodological 
% developments may safely skip to  \Cref{sec:proposed-method}.
\editend

\subsubsection{
\editstart
Equation learning and system identification
\editend
}\label{sec:eql-review}

The discovery or learning of DEs 
from data was brought into prominence by
\cite{bongard2007automated, schmidt2009distilling} that 
used symbolic regression to discover physical laws, although  
learning  dynamical 
systems from time-series data has a long history going back to the 70s in {\it system identification}
 \cite{aastrom1971system,ljung2010perspectives,keesman2011system}.
Modern iterations of equation learning are often focused 
on learning dynamical systems and DEs 
under a sparsity prior over a dictionary \cite{brunton2016discovering, rudy2017data, schaeffer2017learning}
among which 
the Sparse Identification of Nonlinear Dynamics (SINDy) \cite{brunton2016discovering} is perhaps 
the most widely known instance, although many other examples exist \cite{schaeffer2017learning,kang2021ident}.
Since these contributions, many extensions of the sparse 
regression framework  have been proposed such as weak-SINDy \cite{messenger2021weak}, that improves accuracy and robustness by 
working with weak-form DEs, ensemble methods \cite{fasel2022ensemble} that 
combine multiple models to improve accuracy
and robustness, and Bayesian methods \cite{zhang2018robust, niven2019bayesian, yang2020bayesian, hirsh2022sparsifying, north2022bayesian, north2023bayesian} that enable uncertainty 
quantification; see  \cite{north2023review} 
and references within for a more detailed review. 

The  majority of the methods
mentioned above belong to a class of methods which we  refer to 
as "2-step methods" in \Cref{sec:proposed-method}. These are 
methods where the state variables $u^m$ are first
estimated from the observations $\varphi^m(u^m)$ before being used to infer $\P$.
As we will show in \Cref{sec:numerics}, these methods are successful only in high data regimes and 
can degrade substantially in scarce data regimes. To overcome this hurdle we propose a so-called "1-step method" 
\editstart
where the $u^m$ are estimated simultaneously alongside $\P$.
\editend
1-step methods for sparse regression of DEs are less developed. To our 
knowledge these ideas were  explored in \cite{ijcai2021p283} for ODEs 
and \cite{chen2021physics} for PDEs,
although methods such as weak-SINDy \cite{messenger2021weak} and UQ-SINDy \cite{hirsh2022sparsifying} can be formulated as 1-step methods as well. Similarly, joint optimization methods based on least squares collocation formulations using splines \cite{ijcai2021p283} and neural networks \cite{chen2021physics} have been proposed to extend the SINDy framework towards 1-step methods.

Beyond sparse regression models, RKHS or GP methods have been widely developed 
for solving PDEs \cite{owhadi2015bayesian,chkrebtii2016bayesian,cockayne2016probabilistic,raissi2018numerical,owhadi2019statistical,owhadi2019operator,cockayne2019bayesian,wang2021bayesian,chen2021solving} 
but they have recently been extended for learning equations, most notably \cite{raissi2017machine,lahouel2024learning} introduced 1-step 
methods for learning linear DEs and ODEs while \cite{Raissi_2018} used GPs to learn 
the flow field of a time-discretized PDE. The series of papers
\cite{hamzi-I,hamzi-II,hamzi-III,hamzi-IV, hamzi-V, hamzi-VI} also 
introduced GPs for  
inference of ODEs  towards data assimilation. More recently,
\cite{rosenfeld2024occupation} introduced a kernel analogue of 
SINDy and weak-SINDy  while 
\cite{wells2024stochastic, darcy2023one} aimed at learning stochastic differential equations. Finally,
\cite{LONG2024134095} introduced a simple 2-step kernel method 
for learning ODEs and PDEs with unknown 
coefficients, inspiring our present paper.

\editstart
Beyond dictionary-based and kernel methods, 
various other models have been proposed for equation learning tasks. 
These methods range from symbolic regression \cite{long2018pde, long2019pde, xu2021dl, xu2020dlga, xu2021deep} to neural net regression 
\cite{raissi2018deep, both2021deepmod, chen2021physics, koenig2024kan, stephany2024pde} 
with the latter class  being largely 
based on the methodology of physics-informed neural nets (PINNs) \cite{raissi2019physics}.
The works \cite{xie2019non, raissi2018multistep, tipireddy2019comparative, keller2021discovery, du2022discovery} utilize linear  multi-step methods (LMMs) for discretization of ODEs to formulate 
supervised learning problems for discovery of dynamics. 
Although these works largely focus on neural net parameterizations, the underlying 
formulations are more general and can be extended to other models.
Finally, the recent works \cite{karnakov2024solving, karnakov2022optimizing} 
solve equation learning problems in the guise of inverse problems 
by directly discretizing a variational objective function using classic numerical methods before optimizing using automatic differentiation.
\editend

% \textcolor{red}{

% \textcolor{red}{YJ to BH: there is also PDE-Net and PDE-LEARN: Using Deep Learning to Discover Partial Differential Equations from Noisy, Limited Data by
% Robert Stephany, Christopher Earls.}

\subsubsection{
Theory of equation learning
}\label{sec:eql-theory-review}

The problem of learning dynamical systems has a vast and
mature theoretical foundation \cite{berry2025limits}. The theory of equation learning for DEs, on the other hand, is recent and scarce in comparison. 
The convergence properties of sparse regression methods  (e.g. SINDy) were studied in  \cite{schaeffer2017learning, zhang2019convergence,schaeffer2020extracting,kang2021ident} using
 compressed sensing and random matrix theory to provide
 recovery guarantees.
 \editstart
 The article \cite{russo2024convergence} outlines a detailed theoretical treatment of 
 the weak-SINDy method showing that the dynamics identified by weak-SINDy is precisely a 
 projection of the true dynamics in a Hilbert space spanned by the feature maps 
 that define the dictionary and gives error bounds and convergence results for ODE models
 both for the learned equation as well as the solution trajectories. 
 The paper further unifies weak-SINDy with the standard SINDy algorithm as well as 
 the occupation kernel approach of \cite{rosenfeld2024occupation}. However, the 
 analysis of \cite{russo2024convergence} is primarily focused on ODEs and only covers 
 PDEs after they are reduced to ODEs via a dimension reduction map such as POD.
 The articles \cite{du2022discovery,keller2021discovery}  present 
 analysis for the LMM approach for learning ODEs including quantitative 
 error bounds, convergence rates, and stability.
 
%  proves that weak-SINDy is indeed a projection operator onto the span of the library inside a Hilbert function space, and they show that naive SINDy, weak-SINDy, and the occupation-kernel method are all the same encoder-solver-decoder object, differing only in the choice of test functionals (with the reproducing-kernel choice recovering SINDy). Thus the "dictionary system-ID as projection in an RKHS" viewpoint is already established.

%  The articles \cite{scholl2023uniqueness,Scholl2026} study
% the identifiability of equation learning for ODEs and PDEs while \cite{he2024much} considers
% the limitations of learning a PDE from a single trajectory.
% \editstart
% Finally, the works \cite{du2022discovery,keller2021discovery} present quantitative error bounds 
% for the LMM equation learning methods in idealized settings in terms of the data resolution for ODEs.

\editend

\subsubsection{
Inverse problems and data assimilation
}
 Identifying  the parameters  of a DE 
 is the primary focus of the field of inverse problems \cite{tarantola2005inverse,kaipio2006statistical}
 with a rich history of both theoretical analysis \cite{engl1996regularization,kabanikhin2011inverse} and 
 computational methodologies \cite{vogel2002computational}.
 While traditional inverse problems focus on known PDEs with 
 unknown spatio-temporal coefficients, the methodologies developed 
 for those problems can be extended to equation learning in both 
 1-step and 2-step methodologies. 
 Many of the ideas that we 
 develop in the current paper were inspired by the inverse problems literature. Most
  notably joint inversion 
\cite{haber1997joint} and  all-at-once inversion 
\cite{Kaltenbacher2016-zrip, Kaltenbacher_2017}.

The problem of recovering the functions $u^m$ and in turn 
predicting the solution of a dynamic DE in future times falls within 
the fields of data assimilation \cite{law2015data,reich2015probabilistic},
filtering \cite{bar2004estimation,xiong2008introduction}, and control 
\cite{glad2018control,franklin2002feedback}.
The problem of filtering the state of a control system while 
identifying its unknown parameters  
is also classical in filtering \cite{voss2004nonlinear,schoukens2019nonlinear} 
and can be solved using classic techniques such as extended Kalman smoothing \cite[Sec.~5]{sarkka2023bayesian}, 
expectation maximization methods \cite{ghahramani1998learning,schon2011system},
dual extended Kalman filters \cite{wan2001dual},
GP dynamical models \cite{wang2005gaussian,frigola2014variational},
 and  sequential Monte Carlo \cite{liu2001combined}.

\subsubsection{
Operator learning
}\label{sec:opl-review}
Operator learning has become very popular in 
recent years, especially after the seminal 
papers \cite{bhattacharya2021model,li2020fourier, lu2021learning}
where neural net models were developed for the approximation of 
mappings between Banach spaces. Since then, a large body of work has been developed around operator learning focusing on methodologies as well as
theory; see \cite{KOVACHKI2024419} and references within
for neural operators and \cite{batlle2023kernel} for the kernel
perspective.
While modern operator learning is largely focused 
on neural net models this problem has appeared previously in the fields of computer model 
emulation \cite{kennedy2001bayesian}, 
polynomial chaos \cite{xiu2002wiener},
stochastic Galerkin methods \cite{ghanem2003stochastic}, 
 reduced order modeling \cite{lucia2004reduced},
and model order reduction \cite{schilders2008model}.  
\editstart
Error analysis for operator learning is also a growing research direction
\cite{kovachki2024operator, schafer2024sparse, boulle2023elliptic, de2023convergence, batlle2023kernel}.
\editend
% We also mention the works \cite{schafer2021compression,boulle2022learning,boulle2023elliptic,wang2023operator,schafer2024sparse} that consider the operator learning of linear PDEs 
% with quantitative approximation rates.
We consider a different perspective towards operator learning. 
Instead of learning $\P^{-1}$ as a high dimensional mapping, we 
 learn $\P$ as a local differential operator and then 
invert it numerically to evaluate $\P^{-1}$.
To our knowledge, this approach towards operator learning 
is largely unexplored with the exception of  \cite{de2022generic,goswami2023physics,li2024physics}
that use the underlying DE to enhance neural 
operator training. 
\editstart
An important distinction here is that our method only implicitly identifies approximations to $\P^{-1}$ through a 
variational problem as opposed to a static map that is often obtained by neural operators.
\editend

% While this approach does not help for emulating expensive numerical solvers, we will show in \Cref{sec:numerics} that our approach 
% towards operator learning gives significant improvements in accuracy 
% when data is scarce. In fact, we can perform one-shot 
% operator learning, i.e., when $M=1$ with reasonable accuracy.

% which deliberately utilizes the fact that the differential operator $\P$
% is local, is far from  the dominant approach in the aforementioned 
% works. However, this idea has been investigated in the 
% context of physics-informed operators in recent years .
% These models train a neural net to learn the solution map of a PDE 
% with an additional term in the training loss that minimizes the 
% residual of the PDE for the predicted solutions on a collocation 
% mesh akin to our 1-step method. The main departing feature however is 
% that physics-informed neural operators assume knowledge of the underlying 
% PDE. 

\subsection{Outline of the article}\label{sec:outline}
    The rest of the article is organized as 
    follows: In \Cref{sec:proposed-method} we outline 
    the optimal recovery framework of 2-step and 1-step kernel 
    methods for learning equations and their solution operators. 
    \Cref{sec:algorithms} describes the algorithmic and 
    implementation details of our methods followed by 
    numerical experiments and benchmarks in \Cref{sec:numerics}. 
    Our theoretical analysis and quantitative error bounds are 
    summarized in \Cref{sec:theory}. We  give our concluding remarks  in 
    \Cref{sec:conclusion}. Finally, additional information including 
    omitted proofs, algorithmic details, and extra 
    numerical results are collected in the appendices.

\section{Kernel equation learning}\label{sec:proposed-method}
In this section we outline the KEqL framework
for learning DEs. 
\editstart
We primarily consider PDE problems but note that the same methodology 
applies to ODEs. 
Recalling the setup of \Cref{sec:problem-statement} 
we propose two approaches for equation learning and state reconstruction, i.e., finding 
$\whP \approx P$ and $\whu \approx u$: 
(1) a 2-step kernel method (2-step KEqL), where we first learn the $u^m$'s from the 
observations $u^m(Y^m)$
and use the result to approximate $P$. 
This method is outlined in \Cref{sec:methods:two-step-learning};
(2)  a 1-step kernel method (1-step KEqL), 
where $u^m$ and $P$ are learned jointly.
 This method is outlined in \Cref{sec:methods:one-step-learning}.
 As we will discuss in detail, the 1-step method is more accurate 
 in scarce data regimes while the 2-step method 
 is conceptually simple and computationally efficient.
 In \Cref{sec:methods:operator-learning} we discuss how our methodology can be extended to 
perform equation-mediated operator learning by defining a pseudo inverse for the learned 
operator $\wh{\P}$. 

Finally, we note that our 2-step and 1-step methods can also be cast as computational graph 
completion problems in the parlance of \cite{owhadi2022computational, bourdais2024codiscovering}. We show these computational graphs in  \Cref{fig:computational-graphs} and describe the underlying 
problems in the caption. As suggested by name, our 2-step method completes two 
disjoint graphs sequentially while our 1-step method completes a single graph.

Henceforth, we assume that all optimization problems are feasible and the notation $\arg\min$ is well-defined.

\editend

\begin{figure}[htp]
    \centering
    \begingroup
    \captionsetup[subfigure]{oneside,margin={-1ex,0ex}}
    \subfloat[Computational graph of 2-step KEqL]{
    \includegraphics[width=0.32\linewidth]{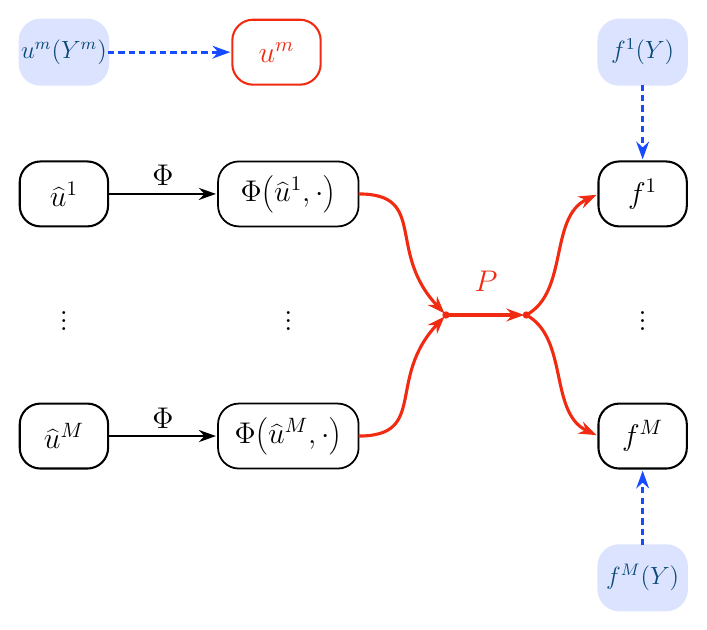}
    }
    \endgroup
    \qquad
    \subfloat[Computational graph of 1-step KEqL]{
    \includegraphics[width=.32\linewidth]{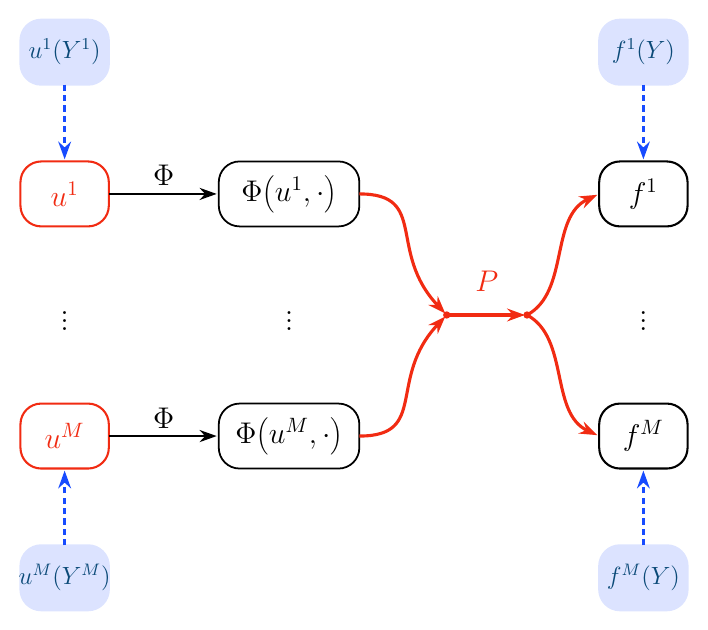}
    }
    \caption{
    \editstart
    (a) The computational graph \cite{owhadi2022computational} of 2-step KEqL. 
    Red edges and nodes are unknown functions of interest, blue 
    nodes denote data with 
    dashed lines pointing to places where that data 
    is injected, and black edges 
    and nodes denote known functions.
    Note that the graphs for $u^m$ and $P$ are disconnected, hence 
    the learning of $u^m$ and $P$ is performed sequentially in two steps; 
    (b) The computational graph for 1-step KEqL.
    Coloring conventions follow panel (a) 
    with the main difference being that the  $u^m$ and $P$
    are now connected and have to be learned simultaneously.
    \editend
    }
    \label{fig:computational-graphs}
\end{figure}

\subsection{2-step KEqL: first learn the $u^m$, then learn $P$}\label{sec:methods:two-step-learning}
Let $\UK: \overline{\mY} \times \overline{\mY} \to \R$ denote a positive definite and 
symmetric (PDS) kernel with its  associated RKHS 
$\UU$ with inner product $\langle \cdot, \cdot \rangle_\UU$ and norm $\| \cdot \|_\UU$; 
for a review of RKHS theory see 
\Cref{sec:theory:prelims} or the texts \cite{berlinet2011reproducing,steinwart2008support}. 
Then the first step in 2-step KEqL
approximates $u^m$ 
via the  optimal recovery problems \footnote{Note that one could also employ a different kernel $\UK^m$ for 
each instance of the  problem but we will not pursue this for brevity.}
\editstart
\begin{equation}\label{u-optimal-recovery}
    \begin{aligned}
    \whu^m & :=  \argmin_{v^m \in \UU} \| v^m \|_\UU \quad \text{subject to (s.t.)} 
    \quad v^m(Y^m) = u^m(Y^m) 
    \end{aligned}
\end{equation}
where, here and moving forward, we assume that the 
$Y^m$ contain a set of observation points  inside $\mY$
as well as on its boundary $\partial \mY$ where the Dirichlet 
data may be given.
We can also relax the equality constraints using a penalty method leading to a
quadratic optimization problem with nugget parameters (or observation noise standard deviations) $\sigma_u^2 >0$,
\begin{equation}\label{u-optimal-recovery-noisy}
    \whu^m := \argmin_{v^m \in \UU} \| v^m \|^2_\UU 
    + \frac{1}{2 \sigma_u^2} \| v^m(Y^m) -  u^m(Y^m)\|_2^2.
\end{equation}
\editend
With $\whu^m$ at hand we proceed to step two where
we approximate $P$ through a second optimal recovery problem. 
To do this, let us consider an independent set of points 
$Y = \{ y_1, \dots, y_K \} \subset \mY$  for integer $K \ge 1$ which we call 
{\it collocation points}.
This set of points may be chosen independent of the individual $Y^m$ but
to make our formulation simpler
we will assume it is 
chosen such that $\cup_{m=1}^M Y^m \subset Y$, i.e., the collocation points
$Y$ contain all of the observation points $Y^m$.

Now observe that the differential operators $L_q$ can be directly
applied to the $\whu^m$ 
\editstart 
and the functions $L_q \whu^m$ can be computed directly using the kernel representation. \editend
% , in fact, as we show in \Cref{sec:algorithms} 
% the functions $L_q \whu^m$ can be computed analytically or 
% using automatic differentiation.
Then given a PDS kernel $\PK: \mS \times \mS \to \R$ with RKHS $\PP$ we can approximate $P$ via the optimal recovery problem 
\begin{equation}\label{P-optimal-recovery}
    \whP := \argmin_{G \in \PP} \| G \|_\PP \quad 
    \text{s.t.} \quad  G \circ \Phi( \whu^m, Y) = f^m(Y), \quad m=1, \dots, M,
\end{equation}
where $f^m(Y) := (f^m(y_1), \dots, f^m(y_K) ) \in \R^K$ and 
we used the shorthand notation $\Phi(\whu^m, Y) = 
( \Phi(\whu^m, y_1), \dots, \Phi(\whu^m, y_K) ) \subset \mS$.
Similar to \eqref{u-optimal-recovery-noisy} this problem can also 
be relaxed using a nugget parameter $\sigma_P^2 >0$, 
\begin{equation}\label{P-optimal-recovery-noisy}
    \whP := \argmin_{G \in \PP} \| G \|_\PP^2 
    \editstart
    + \frac{1}{2\sigma_P^2 {M}} 
    \editend
    \sum_{m=1}^M \|  G \circ \Phi( \whu^m, Y) - f^m(Y) \|_2^2.
\end{equation}
In either case 
$\whP$ also admits 
% an analytic formula 
\editstart
a simple closed-form solution 
\editend
akin to the $\whu^m$, which we present
in \Cref{sec:algorithms}.
% \begin{equation}\label{P-kernel-interpolant}
%     \whP(s) = \PK(s, S) \PK(S,S)^{-1} f(Y), 
% \end{equation}
Note that in (\ref{P-optimal-recovery},~\ref{P-optimal-recovery-noisy}) we are using
the collocation points $Y$  to 
approximate 
% impose 
the infinite PDE constraint $G \circ \Phi(\whu^m, y) = f(y)$ for all $y \in \mY$, using a 
finite %discrete 
set of points, which justifies 
our choice of terminology as this is precisely 
the role of collocation points in PDE solvers \cite{chen2021solving}. This also motivates our preference to choose 
$Y$ to be as dense as the computational budget allows. 
% to be dense within computational budget constraints. 

In many applications we have some prior knowledge about the 
differential operator $\P$, for example in dynamic problems 
we know that a time derivative of  $\partial^h_t u$ for some integer $h$
is present. In such cases we can simply work with the model 
$\P(u) = (P + \overline{P}) \circ \Phi(u, \cdot)$ with $\overline{P}$
representing the known part of the differential operator.
This modification amounts to a simple reformulation of 
 \eqref{P-optimal-recovery-noisy}
(and similarly \eqref{P-optimal-recovery}) as 
\begin{align}\label{P-optimal-recovery-shifted}
    \whP = \argmin_{G \in \PP} \| G \|_\PP^2 
    \editstart
    + \frac{1}{2\sigma_P^2{M}} 
    \editend
    \sum_{m=1}^M \|  (G + \overline{P}) \circ \Phi( \whu^m, Y) -
    f^m(Y) \|_2^2.
\end{align}
Henceforth we will include $\overline{P}$ to account 
for the known parts of $\P$.

\begin{remark}\label{remark:1}
    We highlight two key observations: (1) Taking $\UU, \PP$ 
    to be RKHSs is innocuous at this 
    point in our exposition and only becomes important for our 
    theory and algorithms later. Therefore, one can take $\UU, \PP$ 
    to be abstract Banach spaces to obtain a generalization of our 
    framework so long as the observation functionals $\varphi_{\UU}$ are bounded in the appropriate dual space, and functions $\PP$ are regular enough for solutions to be defined. For example, taking these to be Barron spaces \cite{ma2022barron}
    models a neural net method such as that of \cite{chen2021physics} 
    while a sequential $l_1$ or $l_0$ spaces  
    over some basis coefficients model SINDy type methods \cite{brunton2016discovering, rudy2016datadriven}, 
    \editstart
    see also \Cref{sec:SINDy-PINN-review};
    \editend
    (2) as mentioned previously, 
    problems \eqref{u-optimal-recovery} and \eqref{P-optimal-recovery} 
    can be viewed as sequential  computational graph completion 
    tasks following the framework of \cite{owhadi2022computational} and as 
    shown in \Cref{fig:computational-graphs}(a). The simplicity of these 
    graphs reflects the convenience of the 2-step method.
% However, this convenience comes at the cost of accuracy and data efficiency since 
% the 2-step method will only succeed when data is 
% sufficient for the $L_q \widehat{u}^m$ to be good approximations to 
% $L_q u^m$ which is difficult for higher order derivatives.
\end{remark}

% \paragraph{Connection to existing works}%\label{sec:2-step-connections}
% {\color{red} BH TODO: INCORPORATE INTO LIT REVIEW}

% A slightly different version of 2-step KEqL was first
% introduced in \cite{LONG2024134095} as the RKHS/GP analogue of SINDy/PDE-FIND \cite{brunton2016discovering, rudy2017data}. The RKHS formulation is different from these works in three 
% directions: (1) the prior knowledge about $P$ is encoded in 
% the kernel $\PK$ as opposed to a dictionary; 
% (2) SINDy looks for $\wh{P}$ that is sparse in the dictionary
% while KEqL computes a minimum RKHS norm solution that is not 
% necessarily sparse; 
% (3) we readily include $y$ as an input to $P$ to
% accommodate variable coefficient PDEs. Dealing with 
% such problems is problematic for SINDy \cite{voina2024deep} since the
% variable coefficients  may depend on $y$ in a complex 
% manner that is not sparse in a particular dictionary. 
% \editstart

% \editend

\subsection{1-step KEqL: simultaneously learn the $u^m$ and $P$}\label{sec:methods:one-step-learning}

\editstart
We empirically observe that 
a major  shortcoming of 2-step KEqL is the decoupled learning 
of the $u^m$ and $P$ which limits its performance in scarce data 
regimes (see \Cref{sec:numerics}). Indeed, this is theoretically intuitive, since in low-data regimes our estimates $\whu^m$, and consequently $L_q \whu^m$, will be less accurate. This means that the inputs to $\whP$ via $\Phi(\whu^m, \cdot)$ will be too noisy, leading to poor equation recovery.
\editend
To remedy this, we propose 
 1-step KEqL that estimates $\whu^m$ and $\whP$ at the same time 
while imposing the requisite PDE constraint on the collocation points $Y$.
To this end, consider the joint optimal recovery problem 
\begin{equation}\label{one-shot-optimal-recover}
\begin{aligned}
     (\wh{\pmb u}, \whP) = & \argmin_{ \pmb v \in \UU^M, G \in \PP} && 
 \| G \|^2_\PP 
 \editstart
 + \frac{\lambda}{{M}} 
 \editend
 \sum_{m=1}^M \| v^m\|^2_\UU \\ 
& \text{s.t.} &&   v^m(Y^m) = u^m(Y^m),
% \editstart
% \qquad 
% \B(v^m)(Y^{\partial \mY}) = g^m(Y^{\partial \mY}),
% \editend
\\ 
& && (G + \overline{P}) \circ \Phi(v^m, Y) = f^m(Y), \quad \text{for} \quad m=1, \dots, M,
\end{aligned}
\end{equation}
where 
% $S^m = \{ s^m_1, \dots, s^m_K\}$ with $s^m_k := \Phi(v^m, y_k)$, 
$\lambda > 0$ is a user defined parameter and we used the notation $\pmb v := (v^1, \dots, v^M) \in \UU^M$ to denote the vector of candidate RKHS functions 
with their optimal values denoted as $\wh{\pmb u} \in \UU^M$.
In complete analogy with 2-step KEqL we can also relax the 
equality constraint in \eqref{one-shot-optimal-recover} using appropriate
nugget parameters $\sigma_u^2, \sigma_P^2 >0$ to obtain the formulation,
\begin{equation}\label{one-shot-optimal-recovery-noisy}
\begin{aligned}
     (\wh{\pmb u}, \whP) = & \argmin_{ \pmb v \in \UU^M, G \in \PP}  
 \| G \|^2_\PP +    \sum_{m=1}^M   
 \editstart
 \frac{\lambda}{{M}} 
 \editend
 \| v^m\|^2_\UU   + \frac{1}{2 \sigma_u^2} \| v^m(Y^m) - u^m(Y^m) \|_2^2  \\
  & \qquad  
 %  \editstart
 % + \frac{1}{2 \sigma_{\partial \mY}^2} \| \B(v^m)(Y^{\partial \mY}) - g^m(Y^{\partial \mY}) \|_2^2
 % \editend
 + \frac{1}{2 \sigma_P^2} \| (G + \overline{P}) \circ \Phi(v^m, Y) - f^m(Y) \|_2^2.
\end{aligned}
\end{equation}
Note that  the estimation of the $\wh{\pmb u}$ and $\whP$ is now coupled 
   due to the composition of $G$ and the $v^m$ in the constraint in \eqref{one-shot-optimal-recover} or in the last term in \eqref{one-shot-optimal-recovery-noisy}.
This coupling of the estimation of the $u^m$ and $P$ is the source of algorithmic challenges to solving 1-step KEqL. Nonetheless, we solve this problem 
after reformulation using a representer theorem and 
using a  Levenberg–Marquardt 
(LM) scheme, described in \Cref{sec:algorithms}. 

\begin{remark}
   We highlight three key observations: 
   (1) Once again our choice of $\UU, \PP$ as RKHSs is innocuous at 
   this point and choosing these to be other Banach spaces leads to other 1-step methods
   but this will have important consequences at the level of algorithms and theory
   as was discussed in \Cref{remark:1};
   (2) the second term in \eqref{one-shot-optimal-recover} resembles the discretization of 
   a PDE using a kernel collocation method akin to the method of \cite{chen2021solving};
   (3) as mentioned previously 1-step KEqL can also be viewed as a computational graph 
   completion problem with its graph shown in \Cref{fig:computational-graphs}(b). 
   However, compared to 
   the graph of 2-step KEqL we see that the underlying graph is more complex and in particular 
   edges corresponding to the $u^m$ and $P$ are both unknown, reflecting the compositional 
   structure of the problems \eqref{one-shot-optimal-recover} and \eqref{one-shot-optimal-recovery-noisy}. 
\end{remark}

% \paragraph{Connection to existing methods}%\label{sec:1-step-connections} 
% The concept of optimizing a loss function that jointly matches the given observations and the PDE constraint has been previously explored in the PIML and inverse problem literature. 
% Most notably, in \cite{chen2021physics, ijcai2021p283} use sparse regression for $P$
% but neural net or spline models for the $u^m$.
% At the same time weak SINDy  \cite{messenger2021weak,messenger2021weakPDE} 
% and  SINDy-UQ  \cite{hirsh2022sparsifying} can also be formulated as 1-step methods 
% although they were not presented or implemented as such.
% Our 1-step KEqL formulation is different  from these  works in two directions: 
% (1) the regularization due to our formulation 
% automatically leads to a stable algorithm that is amenable to 
% second order optimization when using RKHS methods;
% (2) the kernel formulation accommodates 
% a larger number of features and more flexibility in choosing and tuning the 
% features for both $u^m$ and $P$ including variable coefficients.

\subsection{
\editstart
Equation-mediated operator learning
\editend 
% \textcolor{red}{(Variationally induced solution map (equation learning + solver))}
}\label{sec:methods:operator-learning}

We now turn our attention to the operator learning 
problem of estimating the solution map 
\editstart
$\P^{-1}: \FF \to \UU$
\editend
of the PDE 
\eqref{generic-DE}. The dominant paradigm for operator learning \cite{KOVACHKI2024419} 
approximates $\P^{-1}$ via function space regression 
\editstart
directly from input-output pairs. A key appeal in doing this is learning discretization-independent maps between functions, using grids only through finite measurements. However, the inverse solution operator $\P^{-1}$ is often global and nonlocal, so purely data-driven learning must infer a complex high-dimensional map without directly using the PDE’s local structure. With many solution pairs and densely observed grids this works well, but in sparsely measured data regimes, direct operator regression is highly underdetermined. This motivates an alternative, equation-mediated view of operator learning. 
\editend
In this perspective, 
operator and equation learning problems are two sides of the 
same coin: where equation learning approximates $\P$, operator learning 
approximates the inverse $\P^{-1}$. To this end, following \cite{LONG2024134095}, 
let us write 
\editstart
$\widehat{\P}(u) := (\whP + \overline P) \circ \Phi (u,\cdot)$ to denote 
the differential operator associated with $\whP$. From this standpoint, \(\widehat{\P}\) naturally furnishes an implicitly defined estimate of the operator \(\P^{-1}\), which may be understood formally as solving the  PDE \(\widehat{\P}(u) = f\)
subject to boundary conditions. With a view towards numerical implementation, and 
following \cite{chen2021solving},
\editend
we propose approximating the solution operator $\P^{-1}$
by computing a 
{\it pseudo inverse} operator  $\widehat{\P}^{\dagger}$ defined variationally as
\editstart
\begin{equation}\label{eq:operator-learning}
    \widehat{\P}^{\dagger}(f):= \argmin_{v \in \UU} \| v \|_\UU 
    \quad \text{s.t.} \;\; \widehat{\P}(v)(Y) = f(Y) \;\; \text{and} \;\; 
    \B(v)(Y^{\partial \mY}) = g(Y^{\partial \mY}).
\end{equation}
Here we reintroduced the boundary operator $\B$ which may now be different from the Dirichlet 
boundary condition we assumed during the computation of $\widehat{\P}$
and we wrote $Y^{\partial \mY}$ to denote a set of collocation points used to impose the 
boundary conditions. 
% \textcolor{red}{This is dependent on $f$. Usual operator learning has a loss in which $\widehat{\P}^{\dagger}$ is not dependent on a particular $f$ or $u$. I don't mind calling our approach operator learning, but what is the advantage of what we are doing here? because we are clearly missing the usual operator learning benefits, i.e. fast evaluation after training, resolution-independence, discretization-invariant, etc. + mention benchmarks with operator learning methods in Shandian's paper}
Again, we can relax this problem by introducing 
% Lagrange multipliers to 
penalty parameters, obtaining \footnote{Note that we are abusing 
notation here since the operators defined in \eqref{eq:operator-learning} 
and \eqref{eq:operator-learning-relaxed-bc} are indeed different. But we assume 
that $\widehat{\P}$ always denotes one of these definitions per user's choice.}
\begin{equation}\label{eq:operator-learning-relaxed-bc}
\begin{aligned}
    \widehat{\P}^{\dagger}(f) := \argmin_{v \in \UU} \; &
    \| v \|^2_\UU
     + \frac{1}{2 \sigma_{\partial \mY}^2} \| \B(v)(Y^{\partial \mY}) -  g(Y^{\partial \mY}) \|_2^2 \\
    & + \frac{1}{2 \sigma_P^2}
    \big\| (\whP + \overline{P}) \circ \Phi(v, Y) - f(Y) \big\|_2^2.
\end{aligned}
\end{equation}
 Then \eqref{eq:operator-learning} and \eqref{eq:operator-learning-relaxed-bc} 
give regularized solutions to the possibly ill-posed learned PDE. 
\begin{remark}
We highlight three important observations: (1) 
 The learned PDE is not necessarily  
well-posed and so the $\wh{\P}^{-1}$ map may be ill-defined. Additionally,  $\widehat{\P}$ may contain arbitrary nonlinearities that 
make the design of traditional finite volume or finite element methods difficult. Thus,
employing a residual based solver such as the GP-PDE method of \cite{chen2021solving} or PINNs \cite{raissi2019physics}
is crucial; (2) the evaluation of $\widehat{\P}^\dagger$ comes at the cost of a 
PDE solve. This may be expensive if a fine collocation mesh is employed and so 
our notion of operator learning may not be readily suitable for model emulation applications. 
 However, one can still train an emulator 
for $\widehat{\P}^\dagger$ by paying an upfront cost to solve  \eqref{eq:operator-learning} 
with a large set of right hand side functions $f$ and boundary data $g$
in order to create a training data set for fast emulators such as FNO or DeepONet;
\editstart
(3) With $\widehat{\P}$ at hand we have a lot of freedom in choosing $\B$ and so 
in principle we can learn the PDE from data with a simple boundary data, such as 
Dirichlet, and solve it with an entirely new type of boundary condition. We could 
also pose the PDE on entirely new geometries. To our knowledge, these are 
new functionalities in the field of operator learning.
\editend
% can consider 
% both \eqref{eq:operator-learning} and 
% \eqref{eq:operator-learning-relaxed-bc} with boundary operators 
% other than $\B$ which expands the applicability of equation 
% mediated operator learning beyond new right hand side 
% functions $f$ to entirely new boundary data.
% , using an auxiliary operator learning method, with the benefit of inheriting the often improved statistical sample complexity the equation mediated operator learning approach. 
\end{remark}
\editend

\section{Implementation and algorithms}\label{sec:algorithms}
In this section, we outline a practical implementation of the 
2-step and 1-step KEqL problems. 
We focus our discussion on high-level ideas and refer the reader to  \Cref{app:algorithms} and our GitHub repository\footnote{\url{https://github.com/TADSGroup/kernelequationlearning}} for further details. 
\editstart
We first present the representer theorem, which forms the backbone of our implementation in \Cref{sec:representer}, then
\editend
outline the implementation of 2-step KEqL in \Cref{sec:algorithm:2-step}
followed by 1-step KEqL in \Cref{sec:algorithms:1-step}. 
\editstart
We  only consider Dirichlet boundary conditions where $\B(u) = u$
and discuss the extension to other boundary conditions in \Cref{app:beyond-Dirichlet-data}.
We discuss the implementation of our operator learning approach in \Cref{sec:algorithms:operator} followed by a brief summary of algorithms in \Cref{sec:algorithm-summary}
and kernel choice and hyperparameter tuning
in \Cref{sec:kernels}. 
\editend

\editstart
\subsection{Representer theorem}\label{sec:representer}
We start by stating the representer theorem for 
interpolation problems in RKHSs which directly implies the closed form solution formulae for 2-step KEqL 
and is the cornerstone of the solution formula for 1-step KEqL. 
\begin{proposition}[RKHS representer theorem {{\cite[Cor.~17.12]{owhadi2019operator}}}]\label{prop:generalized-rep-theorem}
 Let $\mH$ be an RKHS with kernel $\HK$ and let $\phi=(\phi_1,\dots,\phi_N)\in(\mH^\star)^N$ be a vector of bounded linear functionals on $\mH$. Consider
 \begin{equation*}
     \wh{f} := \argmin_{f \in \mH} \| f\|_{\mH} \quad \text{s.t.} \quad 
     \phi_i(f) = z_i, \qquad i=1, \dots, N, 
 \end{equation*}
 for $z = (z_1, \dots, z_N) \in \R^N$. Then every minimizer $\wh{f}$
 has the form $\wh{f} = \HK(\phi, \cdot)^T \wh\alpha$ 
 where $\wh\alpha = \HK(\phi, \phi)^{-1} z$. 
Here we defined the RKHS functions
$\HK(\phi_i, x) = \phi_i( \HK(\cdot, x)) \in \R$, along with 
the vector field
$\HK(\phi, \cdot) = ( \HK(\phi_1, \cdot), \dots, \HK(\phi_N, \cdot)) \in \mH^N$, 
 and the matrix $\HK(\phi, \phi) \in \R^{N \times N} $ with 
 entries $\phi_j( \HK(\phi_i, \cdot))$. Whenever $\HK(\phi, \phi)$
 is not invertible, the vector $\wh\alpha$ is defined in the least squares sense.
\end{proposition}
\editend

% \textcolor{orange}{The nugget terms here need to get adjusted as well, once we confirm the new scaling is the right factor.}

\subsection{Implementation of 2-step KEqL}\label{sec:algorithm:2-step}

Our implementation of 2-step KEqL is straightforward and directly follows from
an application of \Cref{prop:generalized-rep-theorem} with the $\phi_i$
taken to be pointwise evaluations on the mesh points $y^m_i$.
Then, the minimizer admits the familiar closed form
\editend
\begin{equation}\label{u-kernel-interpolant}
    \whu^m(y) =   \UK(Y^m, y)^T  \wh{\alpha}^m, \quad \text{where} \quad 
    \wh{\alpha}^m =  \left( \UK(Y^m, Y^m) + \sigma_u^2 I \right)^{-1} u^m(Y^m).
\end{equation}
Here $\UK(Y^m, \cdot) := ( \UK(y_1^m, \cdot), \cdots, \UK(y^m_N, \cdot) ) \in \UU^N$ 
is a column vector field and $\UK(Y^m, Y^m) \in \R^{N \times N}$
is a kernel matrix with entries $\UK(Y^m, Y^m)_{ij} = \UK(y^m_i, y^m_j)$.
Setting $\sigma_u =0$ further characterizes the solution of \eqref{u-optimal-recovery}
 assuming $\UK(Y^m, Y^m)$ is invertible.
Thanks to \eqref{u-kernel-interpolant}, we can directly compute
\begin{equation*}
    L_q \whu (y)  = L_q \UK(Y^m, y)^T \wh{\alpha}^m, 
    \quad \text{where} \quad L_q \UK(Y^m, y) = 
    \left( L_q \UK(y_1^m, y), \cdots, L_q \UK(y^m_N, y)   \right),
\end{equation*}
which requires us to apply the $L_q$ operators to the kernel $\UK$ 
as a function of its second input (the $y$ variable) 
for fixed 
values of its first input (the $y^m_n$ values). For typical PDEs this amounts to computing 
partial derivatives of  $\UK$ either analytically 
 or using automatic differentiation software.

\editstart
With \eqref{u-kernel-interpolant} at hand, we can solve \eqref{P-optimal-recovery-shifted} identically to the above.
\editend
Writing $S = \cup_{m=1}^M \Phi(\whu^m, Y) \subset \mS$, we have the  formula
\begin{equation}\label{P-kernel-interpolant}
    \whP(s) = \PK(S, s)^T \wh{\beta}, \qquad  \quad \text{where} \quad 
    \wh{\beta} = (\PK(S,S) + \sigma_P^2)^{-1} (f(Y) - \overline{P}(S)).
\end{equation}
Here we used the familiar shorthand notation $f(Y) = (f^1(Y), \dots, f^M(Y)) \in \R^{MK}$, 
the concatenation of the vectors $f^m(Y)$,
and $\overline{P}(S) \in \R^{MK}$,
 the vector of point values of $\overline{P}$ evaluated on  $S$, both viewed 
 as column vectors. Similar to \eqref{u-kernel-interpolant}, we also defined
the vector field 
\editstart
$\PK(S, \cdot) := (\PK(s_1, \cdot), \dots, \PK(s_{MK}, \cdot)) \in \PP^{MK}$,
\editend
and the matrix 
$\PK(S,S) = \big( \PK(s_i, s_j) \big)_{i,j=1}^{MK} \in \R^{MK \times MK}$, 
where the $s_j$'s denote the elements of the point cloud $S$ 
\editstart
(recall \Cref{fig:equation-learning-cartoon}).
\editend
In further analogy with the first step, 
setting $\sigma_P = 0$ gives the solution to \eqref{P-optimal-recovery}
when $\PK(S,S)$ is invertible.

The formulae \eqref{u-kernel-interpolant} and \eqref{P-kernel-interpolant}
 highlight the convenience and computational efficacy of 2-step KEqL
since each equation requires  a single linear solve involving a kernel matrix;
 this can be done efficiently using sparse or randomized algorithms for large problems  \cite{schafer2020sparse, chen2022randomly, chen2024sparse} or \editstart generic least squares solvers for smaller 
 problems.\editend

\subsection{Implementation of 1-step KEqL}\label{sec:algorithms:1-step}

The implementation of 1-step KEqL is more involved and requires 
the solution of a compositional optimization problem. 
The first step is to derive an equivalent discrete formulation 
of \eqref{one-shot-optimal-recover} and 
\eqref{one-shot-optimal-recovery-noisy}; this amounts to the derivation of a representer theorem which we state for \eqref{one-shot-optimal-recover}.

Assuming $\UU$ is sufficiently regular, define 
the bounded and linear functionals $\phi^{q}_k := 
\delta_{y_k} \circ L_q \in \UU^\star$ for $q = 1, \dots, Q$ and 
$k =1, \dots, K$, recalling our convention that $L_1$ denotes 
the identity map so that $\phi^1_k = \delta_{y_k}$.
Further, observe that $\Phi(u, y_k) = (y_k, \phi^1_k(u), \dots, \phi^Q_k(u))$ by definition. 
% \editstart
% where the  $\phi^q_k$ 
% \textcolor{red}{denote the linear operators
% that give the subset of components of $\Phi(u, y_k)$ as a function of $u$, (YJ to BH: I wasn't sure what you meant in this line. maybe we could write it differently.)}
% justifying our choice of notation.
\editstart
Following the notation of \Cref{prop:generalized-rep-theorem},
write 
$\UK( \phi^q_k, y) = \phi^q_k ( \UK(\cdot, y)) \in \R$ where $\UK( \phi^q_k, \cdot) \in \UU$
\editend
, i.e., 
the RKHS function obtained by applying $\phi^q_k$ to $\UK( \cdot, y)$ 
 for every fixed value of $y$.
Further define $\UK( \phi^q_k, \phi^\ell_j) := 
\phi^\ell_j( \UK( \phi^q_k, \cdot) )$ for $\ell =1, \dots, Q$
and $j = 1, \dots, K$. Next define 
the vector field 
\editstart
$\UK( \phi, \cdot) := \big( \UK(\phi^1_1, \cdot), \dots, \UK(\phi^1_K, \cdot), \dots, 
\UK(\phi^Q_1, \cdot), \dots, \UK(\phi^Q_K, \cdot)\big) \in \UU^{QK}, $
\editend
along with the vectors $\UK(\phi, \phi^q_k) = \big( \UK(\phi^1_1, \phi^q_k), \dots, \UK(\phi^1_K, \phi^q_k), \dots, 
\UK(\phi^Q_1, \phi^q_k), \dots, \UK(\phi^Q_K, \phi^q_k)\big) \in \R^{QK}$.
Let $\UK(\phi^q, \phi^\ell) \in \R^{K \times K} $
denote the matrices
with entries $ \UK(\phi^q, \phi^\ell)_{k,j} = \UK(\phi^q_k, \phi^\ell_j)$
and the block-matrix $\UK(\phi, \phi) \in \R^{QK \times QK}$ 
with blocks $\UK(\phi, \phi)_{q,\ell} = \UK(\phi^q, \phi^\ell) \in \R^{K \times K}$.
Once again, note that these vector fields and matrices can be computed offline
using either analytical expressions  or automatic differentiation.
With this new notation, we have the following characterization of the minimizers of \eqref{one-shot-optimal-recover}. 

\begin{theorem}\label{thm:one-step-equivalence}
    Suppose that $\UU$ is sufficiently regular so that 
    the $\phi^q_k$ are well-defined 
    as elements of $\UU^\star$. Then every minimizing  tuple
        $(\wh{ \pmb u}, \whP)$ of 
     \eqref{one-shot-optimal-recover} 
     has the  form
        $\whu^m(y) = \UK(\phi, y)^T\wh{\alpha}^m$
        and
        $\whP(s) =   
        \PK(S(\widehat{\pmb \alpha}), s)^T \wh{\beta}$,
    for  coefficients
    $(\wh{\pmb \alpha} , \wh{\beta})$ that solve 
    \begin{equation}\label{one-step-representer-form}
\begin{aligned}
     (\wh{ \pmb \alpha}, \wh{\beta}) = & 
     \argmin_{ \pmb \alpha \in (\R^{QK})^M, \: \beta \in \R^{MK}} && 
\beta^T \PK(S(\pmb \alpha),S(\pmb \alpha)) \beta 
\editstart
+ \frac{\lambda}{M} 
\editend
\sum_{m=1}^M (\alpha^m)^T \UK(\phi,\phi) \alpha^m, \\ 
 % & \text{s.t.} && z^{m, q}_k = (\gamma^m)^T \UK(\phi, \phi^q_k), 
 % \quad m=1, \dots, M, \quad k=1, \dots, K, \quad q= 1, \dots, Q, \\
 & \text{s.t.} && \UK(\phi, Y^m)^T \alpha^m = u^m(Y^m),  \\   
 & &&  \PK(S(\pmb \alpha), S^m(\pmb \alpha))^T \beta = f^m(Y) - \overline{P}(S^m(\alpha^m)).
\end{aligned}
\end{equation}
Here we introduced the block vector 
$\pmb \alpha := ( \alpha^1, \dots, \alpha^M) \in (\R^{QK})^M$ that collects
$\alpha^m \in \R^{QK}$, the point clouds $S^m(\alpha^m) = \Phi( \UK(\phi, \cdot)^T \alpha^m ,Y) \subset \mS$ and in turn $S(\pmb \alpha) = \cup_{m=1}^M S^m(\alpha^m)$, 
as well as the matrices
$\UK(\phi, Y^m) \in \R^{QK \times N}$  with 
columns $\UK(\phi, y^m_n)$, and $\PK(S(\pmb \alpha), S^m(\alpha^m)) \in \R^{MK \times K}$
with columns $\PK(S, s^m_k)$.
\end{theorem}
%%%%% END OPTION 2 %%%%%%
 The proof technique  is based on the observation that 
 conditioned on the values $L_q v^m(y_k)$, the optimization 
 problems for the $v^m$ and $G$ variables in \eqref{one-shot-optimal-recover} will decouple. 
 Then the result is obtained 
 by introducing auxiliary  variables representing the $L_q v^m(y_k)$ 
 and applying \Cref{prop:generalized-rep-theorem}; 
 see \Cref{app:theory:representer-theorem-1-step-method} 
 for details. 
 
Problem \eqref{one-step-representer-form} is the key to the design of 
our algorithms as it is readily implementable without further discretization. 
%  This fact further implies that 
% our theoretical analysis of \eqref{one-shot-optimal-recover} in \Cref{sec:theory}  applies to the algorithms we implement in practice. 
We can also  relax the equality constraints in \eqref{one-step-representer-form}
to obtain the analogous problem to \eqref{one-shot-optimal-recovery-noisy} which 
is more convenient in practice:
\begin{equation}\label{one-step-representer-relaxed}
  \begin{aligned}
     (\wh{\pmb \alpha}, \wh{\beta}) = &
      \argmin_{{ \pmb \alpha} \in (\R^{QK})^M,  \: \beta \in \R^{MK}}    
      \beta^T \PK(S(\pmb \alpha), S(\pmb \alpha)) \beta 
+   \sum_{m=1}^M 
\editstart
\frac{\lambda}{M} 
\editend
(\alpha^m)^T \UK(\phi,\phi) \alpha^m \\ 
& \qquad 
+    \frac{1}{2 \sigma_u^2} \| \UK(\phi, Y^m)^T \alpha^m - u^m(Y^m) \|_2^2 \\
& \qquad + \frac{1}{2 \sigma_P^2} 
\| \PK(S(\pmb \alpha), S^m(\alpha^m))^T \beta +
\overline{P}(S^m(\alpha^m)) - f^m(Y) \|_2^2. 
  \end{aligned}
\end{equation} 

\editstart
\subsection{Implementation of equation-mediated operator learning} \label{sec:algorithms:operator}
As we mentioned earlier, our approach to solving 
\eqref{eq:operator-learning} and \eqref{eq:operator-learning-relaxed-bc} is precisely 
the GP-PDE method of \cite{chen2021solving} which relies on a simplification of 
\Cref{thm:one-step-equivalence}. To see the connection,
consider \eqref{one-step-representer-relaxed} with $M=1$ and 
simply fix the variable $\beta = \widehat{\beta}$, 
i.e., we no longer learn the PDE. Further replace 
 the data misfit 
 $\frac{1}{2 \sigma_u^2} \| \UK(\phi, Y^m)^T \alpha^m - u^m(Y^m) \|_2^2$ with the 
 boundary data misfit 
$\frac{1}{2 \sigma_{\partial \mY}^2} \| \B( \UK(\phi, \cdot)^T \alpha^m )(Y^{\partial \mY}) - u^m(Y^{\partial \mY}) \|_2^2$.
 This gives 
 \begin{equation}\label{eq:opl-representer}
 \begin{aligned}
     \widehat \alpha =  \argmin_{\alpha \in \R^{QK}} \quad   
      \alpha^T \UK(\phi,\phi) \alpha &+ \frac{1}{2 \sigma_{\partial \mY}^2} 
      \| \B( \UK(\phi, \cdot)^T \alpha )(Y^{\partial \mY}) - u(Y^{\partial \mY}) \|_2^2 \\ 
    & + \frac{1}{2 \sigma_P^2} \| (\whP + \overline{P}) 
    \circ \Phi( \UK(\phi,\phi) \alpha, Y) - f(Y) \|_2^2,  
  \end{aligned}
 \end{equation} 
 for the coefficients $\widehat{\alpha}$ so that $\widehat{u} = \UK(\phi, \cdot)^T \widehat{\alpha}$ approximates the solution to \eqref{eq:operator-learning-relaxed-bc}.
Since \eqref{eq:opl-representer} is essentially a smaller version 
of \eqref{one-step-representer-relaxed}, both problems can be solved using the same 
numerical algorithms with small modifications
to deal with the boundary term.

\editend
 
%  the result of \Cref{thm:one-step-equivalence} precisely characterizes
% the minimizers of 

% Recall that operator learning in KEqL is not performed by regressing directly between function spaces; rather, once $\widehat{P}$ has been obtained from 1-step or 2-step equation learning, the solution map is computed by solving the optimization problems outlined in \Cref{sec:methods:operator-learning} and in particular \eqref{}. Interestingly, these problems can be solved using the same algorithms developed 
% for the 1-step KEqL problem as we discuss in \Cref{sec:operator-learning-details}.

% \editend

\subsection{Summary of algorithms}\label{sec:algorithm-summary}

\editstart
Observe that the loss \eqref{one-step-representer-relaxed} is quadratic in $\beta$
but non-quadratic in $\pmb{\alpha}$ 
due to the terms involving the 
matrices $\PK( S(\pmb \alpha), S(\pmb \alpha) )$ and $\PK( S(\pmb\alpha), S(\alpha^m))$
as well as $\overline{P}(S^m(\alpha^m))$. This form suggests that by linearizing 
these terms we may obtain an iterative scheme that involves the solution of a 
jointly quadratic optimization problem for the tuple $(\pmb \alpha, \beta)$ at 
each step. This is precisely the idea behind the Gauss-Newton and LM methods 
for numerical optimization \cite[Ch.~10]{nocedal2006numerical} that we also employ here. Since these are well-known algorithms, we will not discuss them here and refer the reader to 
\Cref{app:algorithms} where we give a detailed account of our algorithms along with 
various ideas from dimension reduction, sparse linear algebra, and preconditioning that 
we use to scale our algorithms. 
\editend

% \begin{remark}
% \editstart
% Minimizing \eqref{one-step-representer-relaxed} using gradient descent can be numerically unstable because small changes in the coefficients defining the auxiliary functions $v^m$ may produce large changes in their derivatives and hence in the collocation points $S^m(\alpha^m)$. In practice, we implement an intermediate method called the reduced 1-step method that interpolates between our 1-step and 2-step methods, and serves to improve computational efficiency and scalability for larger problems and data sets. 
% We solve this reduced 1-step problem using a damped LM scheme based on local quadratic approximations of the residual term. See the SM for details.
% \editend
% \end{remark}

% \subsection{Going beyond Dirichlet boundary conditions}
% \label{sec:algorithms:beyond-Dirichlet}

\subsection{Choosing kernels, hyper-parameters, and differential operators}\label{sec:kernels}
\editstart
While the 1- and 2-step KEqL  are generic and well-defined for any choice of $\PK$ and $\UK$ (assuming sufficient regularity), their 
practical performance is closely tied to a good choice of  kernels as is typical for RKHS/GP methods. For most of our experiments, we take $\UK$ to be an RBF kernel 
while $\PK$ is also taken to be RBF, polynomial, or a hybrid kernel as defined in \Cref{sm:choosing-kernels} and also 
 discussed later in \Cref{sec:numerics}.
% Broadly speaking,
% the choice of these kernels imposes constraints on our model 
%  for the  $\wh u^m$ and  $\wh P$.

The loss \eqref{one-step-representer-relaxed} also involves multiple hyperparameters, most notably the lengthscale parameters for $\PK$ and $\UK$, as well as the weights $\lambda, \sigma_u,$ and $\sigma_P$.
We choose these via a mix of heuristics and hand tuning. First, we apply 2-step KEqL, in which case standard
methods such as maximum likelihood estimation (MLE) or cross-validation (CV)
can be applied to tune the kernel parameters for both $\PK$ and $\UK$.
At this stage, we can also obtain a rough estimate
of the RKHS norms of the learned $\whP$ and $\whu^m$. Next, we formulate \eqref{one-step-representer-form}
(or its analogue in our algorithm) using the kernel parameters from the warm-start stage and
initialize our optimizer at the 2-step KEqL solution. We found that this
warm-start procedure has a major impact on the convergence of our optimizers.
We typically hand-tune $\lambda$ while choosing
$\sigma_u$ and $\sigma_P$ to be as small as possible while maintaining numerical
stability during the optimization procedure.
A simple method for choosing these three variables is to
start with very small values and increase the regularization
parameters until a notable increase in the PDE residual or
the data misfit term is detected.
 
Another important hyperparameter in our model is the choice of the differential operators $L_q$, both the individual operators and the total number $Q$, which in turn identifies the map $\Phi$. We assume that these operators are chosen by the user, although we do
present various studies that demonstrate the effect of varying the choice of the $L_q$'s in \Cref{sec:numerics} and \Cref{app:numerics}.
Regardless, designing automated procedures for tuning all of
the above hyperparameters in our model remains open.
\editend

\section{Experiments}\label{sec:numerics}
We now present a series of numerical experiments 
that demonstrate the performance of our 1-step and 2-step methodologies
for estimating the $u^m$, learning  
$\P$, and the  
solution operator $\P^{-1}$.  For benchmarking, we compare our algorithms with 
SINDy \cite{rudy2017data}, as a 2-step 
 dictionary based method, and PINN-SR  \cite{chen2021physics}, 
 as a 1-step neural net based method. 
 We will focus on the key results for each experiment and give additional results and details about our setup in \Cref{app:numerics}.

Common structure and performance metrics are discussed in \Cref{sec:numerics:common-setup}. 
A Duffing oscillator is studied in \Cref{sec:numerics:ODE-example} followed 
by Burgers' PDE in \Cref{sec:numerics:burgers} and Darcy's flow elliptic PDE 
in \Cref{sec:numerics:darcy}, 
\editstart 
followed by a set of experiments regarding the effect of the choice of the $L_q$'s in \Cref{sec:choice-of-dictionaries}.
\editend

\subsection{Common structure}\label{sec:numerics:common-setup}
For dynamic problems, we assume $\mY = (0, T] \times \Omega$ for some $T> 0$ and a 
spacial domain $\Omega$ and take $\overline{P}(\Phi(u, y)) = \overline{P}(\Phi(u, (t, x))) = 
\partial_t^h u(t, x)$  to be known (the order $h$ will be clarified for each example). 
We often use RKHSs $\UU$ that do not automatically satisfy the boundary conditions of the PDEs 
so we use additional observation points on the boundaries to impose 
those conditions. To this end, we write $N_{\mY}$ to denote the 
number of observation points in the interior of $\mY$ while we use 
 $N_{\partial \mY}$ to denote the number of observation points on the boundary, so that the total number of observation points is always  $N = N_{\mY} + N_{\partial \mY}$.
For all experiments, we consider training data of the form 
$(u^m(Y^m), f^m)_{m=1}^M$. In addition to the collocation points $Y$
and the observation points $Y^m$,
we also consider an 
independent test mesh $Y_{\test} \subset \mY$ which 
will be used for reporting errors. 

We will consider 
three types of test errors: (i) the 
\editstart 
state reconstruction error, 
\editend which quantifies the relative error of the $\whu^m$ on the training data, 
\begin{equation*}
    \mR_{\text{rec}} \big( (u^m)_{m=1}^M \big) 
    := \frac{1}{M} \sum_{m=1}^M 
    \frac{\| u^m(Y_\test) - \whu^m(Y_\test) \|_2^2}{\| u^m (Y_\test) \|_2^2},
\end{equation*}
(ii) the {\it equation learning error}, quantifying the relative error of $\wh\P$ over a set of functions $\mathcal{W}$,
\begin{equation*}
    \mR_{\text{eql}} \big(\mathcal{W}) 
    := \frac{1}{| \mathcal{W}|} \sum_{w \in \mathcal{W}} 
    \frac{\| \P(w) (Y_\test) - \widehat{\P}(w) (Y_\test) \|_2^2}
    {\| \P(w)(Y_\test) \|_2^2},
\end{equation*}
where $| \mathcal{W} |$ denotes the cardinality of $\mathcal{W}$, and (iii) the {\it operator learning error}, which quantifies the relative error of the map 
$\widehat{\P}^\dagger$, 
\begin{equation*}
    \mR_{\text{opl}} \big(\mathcal{W}) 
    := \frac{1}{| \mathcal{W}|} \sum_{w \in \mathcal{W}} 
    \frac{\| \P^{-1}(w) (Y_\test) - \widehat{\P}^\dagger(w) (Y_\test) \|_2^2}
    {\| \P^{-1}(w)(Y_\test) \|_2^2}\,.
\end{equation*}
We  highlight that here $\mathcal{W}$ constitutes initial or boundary conditions as well as source terms depending on the problem at hand. 
For different experiments, we report errors with  $\mathcal{W}$ taken to be
the training data (train), an in-distribution test set (ID), or  an out-of-distribution  test set (OOD).

\subsection{The Duffing oscillator}\label{sec:numerics:ODE-example}
In this example, we compared the performance of 1-step and 2-step KEqL 
along with SINDy for learning a 1D nonlinear ODE. Our focus 
is on  
\editstart 
state reconstruction
\editend
and forecasting performance.

\subsubsection{Problem setup} Consider $\mY = (0,50)$ and
the Duffing ODE
\begin{equation}\label{duffing_ODE}
\left\{
\begin{aligned}
        & \P(u) = \partial_{t}^2 u  - 3u +3u^3 +0.2 \partial_t u= \cos(2t),  
        \quad  t \in \mY,\\
        & u(0)  = \partial_t u(0) = 0. 
\end{aligned}
\right.
\end{equation}

To generate the training data we solved the ODE  using a standard  adaptive solver. 
The numerical solution was then subsampled on the observation points that were also 
picked on a uniform grid; see \Cref{fig:duffing_main}(a). We chose $M=1$, and so we aim to learn the ODE 
and its solution from observations of a single trajectory. The test data sets were generated similarly but on a finer uniform mesh
with different initial conditions.

\subsubsection{Algorithm setup} 

We chose the collocation points $Y$ to be a fine uniform grid in $\mY$.
We further chose the operators $L_1= \text{Id}:u \mapsto u$ , $L_2 : u \mapsto\partial_t u$, 
and $\overline{P} (u) = \partial_{t}^2 u$. 
We chose $\UK$ to be a rational quadratic kernel and $\PK$ to be RBF.
To approximate the derivatives in SINDy we used the same kernel 
interpolant as in 2-step KEqL. As for the SINDy dictionary, we chose 
all polynomial terms of up to third degree in the variables $\left\{1,u,\partial_t u\right\}$; we empirically found this dictionary to give the best training 
and test errors for SINDy.

\subsubsection{Results}
We summarize our key findings for this example in \Cref{fig:duffing_main}. 
Panel (a) visually compares the quality of the 
\editstart
reconstructed
\editend
solution $\whu$ for 1-step and 2-step KEqL. We clearly see that 1-step KEqL is superior in this task.  This visual performance is further confirmed 
in Panel (b) where we observe an order of magnitude improvement 
between 1-step and 2-step methods in terms of relative errors. 

For our next test we considered simulating the learned ODE 
with new initial conditions, this is essentially a 
forecasting/extrapolation 
problem where we try to predict the state of the system 
outside of the window of observations. We show three examples
of random initial conditions in \Cref{fig:duffing_main}(c), 
comparing the solutions computed using 1- and 2-step KEqL and SINDy. 
We observe that the 1-step method generally tracks 
the true solution for a longer time interval although all solutions 
eventually diverge. The superior performance of 1-step KEqL 
is also evident in panel (b) where we used $\mathcal{R}_{\text{opl}}$
to denote the pointwise error between the extrapolated solutions and the true 
states. Once again we observe an order of magnitude improvement 
between 1-step and 2-step with smaller gaps for larger time windows as expected. 

\begin{figure}[htp]
\centering
\begin{minipage}{0.58\textwidth}
    \scriptsize
    \centering    
    \hspace{0ex}
    \captionsetup[subfigure]{oneside,margin={-8ex,0ex}}
    \subfloat[The true $u$, the data, and the learned $\whu$s]{
    \begin{overpic}[scale = 0.32]{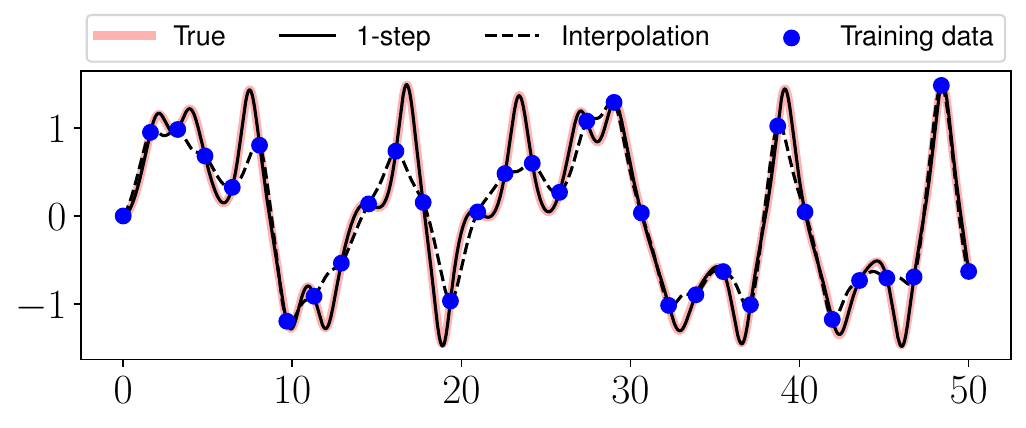}
    \put(-1,19){\rotatebox{90}{$u$}}
    \put(53,-1){$t$}
    \end{overpic}
    }\\
    \captionsetup[subfigure]{oneside,margin={-1ex,0ex}}
    \subfloat[
    \editstart 
    state reconstruction
    \editend
    and extrapolation errors]{
    \renewcommand{\arraystretch}{1.3} % Adjust row height
    \hspace{0ex}
    \begin{tabular}{cl|ccc}
         & & & Method &  \\
        \hline
         &$Y_{\text{test}}$ &1-step & 2-step & SINDy \\
        \hline
         $\mathcal{R}_{\text{rec}}$ & [0,50]&$9.6e^{-3}$ & $3.4e^{-1}$ & $3.4e^{-1}$ \\ \hline
         & [0,3]& $2.2e^{-2}$ & $3.4e^{-1}$ & $4.0e^{-1}$ \\
         $\mathcal{R}_{\text{opl}}$& [0,6]&$1.3e^{-1}$ & $6.2e^{-1}$ & $7.6e^{-1}$ \\
         &[0,10]& $4.2e^{-1}$ & $1.1e^{0}$ & $1.1e^{0}$ \\
        \hline  
    \end{tabular}
    }
\end{minipage}
\begin{minipage}{0.4\textwidth}
    % \centering
    % \captionsetup[subfigure]{oneside,margin={1ex,4ex}}
    \subfloat[Extrapolated solutions
    using the learned ODEs]{
    \scriptsize
    \begingroup
    \renewcommand{\arraystretch}{-20}
    \begin{tabular}{c}
           \begin{overpic}[width=0.8\textwidth]{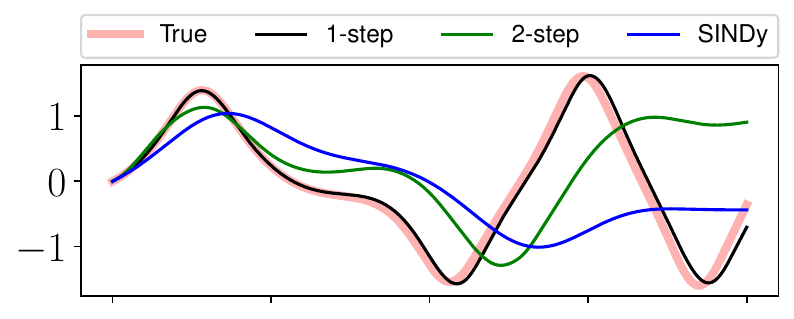}
           \put(0,15){\rotatebox{90}{$u$}}
           \end{overpic}
           \\
           \begin{overpic}[width=0.8\textwidth]{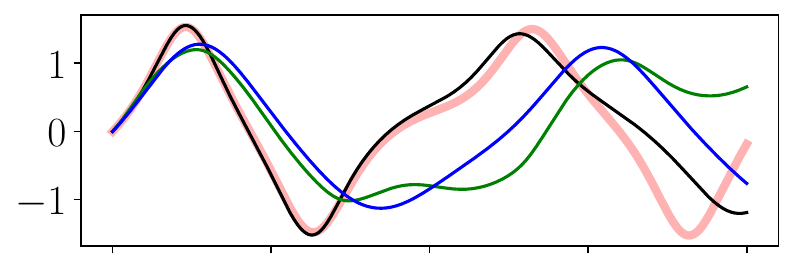}
           \put(0,15){\rotatebox{90}{$u$}}
           \end{overpic}
           \\
          \begin{overpic}[width=0.8\textwidth]{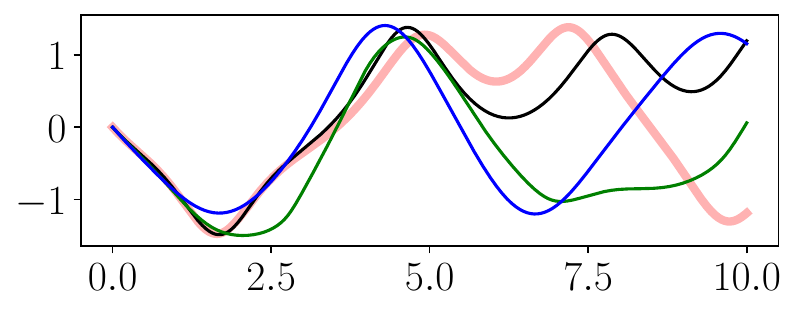}
          \put(0,22){\rotatebox{90}{$u$}}
          \put(53,-2){$t$}       
            \end{overpic}
    \end{tabular}
    \endgroup
    }
\end{minipage}

\caption{Representative numerical results for the Duffing 
oscillator \eqref{duffing_ODE}: 
(a) Shows the training data 
and the ground truth  state of the system $u$ in comparison to the 
\editstart
reconstructed
\editend
state $\whu$ using 1-step KEqL and 2-step methods;
(b) Quantitative values of relative 
\editstart
state reconstruction
\editend
and operator learning errors, the 
latter is reported for different time windows. 
These values were averaged over three random initial conditions; 
(c) Visualization of three extrapolated dynamics
used to compute $\mathcal{R}_{\text{opl}}$.}
\label{fig:duffing_main}
\end{figure}

\subsection{The Burgers' PDE}\label{sec:numerics:burgers}
In this example, we compared 1-step KEqL with
SINDy (as a 2-step competitor) and PINN-SR (as a 1-step competitor) for the 1D Burgers' PDE.
Our focus is on equation and operator learning from scarce observations of a 
single trajectory.

\subsubsection{Problem setup}
We consider  $\mY = (0,1] \times (0,1)$ and the Burgers' PDE 
\begin{equation}\label{burgers-PDE}
\left\{
\begin{aligned}
        &\P(u) = \partial_t u  + \vartheta u \partial_x u - \nu \partial_{x}^2 u = 0
        \text{ for }(t,x) \in \mY, \\ 
        &u(0, x) = u_0(x), 
        \text{ and }
        u(t, 0)  = u(t, 1) = 0. 
\end{aligned}
\right.
\end{equation}
In all experiments we took $M=1$ 
so  that a single solution is observed on a scarce set of observations,
\editstart
the parameters $\vartheta, \nu$ vary by experiment and their values 
are specified in \Cref{app:numerics:burgers}.
\editend 
The ground truth solutions were generated by prescribing the initial conditions 
$u_0$ and then solving the PDE using a 
traditional numerical solver on a fine mesh. The observed data 
was then subsampled from a set of Chebyshev collocation points that
were also used in the implementation of all algorithms as the collocation mesh $Y$.

\subsubsection{Algorithm setup}
We took
$L_1: u \mapsto u$ , $L_2 : u \mapsto\partial_x u$, $L_3 : u \mapsto\partial_{x}^2 u$ and also $\overline{P}(u) = \partial_t u$.  $\UK$ was taken to be RBF  and $\PK$ 
was a polynomial kernel of degree 2.
To approximate the derivatives for SINDy, we used the same RBF kernel used in KEqL and the dictionary terms were polynomial features 
$\{1,u,\partial_x u,\partial_{x}^2 u\}$; we \editstart empirically \editend found that this dictionary
gave the best results for SINDy. For the PINN-SR method, we used the same  dictionary for $P$ along with a feed-forward neural network with 
8 layers and width 20 to approximate $u$.

% \editstart
% In \Cref{subsubsec: exp_shock_ablation} and \Cref{subsubsec: one-shot-ablation-study} we present 
% ablation studies of the effect of the choice of the 
% dictionary on the quality of various methods in the setting of 
% the Burgers' equation benchmark. [BRIEFLY DESCRIBE WHAT WE SEE ONCE THOSE ARE DONE] We test two variants per method, one in which the degree-2 polynomial features are restricted to $\{1, u, \partial_x u\}$ (excluding $\partial_x^2 u$), and one in which the feature set is expanded to $\{1, u, \partial_x u, \partial_x^2 u, \partial_{tx} u, \partial_t^2 u\}$. The experiments evaluate state reconstruction and operator learning errors under boundary-only observations, comparing the 1-step KEqL method and PINN-SR across reduced and expanded feature sets. For each method and variant, we compute the recovered equation coefficients on features including $\partial_x u$, $\partial_x^2 u$, and the nonlinear advection term $u\partial_x u$. We also compare these against the fully-observed setting and reports the kernel interpolant, which is computed independently of the dynamics-inference second step.
% \editend

\subsubsection{Results}
We begin by considering the 
\editstart
state reconstruction
\editend
and equation learning 
errors for one-shot learning of Burgers' as shown in \Cref{fig:burgers_main}(a) where 
we are reporting the errors for a fixed smooth initial condition 
with different number of observation points. 
Note that these errors were computed over the training data 
but they are different from the training error 
of the algorithm as they compare the 
\editstart
reconstructed
\editend
solution and the learned PDE on the test mesh $Y_{\text{test}}$.
Since the observation 
points are random, we are reporting the average errors along with the best and worst errors over different runs giving the shaded regions. 
We found that 1-step KEqL leads to the best 
\editstart
state reconstruction
\editend
errors 
with a wider performance gap when $N_\mY$ is small. Interestingly, the 
performance gap appears to be more pronounced in the equation learning 
case where we observed an order of magnitude improvement over 
SINDy and almost two orders improvement over PINN-SR. 
We observed significantly larger errors for PINN-SR here which  we 
attribute  to difficulties in tuning and optimizing PINN models. In fact,
 we could often bring these errors down by hand tuning the algorithm 
 in each instance of the experiments but 
automatic strategies proposed in \cite{chen2021physics} appeared 
to be very sensitive to the location of observation points.

This issue was exacerbated further with
random initial conditions as depicted in  \Cref{fig:burgers_main}(b).  
We report the average errors  along with the best and worst errors as shaded areas as before. A similar conclusion can be drawn here as in \Cref{fig:burgers_main}(a), where 1-step KEqL continues to achieve the lowest 
\editstart 
state reconstruction
\editend
and equation learning errors, with a larger performance gap when $N_\mY$ is small. However, we now observe that the KEqL and SINDy methods exhibit greater robustness across different experiments, whereas the performance of the PINN-SR deteriorates significantly. 
\editend

In \Cref{fig:burgers_main}(c,d,e) we show three enlightening examples 
that highlight the  performance gap of 1-step and 2-step methods. 
In (c) we hand picked an initial condition that leads to a solution with 
multiple (near) shocks; we see that despite scarcity of observations, 1-step KEqL 
and PINN-SR are able to capture the general shape of the solution while
2-step methods are significantly worse. In (d) we show an instance where
we only observed the value of the solution on the boundary of the domain
while in row (e) we solved the learned equation \editstart from panel (d) \editend for a new initial condition. 
Once again, we see that 1-step methods give better performance 
for 
\editstart 
state reconstruction,
\editend
however, when solving 
the learned PDE, a noticeable performance gap appears across the three 
methods, with 1-step KEqL giving the best solution.
This example further reflects the performance gaps observed in 
panel (a).

\begin{figure}[htp]

\centering
\scriptsize
\begingroup
\setlength{\tabcolsep}{0pt}
\begin{tabular}{cc}
\vspace{4ex}
\adjustbox{valign=c}{
\begin{overpic}[width=0.28\textwidth]{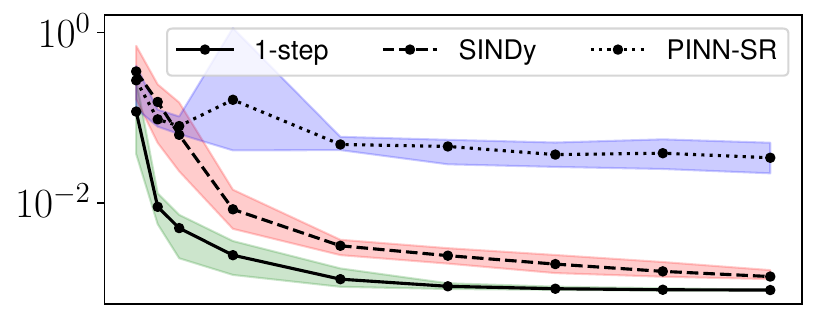}
\put(40,41){\scriptsize $\mathcal{R}_{\rm{rec}}(\rm train)$}
\put(-12,-7){\rotatebox{90}{
\begin{minipage}{0.15 \textwidth}
\centering
(a) Fixed IC \\ Errors
\end{minipage}
}}
\end{overpic}
}
&
\adjustbox{valign=c}{
\begin{overpic}[width=0.28\textwidth]{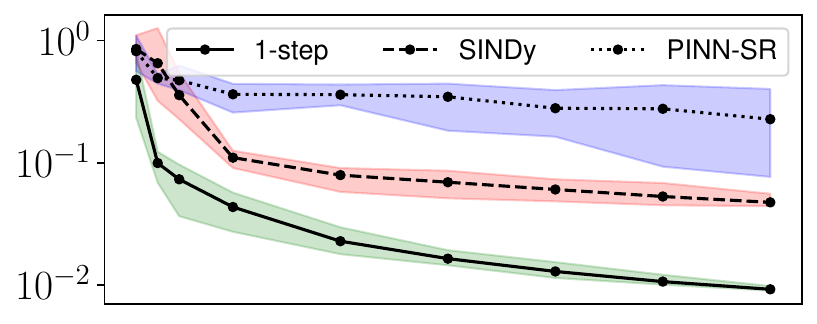}
\put(40,41){\scriptsize $\mathcal{R}_{\rm{eql}}\left(\rm train\right)$}
\end{overpic}
\hspace{8ex}
}
\\
\vspace{4ex}
\adjustbox{valign=c}{
\begin{overpic}[width=0.28\textwidth]{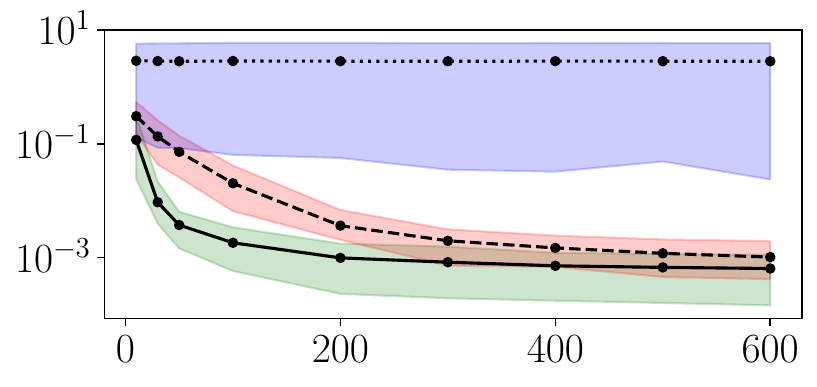}
\put(52,0){\scriptsize $N_\mY$}
\put(-12,-2){\rotatebox{90}{ 
\begin{minipage}{0.15 \textwidth}
\centering
(b) Varying IC \\ Errors
\end{minipage}
} }
\end{overpic}
}
&
\adjustbox{valign=c}{
\begin{overpic}[width=0.28\textwidth]{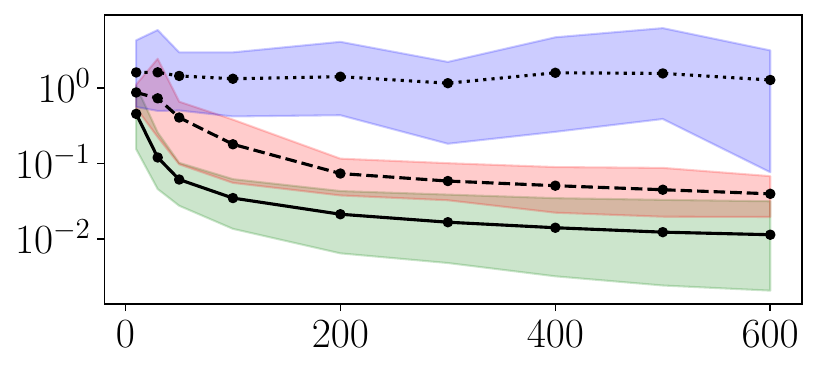}
\put(52,0){\scriptsize $N_\mY$}
\end{overpic}
\hspace{8ex}
}
\end{tabular} 
\endgroup
\\
%\vspace{1ex}
\begingroup
\setlength{\tabcolsep}{0pt}
\begin{tabular}{lccc}
\centering
\adjustbox{valign=c}{
\begin{overpic}[scale=0.2]{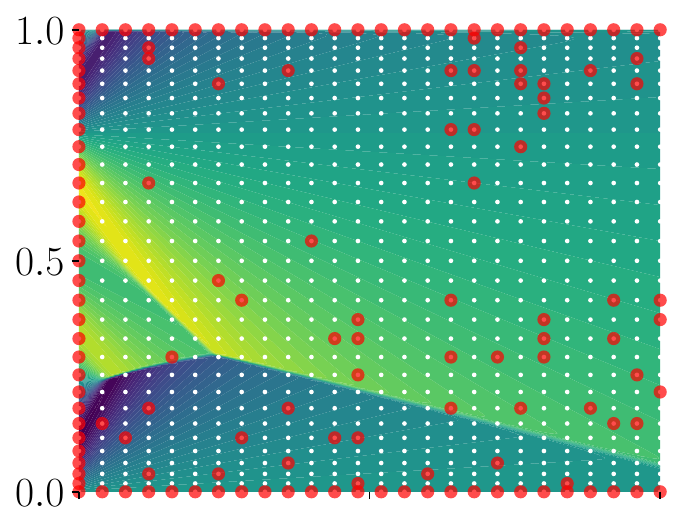}
\put(-3,37){\rotatebox{90}{$x$}}
\put(42,76){{\scriptsize True}}
\put(-23,20){\rotatebox{90}{(c) Ex. 1}}
\end{overpic}
}
&
\adjustbox{valign=c}{
\begin{overpic}[scale=0.2]{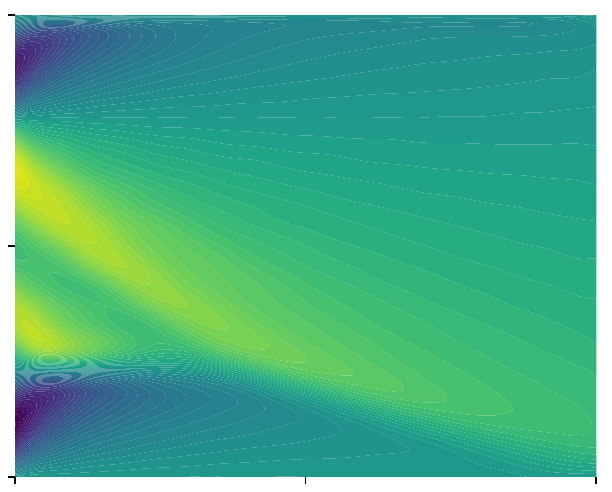}
\put(35,80){{\scriptsize 1-step}}
\end{overpic}
}
&
\adjustbox{valign=c}{
\begin{overpic}[scale=0.2]{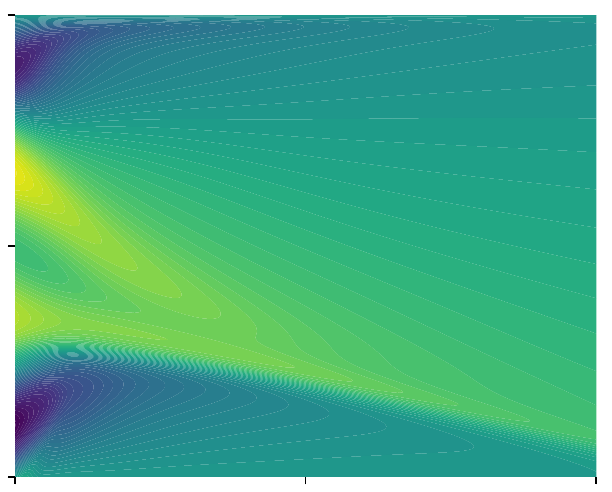}
\put(25,80){{\scriptsize PINN-SR}}
\end{overpic}
}
&
\adjustbox{valign=c}{
\begin{overpic}[scale=0.2]{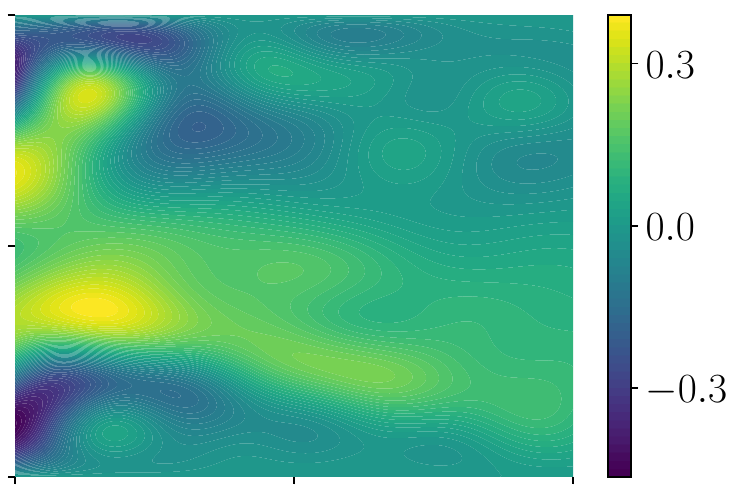}
\put(-1,66){{\scriptsize Interpolant + SINDy}}
\end{overpic}
}
\\
\adjustbox{valign=c}{
\begin{overpic}[scale=0.2]{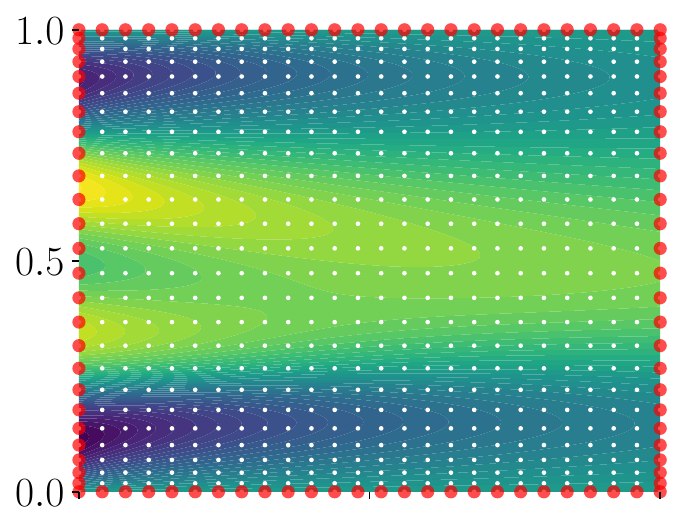}
\put(-3,37){\rotatebox{90}{$x$}}
\put(-23,20){\rotatebox{90}{(d) Ex. 2}}
\end{overpic}
}
&
 \hspace{.3ex}
 % \raisebox{.02\height}{
\adjustbox{valign=c}{
\includegraphics[scale=0.2]{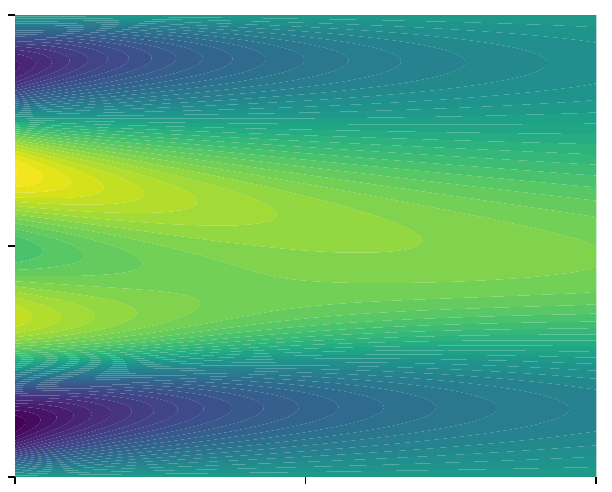}
}
% }
&
\hspace{.3ex}
\adjustbox{valign=c}{
\includegraphics[scale=0.2]{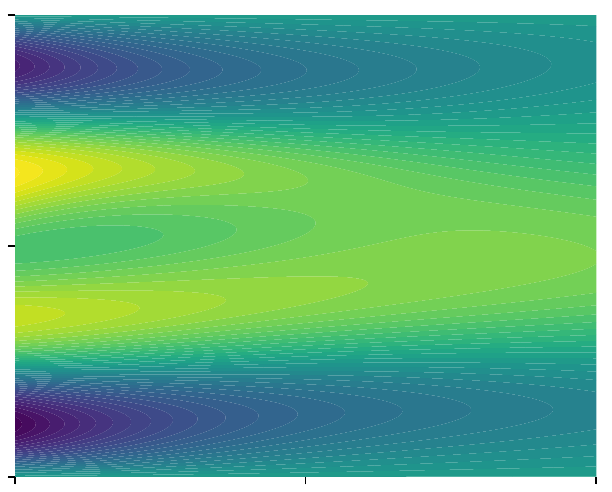}
}    
&
\hspace{.5ex}
\adjustbox{valign=c}{
\includegraphics[scale=0.2]{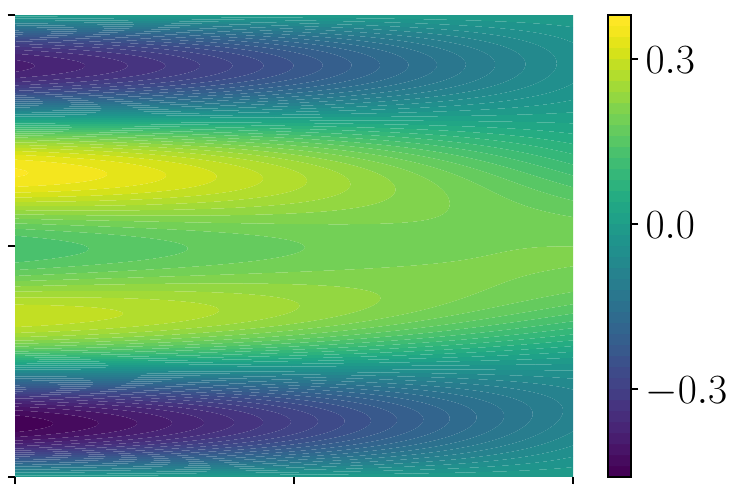}
}
\\
\hspace{.4ex}
\begin{overpic}[scale=0.2]{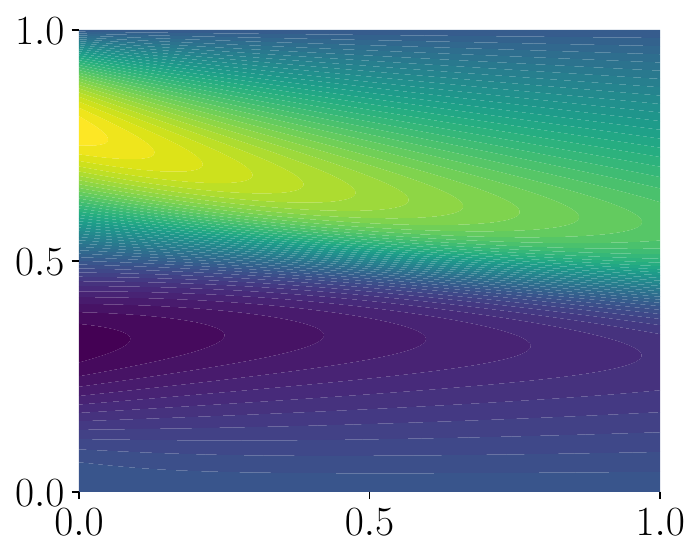}
\put(-3,40){\rotatebox{90}{$x$}}
\put(51,-3){$t$}
\put(-20,20){\rotatebox{90}{(e) New IC}}
\end{overpic}
&
\begin{overpic}[scale=0.2]{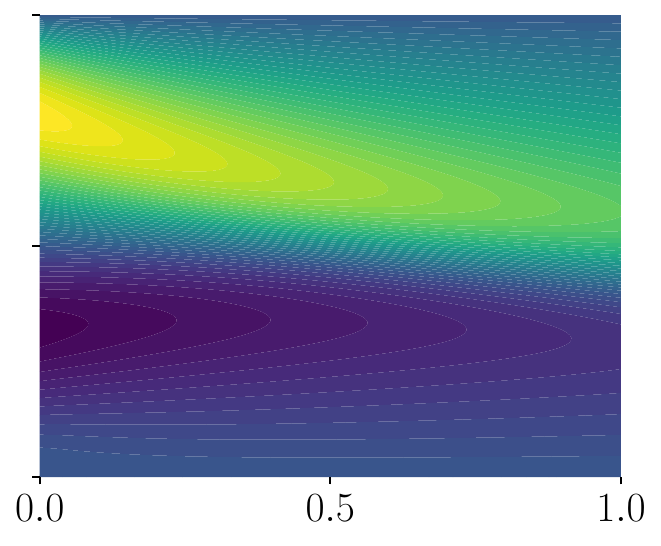}
\put(47,-3){$t$}
\end{overpic}
&
\begin{overpic}[scale=0.2]{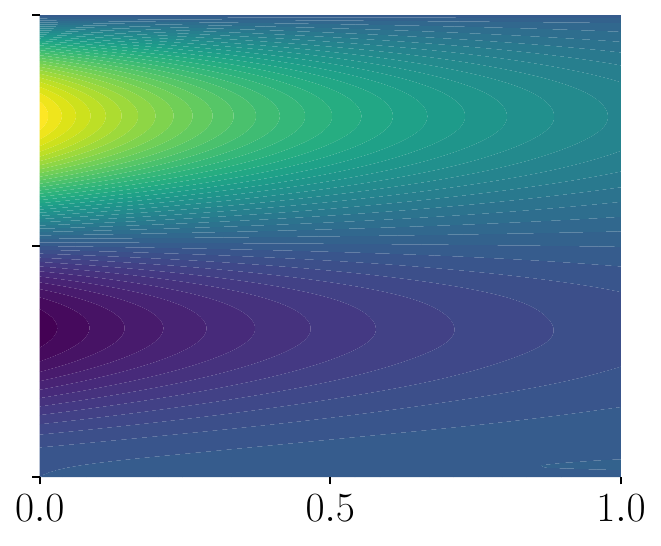}
\put(47,-3){$t$}
\end{overpic}
&
\hspace{-2.5ex}
\begin{overpic}[scale=0.2]{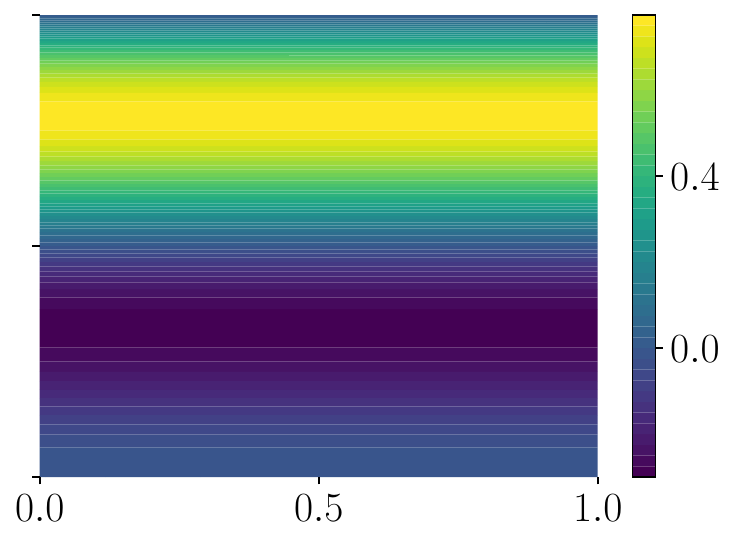}
\put(41,-3){$t$}
\end{overpic}
\end{tabular}
\endgroup
\caption{Representative numerical results for Burgers' PDE \eqref{burgers-PDE}: 
(a) \editstart
State reconstruction
\editend
and equation learning 
errors computed on the training functions for 1-step KEqL, kernel interpolant + SINDy, and 
PINN-SR using $M=1$ training pairs and varying the number
of interior observations $N_\mY$; (b) Similar experiment as panel (a) 
but with randomized initial conditions;
(c) 
\editstart
State reconstruction
\editend
results from scarce data using parameters that lead to a near shock comparing 1-step KEqL, PINN-SR and a kernel interpolation;
(d) Similar setup to row (c) with a smooth solution and observations only on the boundary; 
(e) Computed solutions to the PDEs identified by 1-step KEqL, PINN-SR, and SINDy from the data in row (d) for a new initial condition.
}
\label{fig:burgers_main}
\end{figure}

\editstart
To further investigate the effect of architecture choices on the 
performance of PINN-SR, we repeated our experiments for an 
assortment of shallower and wider  networks and compared their 
performance in state reconstruction and equation learning. These results are shown in \Cref{fig:burgers_main_pinnsr}. While we observe that some 
architectures obtain better average errors across runs,
the performance 
remains inconsistent, some runs diverge and 
the errors oscillate significantly. At the same time, the best 
performing architectures give  noticeably larger 
errors compared to  SINDy and KEqL.

\begin{figure}[htp]
\centering
\scriptsize
\begingroup
\setlength{\tabcolsep}{0pt}
\begin{tabular}{cc}
\vspace{4ex}
\adjustbox{valign=c}{
\begin{overpic}[width=0.32\textwidth]{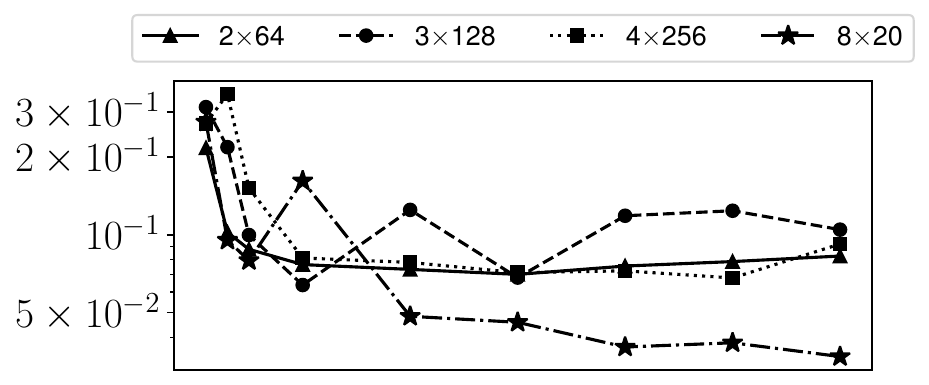}
\put(40,43){\scriptsize $\mathcal{R}_{\rm{rec}}(\rm train)$}
\put(-12,-7){\rotatebox{90}{
\begin{minipage}{0.15 \textwidth}
\centering
(a) Fixed IC \\ Errors
\end{minipage}
}}
\end{overpic}
}
&
\adjustbox{valign=c}{
\begin{overpic}[width=0.32\textwidth]{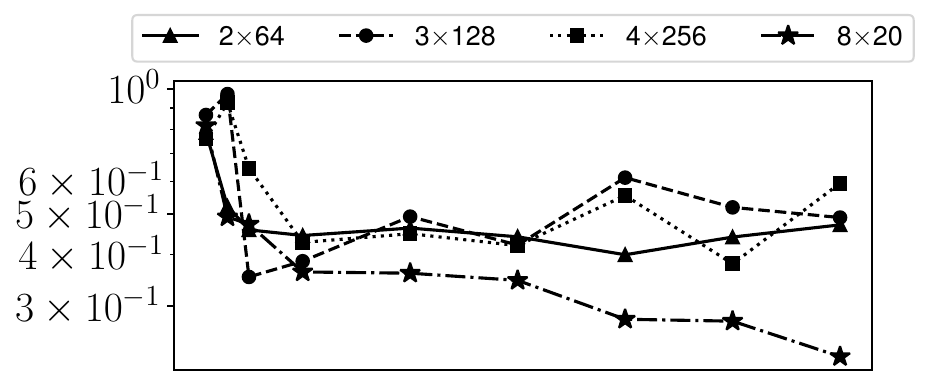}
\put(40,43){\scriptsize $\mathcal{R}_{\rm{eql}}\left(\rm train\right)$}
\end{overpic}
\hspace{8ex}
}
\\
\vspace{4ex}
\adjustbox{valign=c}{
\begin{overpic}[width=0.32\textwidth]{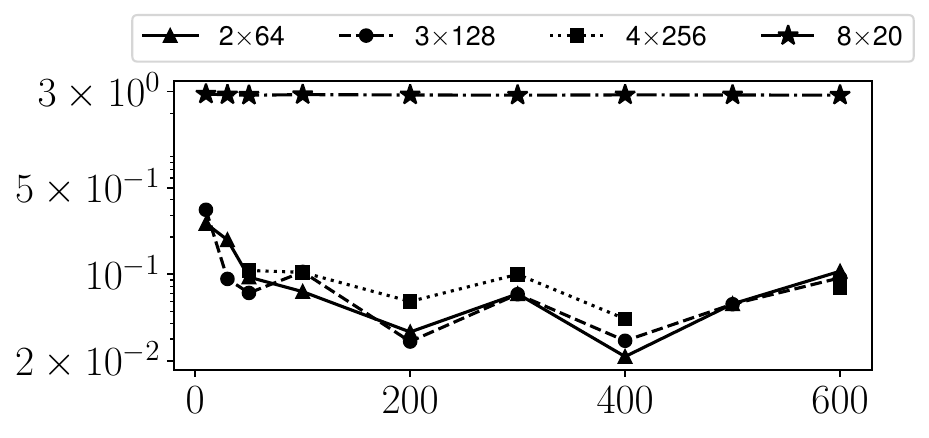}
\put(52,0){\scriptsize $N_\mY$}
\put(-12,-2){\rotatebox{90}{ 
\begin{minipage}{0.15 \textwidth}
\centering
(b) Varying IC \\ Errors
\end{minipage}
} }
\end{overpic}
}
&
\adjustbox{valign=c}{
\begin{overpic}[width=0.32\textwidth]{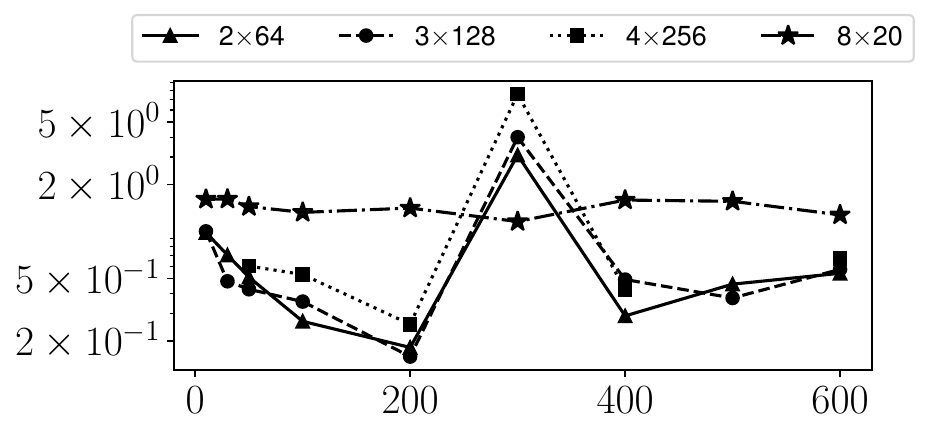}
\put(52,0){\scriptsize $N_\mY$}
\end{overpic}
\hspace{8ex}
}
\end{tabular} 
\endgroup
\caption{
\editstart
Comparing PINN-SR on Burgers' PDE \eqref{burgers-PDE} using different neural network architectures (e.g., $2 \times 64$ denotes a fully connected feed-forward network with 2 layers of 64 neurons each), complementing \Cref{fig:burgers_main} with additional architectures for PINN-SR:
(a) State reconstruction and equation learning errors on the training functions, using $M=1$ training pairs and varying the number of interior observations $N_\mY$;
(b) Same experiment as panel (a) but with randomized initial conditions.
% The reference architecture, used as the baseline and averaged over the randomized data, is the $8 \times 20$ one, which was used in \Cref{fig:burgers_main}.
Missing markers  indicate divergent experiments.
\editend
}
\label{fig:burgers_main_pinnsr}
\end{figure}

\editend

\subsection{Darcy's flow PDE}\label{sec:numerics:darcy}
Here we performed a systematic study of the 
performance of 1-step  and 2-step KEqL for learning an 
elliptic PDE with a variable diffusion coefficient.
In particular, we investigated the ID and OOD performance 
in terms of 
\editstart
state reconstruction,
\editend
equation learning, and operator learning.
Our goal here was to push the algorithm in terms of scalability 
to a large number of solution and right-hand-side pairs (large $M$) 
and very scarce observations (very small $N_\mY$).

\subsubsection{Problem setup}
Take $\mY = (0,1) \times (0,1)$ and  consider the PDE
\begin{equation}\label{darcy_PDE}
\begin{aligned}
        \P(u) = \operatorname{div}\left(a\nabla u\right) = f, \text{ in } \mY, 
        \quad 
        \text{ and } 
        \quad 
        u(y) = g(y), \text{ for } y\in \partial\mY,
\end{aligned}
\end{equation}
where $a(y) = \exp\left(\sin( \cos(y_1)+ \cos(y_2))\right)$.
The training data was generated by drawing $u$ from a smooth GP and setting $f= \P(u)$ and setting
value of $g(y) = u(y)$ on $\partial\mY$ as the boundary condition. Each $u$ was then subsampled on random observation points in the interior of a uniform collocation grid $Y$. The test data sets were generated similarly with the OOD data drawn from a GP with a 
different length scale.

\subsubsection{Algorithm setup}
We picked the operators $L_1: u \mapsto u$ , $L_2 : u \mapsto\partial_{y_1} u$, $L_3 : u \mapsto\partial_{y_1}^2 u$, $L_4 : u \mapsto\partial_{y_2} u$, $L_5 : u \mapsto\partial_{y_2}^2 u$, $L_6 : u \mapsto\partial_{y_1y_2} u$ and $\overline{P}\equiv0$. $\UK$ was taken to be RBF and $\PK$ was a hybrid 
polynomial kernel as the product of an RBF for the 
spatial variables and a polynomial kernel of first degree for 
the rest of the variables; see \Cref{sec:choosing-P} for details. For all experiments we employed 
the 1-step KEqL along with additional sparse numerical linear 
algebra tricks (see the appendices) to scale 
the algorithms to large training data sizes.

\subsubsection{Results}
\Cref{fig:darcy_main} summarizes various training and test errors 
for this example focusing on equation learning and operator learning 
errors.  The black lines represent the average errors, while the shaded regions indicate the range of errors, spanning from the worst to the best across multiple runs; see \cref{app:numerics} for more details.

We observed that the 1-step method consistently outperformed the 2-step method across all tasks both ID and OOD. 
Most notably, a in very scarce observation regime (only $N_\mY=2$  
interior observations per function), we see an order of magnitude 
performance gap between 1-step and 2-step methods across the panel. 
As expected, the gap reduces as we increase the 
observations $N_\mY$ and the size of the training data $M$.
 Finally, we note that 
the operator learning errors follow the same trends as 
the equation learning errors which suggests that 
operator learning should inherit similar rates as equation learning. 
% We will show this in our analysis in \Cref{app:theory:operator-learning-error}.

\begingroup
\setlength{\tabcolsep}{0pt}
\begin{figure}[ht!]
\centering
\scriptsize
\begin{tabular}{cc c c}
\begin{overpic}[scale=0.3]
{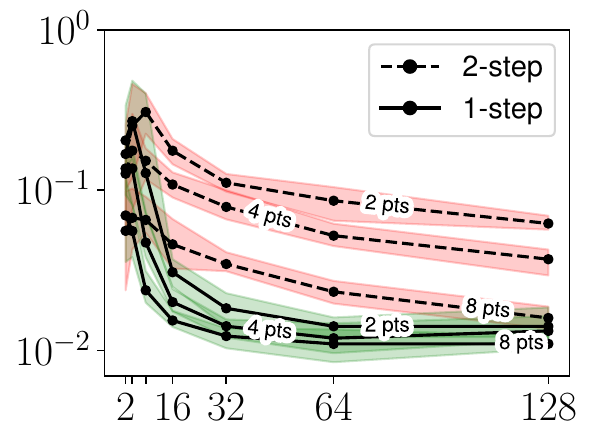}
\put(36,74){$\mathcal{R}_{\rm eql}\left(\rm{train}\right)$}
\put(52,-5){$M$}
\end{overpic}
&
\begin{overpic}[scale=0.3]{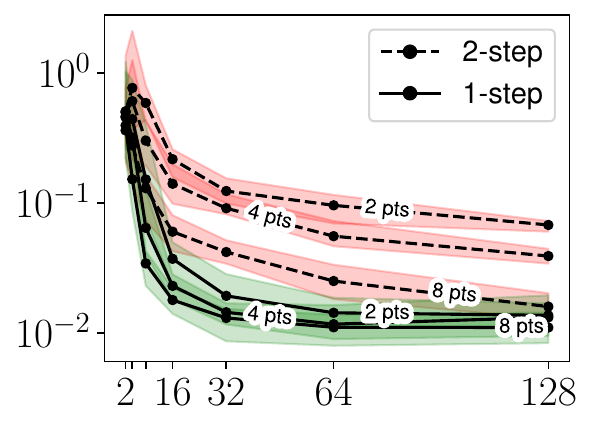}
\put(38,74){$\mathcal{R}_{\rm eql}\left(\rm{ID}\right)$}
\put(52,-5){$M$}
\end{overpic}  \vspace{5ex}
\begin{overpic}[scale=0.3]{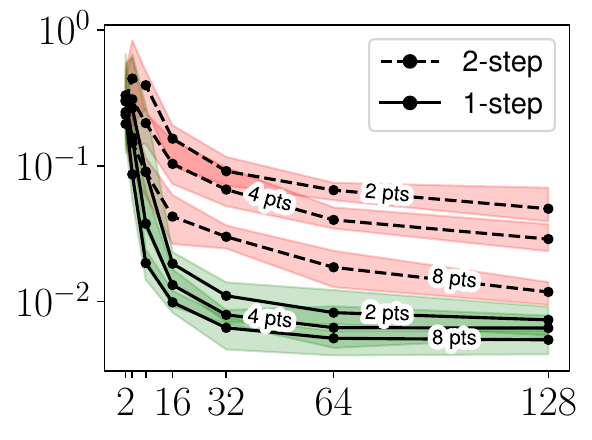}
\put(36,74){$\mathcal{R}_{\rm eql}\left(\rm{OOD}\right)$}
\put(52,-5){$M$}
\end{overpic}
&
\begin{overpic}[scale=0.3]{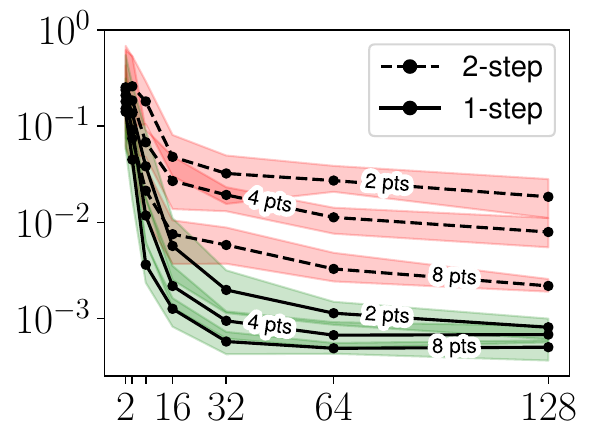}
\put(36,74){$\mathcal{R}_{\rm opl}\left(\rm{ID}\right)$}
\put(52,-5){$M$}
\end{overpic}
\end{tabular}
\caption{Representative numerical results for Darcy's flow 
PDE \eqref{darcy_PDE}: 
The top two figures from the left show the 
equation learning errors computed over training and  ID test sets. 
The OOD test errors are shown in the next row along with the ID operator learning error.
The labels on the lines denote the number of interior observation points $N_\mY$.}
\label{fig:darcy_main}
\end{figure}
\endgroup

\subsection{The effect of the dictionaries}\label{sec:choice-of-dictionaries}
\editstart
For our final set of experiments, we considered the effect of the 
choice of the differential operators $L_q$.
For brevity, we only present the results for Darcy's PDE and refer the reader to \Cref{app:numerics:burgers} for similar results 
for Burgers' PDE. 

We considered three dictionaries of differential operators
beyond our original choice for the earlier experiments. 
In particular, the exact dictionary 
$\{\partial_{y_1}u,\partial_{y_2}u,\partial_{y_1}^2u,\partial_{y_2}^2u\}$ with the same operators as the original PDE, 
a reduced dictionary 
$\{u,\partial_{y_1}u,\partial_{y_2}u\}$ that is missing second 
order terms, and an expanded dictionary 
with one extra feature 
$\{u,\partial_{y_1}u,\partial_{y_2}u,
\partial_{y_1}^2u,\partial_{y_2}^2u,\partial_{y_1y_2}u\}$.
In \Cref{tab: darcy_errs_table} we summarize the errors
with $M=32$ training functions and $N_\mY = 2$ interior observation points for our 
2-step and 1-step methods. 
Our key observation here is that the performance of both methods drops 
significantly with the reduced dictionary while the 
drop is smaller with the expanded dictionary. For 
the exact and expanded dictionaries, the 
1-step method has a clear performance advantage over the 2-step approach 
across all of our error metrics; see \Cref{tab: darcy_errs_table}.
We observed similar trends for the Burger's PDE experiments; see \Cref{app:numerics:burgers}.
 
\begin{table}
\footnotesize
    \centering
    \editstart
\begin{tabular}{cc|cc}
     & & \hspace{1cm}Method &  \\
    \hline
     & Dictionary & 1-step & 2-step \\
    \hline
     \multirow{3}{*}{$\mathcal{R}_{\text{eql}}(\rm ID)$} 
     & Exact  & \num{5.4e-3} & \num{9.0e-2} \\
     &  Reduced  & \num{1.7e0} & \num{8.1e-1} \\
     & Expanded & \num{1.4e-2} & \num{1.2e-1} \\
    \hline
    \multirow{3}{*}{$\mathcal{R}_{\text{eql}}(\rm OOD)$}
     & Exact & \num{3.0e-3} & \num{5.1e-2} \\
     & Reduced & \num{8.4e-1} & \num{9.1e-1} \\
     & Expanded & \num{7.0e-3} & \num{7.1e-2} \\
    \hline
    \multirow{3}{*}{$\mathcal{R}_{\text{opl}}(\rm ID)$}
     & Exact & \num{5.9e-4} & \num{3.4e-2} \\
     & Reduced & \num{6.2e-1} & \num{9.7e-1} \\
     & Expanded & \num{1.3e-3} & \num{3.8e-2} \\
    \hline
\end{tabular}
\editend
\caption{
    \editstart
    Relative equation  and operator learning errors for the Darcy flow PDE \eqref{darcy_PDE}, with $M=32$ training functions and $N_{\mathcal{Y}}=2$ observed interior points. Each row  reports the 1-step and 2-step errors under three feature sets: 
    the exact dictionary of operators 
    $\{u, \partial_{y_1}u,\partial_{y_2}u,\partial_{y_1}^2u,\partial_{y_2}^2u\}$ matching the true operator; the reduced dictionary $\{u,\partial_{y_1}u,\partial_{y_2}u\}$
    which misses both second derivatives; 
    and an expanded dictionary $\{u,\partial_{y_1}u,\partial_{y_2}u,\partial_{y_1}^2u,\partial_{y_2}^2u,\partial_{y_1y_2}u\}$ with two extra features.
    \editend
    }
    \label{tab: darcy_errs_table}
\end{table}

\editend

% \textcolor{red}{I think it would be good to add a short summary here; it could include identifiability of single trajectories or identifiability in sparse regimes and the bias the correct prior induced, why the method is overall performing better than the benchmarks, etc.}

\section{Theoretical analysis}\label{sec:theory}
In this section we turn our attention to the theoretical analysis of 
KEqL and in particular obtaining quantitative error bounds for 
the functions $\wh u^m$ as well as the operator $\wh \P$ 
\editstart
and solutions of the approximate PDE $\wh \P^\dagger(u) = f$.
\editend
% Here we collect our main theoretical results concerning the 
% convergence of KEqL in the form of quantitative error bounds 
% for the learned functions $\whu^m$ and the learned PDE $\whP$.
We begin with a brief review of preliminaries in \Cref{sec:theory:prelims}. 
\editstart
State reconstruction
\editend
and equation learning errors for KEqL are presented in 
\Cref{sec:rec-eql-error-analysis}. The subsequent operator learning error analysis is 
summarized in \Cref{sec:operator-learning-error-analysis}.
We omit some known or repetitive   arguments to save space; interested readers can find 
additional theoretical details in \Cref{app:theory-details}.
\editstart
We note that our analysis does not consider 
optimization gaps due to numerical algorithms and 
focuses on  ideal settings where global minimizers
are available.
% depends only on the optimization problems themselves and is independent of their numerical implementation or the algorithms used to solve them. Consequently, the resulting error analysis is \textit{a priori}.
\editend

\subsection{Preliminaries}\label{sec:theory:prelims}

\editstart
We begin by collecting useful preliminary results relating to the theory of RKHSs 
and Sobolev spaces. 
% I wrote this out because RKHSs looks awkward as a \paragraph
\paragraph{Reproducing Kernel Hilbert Spaces}\label{app:theory:RKHS-review}

We assume familiarity with the standard results in the theory of RKHSs; these results
can be found in many classic texts such as \cite{scholkopf2018learning,owhadi2019operator,steinwart2008support,berlinet2011reproducing};
for convenience, we also collect some basic facts about RKHSs in 
\Cref{sec:SM-theory-prelims}.
Consider a set $\Omega \subset \R^q$ and a
continuous kernel $\HK:\Omega \times \Omega \to \R$ with RKHS $\mH$. Then by 
Mercer's theorem (\Cref{prop:Mercers-thm}), 
this kernel has a spectral expansion, i.e., $\HK(x, x') = \sum_{i=1}^\infty \lambda_i e_i(x) e_i(x')$
with positive eigenvalues $\lambda_i$ (assumed to be decreasing) and eigenfunctions $e_i$ that are orthonormal in \(L^2(\Omega)\).
With this setup in mind,  we can define the family of RKHSs
    \begin{equation}\label{def:H-gamma}
        \mH^\gamma = \left\{ f: \Omega \to \R\, \Bigg|\, f(x) 
        = \sum_{i \in \{ i \mid \lambda_i \neq 0\}} c_i(f)  e_i(x), \;\;
        \sum_{i \in \{i \mid \lambda_i \neq 0 \}} \lambda_i^{-\gamma}c_i(f)^2  < +\infty
        \right\}
    \end{equation}
    for $\gamma >0$. Additionally, for any pair $f,f' \in \mH^\gamma$ we have  
    $\langle f, f' \rangle_{\mH^\gamma} =  \sum_{i \in \{ i \mid \lambda_i \neq 0\} } \lambda_i^{-\gamma} c_i(f) c_i(f')$. 
    In particular, the RKHS $\mH$ is recovered as $\mH=\mH^1$.
 Naturally, larger 
values of $\gamma$ imply more ``smoothness", in particular we have 
the inclusion $\mH^{\gamma_2} \subset \mH^{\gamma_1}$ for $1 \le \gamma_1 < \gamma_2$
following H\"older's inequality.
Observe that our definition of the $\mH^\gamma$ resembles the 
spectral definition of Sobolev spaces $H^\gamma(\Omega)$ on compact sets 
in which case the $\{\lambda_i, e_i\}_{i=1}^\infty$ can be taken as 
the eigenpairs of the inverse of the Laplacian operator.
The following lemma is useful in our proofs later in this section; 
the proof is given in \Cref{sec:SM-RKHS-review}.

\begin{lemma}\label{lem:H-gamma-bound}
    Suppose $f\in \mH^{2\gamma}$ and $f' \in \mH^\gamma$. Then 
        $\langle f', f \rangle_{\mH^\gamma} \le 
        \| f'\|_{L^2(\Omega)} \| f\|_{\mH^{2\gamma}}$.
\end{lemma}
% \textcolor{blue}{Throughout the remainder of the theoretical analysis section, whenever we write $\mH^\gamma$ for $\gamma\in\N$, we mean the nested RKHS, not the product space.}
\paragraph{Sobolev spaces}\label{app:theory:sampling-inequality}

We present some essential background on Sobolev spaces, including the version of the Sobolev embedding theorem used in our proofs in \Cref{sec:SM-Sobolev-review}.
\editend
In addition to the embedding theorem, we will heavily rely 
on the Sobolev sampling inequality which gives control 
over the Sobolev norm of a function that is small or zero 
on a scattered set of points. For an introduction to Sobolev sampling inequalities and related results on 
scattered data approximation see  \cite{wendland2004scattered}. The following theorem 
is a distillation of \cite[Prop.~2.4]{wendland2005approximate} in 
the form that we need in our proofs; see also \cite{narcowich2006sobolev, le2006continuous, arcangeli2007extension}.

\begin{proposition}[Sobolev sampling inequality]\label{prop:sobolev-sampling-inequality}
Suppose $\Omega \subset \R^d$ is a bounded set with Lipschitz 
boundary and consider a set of points
 $X = \{ x_1, \dots, x_N \} \subset \overline{\Omega}$
with 
fill distance  $h_X:= \sup_{x \in \Omega} \inf_{x' \in X} \| x - x' \|_2$.
Let $u|_X$ denote the restriction of $u$ to the set $X$.
Further consider  $\gamma > d/2$ and $0 \le \eta \le \gamma$ and 
let $u \in H^\gamma(\Omega)$. 
\begin{enumerate}[label=(\alph*)]
    \item (Noiseless) Suppose $u|_X = 0$. 
Then there exists $h_0>0$ so that whenever $h_X \le h_0$ we have 
    $\| u \|_{H^\eta(\Omega)} \le C_\Omega h_X^{\gamma - \eta} \| u \|_{H^\gamma(\Omega)}$,
where $C_\Omega >0$ is a constant that depends only on $\Omega$.
\item (Noisy) Suppose $u|_X \neq 0$.
Then there exists $h_0>0$ so that whenever $h_X \le h_0$ we have the inequality 
    $\| u \|_{L^\infty(\Omega)} 
    \le C_\Omega h_X^{\gamma - d/2} \| u \|_{H^\gamma(\Omega)} 
    + 2 \| u|_X \|_\infty,$
where $C_\Omega >0$ is a constant that depends only on $\Omega$.
\end{enumerate}
\end{proposition}

\subsection{Error analysis for equation learning}\label{sec:rec-eql-error-analysis}
We now turn our attention to error bounds for $\wh u^m$ and $\wh P$, i.e., 
the \editstart state reconstruction \editend and equation learning errors. We will present our proofs for 1-step KEqL
 assuming that there is no noise in the data or nugget terms in the formulation. 
 \editstart
We also assume homogeneous boundary conditions in \eqref{generic-DE}, i.e., $\B(u) = 0$.
 \editend
 We give further theoretical results in the appendices including a simpler convergence 
 proof for 1-step KEqL that can generalize beyond the RKHS formulation as well as 
 an analogous error bound for 2-step KEqL.

Let us begin by presenting our main assumptions that will be used in the rest of this section.
First we have a standard assumption on the set $\mY$ that allows us to use 
\editstart
Sobolev embedding theorem and sampling inequalities.
\editend
\begin{assumption}\label{assumption:domain}
    The set $\mY \subset \R^{d}$ is bounded and has Lipschitz boundary. 
\end{assumption}
Next we will need assumptions on the kernels $\UK$ and $\PK$ 
to ensure sufficient regularity and continuous embedding in appropriate 
Sobolev spaces:
\begin{assumption}\label{assumption:UU}
The kernel $\UK: \mY \times \mY \to \R$ and its corresponding RKHS satisfy:
\begin{enumerate}[label=(\roman*)]
    \item $\UK$ is PDS and continuous in its arguments. 
    \item $\UU$ is continuously embedded in $H^\gamma(\mY)$ for 
    some $\gamma > d/2 + \text{order}(\P)$. In particular, there exists $C_\mY >0$ such 
    that $\| u\|_{H^\gamma(\mY)} \le C_\mY \| u\|_{\UU}$ for all $u \in \UU$.
    \item $\B(u) = 0$ on $\partial \mY$ for all $u \in \UU$, i.e., 
    \editstart
    our assumed homogeneous boundary conditions hold.
    \editend
\end{enumerate}
\end{assumption}

\begin{remark}
    We highlight that \Cref{assumption:UU}(iii)  allows us to 
    simplify our theoretical arguments significantly since we do not 
    need to approximate values of the functions on the boundary. However, this 
    assumption can be removed by adding a separate approximation 
    result for the value of estimated functions $\whu^m$ 
    near the boundary under sufficient regularity assumptions.
\end{remark}

\begin{assumption}\label{assumption:PP}
    The kernel $\PK: \mS \times \mS \to \R$ and its corresponding 
    RKHS satisfy: 
\begin{enumerate}[label=(\roman*)]
    \item $\PK$ is PDS. 
    \item $\PP$ is continuously embedded in $H^\eta(\mS)$ for some 
    $\eta > \frac{Q + d}{2}$.
    \item Elements of $\PP$ are locally Lipschitz, i.e., 
    for any compact set $B \subset \mS$ there exists a 
    constant $C(B) >0$ so that 
    $| P(s) - P(s') | \le C(B) \| P\|_{\PP} \| s - s'\|_2$, for all $s,s' \in B$.
\end{enumerate}
\end{assumption}
Since we will be taking the limit of $N$ (number of observation points) 
and $M$ (size of training data) to infinity, we will supplement 
our notations in this section by subscripting pertinent objects with $M, N$ indices.
In particular we will write $\wh u^m_{M,N}$ to denote the KEqL estimators 
for $u^m$ with $N$ observations on the points $Y^m_N$ and $M$ training pairs as well as 
$\wh P_{M, N}$ for the KEqL estimator of $P$. 
% We further write $Y_{M,N}$ to denote
% the collocation mesh that was used to compute these estimators.
With this notation, we further define the fill distances 
\begin{equation*}
    \rho_{m,N} = \sup_{y \in \mY} \inf_{y' \in Y^m_N} \| y - y' \|_2, 
    \quad \varrho_{M, N}(B) = \sup_{s \in B} \inf_{s' \in S \cap B} \| s - s'\|_2,
\end{equation*}
where $S:= \cup_{m=1}^M\cup_{y \in Y_N} \Phi(u^m, y)$. 
\editstart
We additionally define $\rho_{M,N} := \sup_{1 \le m \le M} \rho_{m,N}.$
\editend    

\subsubsection{Error analysis for 1-step KEqL}\label{app:theory:errorbound-1-step}
The following theorem gives a quantitative error bound for 1-step KEqL and 
constitutes our main theoretical contribution.

\begin{theorem}\label{thm:1-step-error-estimate}
    Consider the problem \eqref{one-shot-optimal-recover} 
    with $M,N \in \mathbb{N}$. Suppose 
    \Cref{assumption:domain,assumption:UU,assumption:PP}  hold and
    $P, \overline{P} \in \PP$. Consider 
    pairs $\{ u^m, f^m\}_{m=1}^M$ 
    satisfying \eqref{generic-DE},
    and a bounded set $B \subset \mS$ with Lipschitz boundary. 
    Then there exist constants $\rho_0, \varrho_0(B) \in (0, 1)$ so that 
    whenever $\rho_{m,N} < \rho_0$ and $\varrho_{M,N}(B) < \varrho_0(B)$
    it holds that:
    \begin{enumerate}[label=(\alph*)]
        \item  If $u^m \in \UU$ then it holds for $0\le \gamma' < \gamma$ and a constant $C>0$, depending only on $\mY$, that 
        \begin{equation*}
        \begin{aligned}
           \editstart
           \frac 1M
           \editend
           \sum_{m=1}^M  \| \whu^m_{M,N} - u^m \|^2_{H^{\gamma'}(\mY)} 
             \le C 
            \editstart
            \rho_{M,N}^{2(\gamma - \gamma')} 
            \editend
            \left( \| P \|^2_{\PP} + 
            \editstart
            \frac{1}{M}
            \editend
            \sum_{m=1}^M \| u^m \|^2_{\UU} \right). 
        \end{aligned}
    \end{equation*}

    \item If $u^m \in \UU^2$  then  it holds for $d/2 + \text{order}(\P) < \gamma' < \gamma$ and a 
    constant $C >0$, depending only on $\mY$ and $B$, that
    \begin{equation*}
       \begin{aligned}
        & \| P - \whP_{M,N} \|_{L^\infty(B)}  \\
    & \qquad \le C  \left( \varrho_{M, N}(B)^{\eta - \frac{Q + d}{2}}
    +  \editstart
    \rho_{M, N}^{\gamma - \gamma'}
    \editend
    \right) 
    \left( \| P \|_\PP^2 + \| \overline{P} \|_\PP^2 + 
     \editstart
     \frac{1}{M}
     \editend
     \sum_{m=1}^M \| u^m \|_{\UU^2}^2 \right)^{1/2}.
       \end{aligned}
    \end{equation*}
    \end{enumerate}

\end{theorem}

We observe that both of our bounds closely resemble the Sobolev sampling inequalities 
from \Cref{prop:sobolev-sampling-inequality}. Indeed, statement (a) is natural in that 
light since we have access to pointwise observations of the $u^m$ on the sets $Y^m_N$. 
The only technicality in the proof of this result is that the $\wh u^m_{M,N}$ are not 
obtained by solving individual optimal recovery problems but rather a joint problem for all of the 
$u^m$ as well as $P$. This is why the sum of those RKHS norms appears on the right hand side. The proof 
of statement (b) is more complex since our recovery problem for $P$ is polluted by 
two sources of error: (1) error due to limited observations of the PDE on a collocation mesh
 and (2) errors due to approximation of $u^m$ with $\wh u^m_{M,N}$ that are then 
 plugged into $P$ via the operator $\Phi$. Essentially, we are dealing with a regression problem 
 where the input and output of the function are noisy. This is why we need to assume that elements of 
 $\PP$ are  Lipschitz in \Cref{assumption:PP}. 

\begin{proof}
For simplicity we will present the proof with the regularization 
parameter $\lambda =1$. Other choices will only affect the 
constant $C$ in the bounds. 

{\it (a)}
 Observe that the equality constraints 
$\whu^m_{M,N}(Y^m_N) = u^m(Y^m_N)$ in \eqref{one-shot-optimal-recover} (with our new 
supplemented notation)
simply state that the $\whu^m_{M,N}$ are 
interpolating the $u^m$ and so we can apply \Cref{prop:sobolev-sampling-inequality}(a). Indeed, 
directly applying that result 
followed by \Cref{assumption:UU}(ii)
we obtain, for $\rho_{m,N} < \rho_0$, the bounds 
\begin{equation*}
\begin{aligned}
    \| \whu^m_{M,N} - u^m \|_{H^{\gamma'}(\mY)} 
    \le C \rho_{m,N}^{\gamma - \gamma'}\, \| \whu^m_{M,N} - u^m \|_{H^{\gamma}(\mY)}
    \le C \rho_{m,N}^{\gamma - \gamma'} \, \| \whu^m_{M,N} - u^m \|_{\UU},
\end{aligned}
\end{equation*}
with the constant $C>0$ changing from one inequality to the next 
and depending only on the domain $\mY$ and the choice of $\UU$ and $\gamma$.
Using the inequality above,the identity $(a + b)^2 \le 2 \, (a^2 + b^2)$,
and the optimality of $\wh u^m_{M,N}$, we can write 
\begin{equation*}
    \begin{aligned}
        \editstart
        \frac 1 M
        \editend
        \sum_{m=1}^M \| \whu^m_{M,N} - u^m \|^2_{H^{\gamma'}(\mY)}
        & \le 
        \editstart
        \frac CM\, \rho_{M, N}^{2(\gamma -\gamma')}
        \editend
        \left(  \sum_{m=1}^M \| \whu^m_{M,N}\|_\UU^2 
        + \sum_{m=1}^M \| u^m \|_\UU^2 \right), \\ 
        & \le C \,
        \editstart
        \rho_{M, N}^{2(\gamma -\gamma')}
        \editend
        \left( \| P\|_\PP^2 
        + 
        \editstart
        \frac 1 M
        \editend
        \sum_{m=1}^M \| u^m \|_\UU^2 \right).
    \end{aligned}
\end{equation*}
This concludes the proof of the first inequality. 

{\it (b)}
Fix $M$ and for each $u^m$ define the optimal interpolant 
    \begin{equation}\label{u-bar-def}
        \overline{u}^m_N:= \argmin_{v \in \UU} \| v\|_\UU
        \quad \text{s.t.} \quad
        v(Y^m_N) = u^m(Y^m_N).
    \end{equation}
    By optimality, we have the bound 
%    \begin{equation*}\label{disp-00}
        $\| \whu^m_{M,N} \|_\UU \ge \| \overline{u}^m_N \|_{\UU}.$
        %    \end{equation*}
    At the same time,  \Cref{prop:generalized-rep-theorem}
    gives $\overline{u}^m_N = \UK(Y^m_N, \cdot)^T \UK(Y^m_N, Y^m_N)^{-1} u^m(Y^m_N)$ and a direct calculation using 
    the RKHS reproducing property, 
    \editstart 
    i.e. $\langle u, \UK(x, \cdot) \rangle_\UU = u(x)$ for all $u \in \UU$,
    \editend
    implies that 
    $\langle u^m - \overline{u}^m_N, \overline{u}^m_N \rangle_\UU =0$,
    which in turn gives the identity 
        $\| u^m - \overline{u}^m_N \|_\UU^2 = \| u^m\|_{\UU}^2 - \| \overline{u}^m_N\|_\UU^2.$

Thus we obtain the lower bound
\begin{equation}\label{disp-000}
    \| \whu^m_{M,N} \|_\UU^2 \ge \| u^m\|_\UU^2 - \| u^m - \overline{u}^m_N\|_\UU^2.
\end{equation}
On the other hand the optimality of $\wh P_{M, N}$ and the $\wh u^m_{M,N}$ gives 
$ \| \whP_{M,N} \|^2_{\PP} \le \| P \|_\PP^2 + 
    \editstart
    \frac 1 M
    \editend
    \sum_{m=1}^M \big( \| u^m \|_\UU^2 - \| \whu^m_{M,N} \|_\UU^2 \big),$
which together with \eqref{disp-000} yields 
\begin{equation}\label{disp-3}
        \| \whP_{M,N} \|^2_{\PP} \le \| P \|_\PP^2 + 
    \editstart
    \frac 1 M
    \editend
    \sum_{m=1}^M  \| u^m - \overline{u}^m_N \|_\UU^2.
\end{equation}
Applying \Cref{lem:H-gamma-bound} (thanks to the assumption that $u^m \in \UU^2$)
then gives the bound $\| u^m - \overline{u}^m_N\|_{\UU}^2 
\le \| u^m - \overline{u}^m_N\|_{L^2(\mY)} \| u^m\|_{\UU^2}$.
Further applying  \Cref{prop:sobolev-sampling-inequality}(a)
(since $\overline{u}^m_N$ interpolates $u^m$)
to control   $\| u^m - \overline{u}^m_N\|_{L^2(\mY)}$
we obtain the bound $\| u^m - \overline{u}^m_N\|_{\UU}^2 
\le C \rho_{m, N}^\gamma \| u^m \|_{\UU} \| u^m\|_{\UU^2}$. Substituting 
into \eqref{disp-3} 
\editstart
and using the embedding $\UU^2 \subset \UU$ 
\editend 
gives 
\begin{equation}\label{disp-5}
        \| \whP_{M,N} \|^2_{\PP} \le \| P \|_\PP^2 + 
    \editstart
    \frac C M
    \editend
    \sum_{m=1}^M \rho_{m, N}^\gamma \| u^m \|_\UU  \| u^m \|_{\UU^2}
    \le \editstart
    \frac C M \sum_{m=1}^M \rho_{m, N}^\gamma \| u^m \|^2_{\UU^2}.
    \editend
\end{equation}
Now we turn our attention to controlling the error between $\whP_{M,N}$ and $P$. 
Observe that the PDE constraint in \eqref{one-shot-optimal-recover}
implies that, for all $k$ and $m$,
$    (\whP_{M,N} + \overline{P} )(s_k(\whu^m_{M,N}))
    =  (P + \overline{P} )(s_k(u^m)), $
where we recall $s_k(v) = \Phi(v, y_k)$ is linear in $v$.
Thus, subtracting $(\whP_{M,N} + \overline{P})(s_k(u^m))$ from both sides of the above equation
we obtain the bound
\begin{equation*}
\begin{aligned}
    &\left| (P - \whP_{M,N}) (s_k(u^m))  \right| \\ 
    & \qquad = \left| \whP_{M,N}(s_k(\whu^m_{M,N} ) ) - \whP_{M,N} ( s_k(u^m)) 
    + \overline{P}(s_k(\whu^m_{M,N} ) ) - \overline{P}( s_k(u^m))  \right| \\
    &\qquad \le \left| \whP_{M,N}(s_k(\whu^m_{M,N} ) ) - \whP_{M,N} ( s_k(u^m))\right| 
    + \left| \overline{P}(s_k(\whu^m_{M,N} ) ) - \overline{P}( s_k(u^m))  \right|.
\end{aligned}
\end{equation*}
Now consider indices $k,m$ such that $s_k(u^m)$ and 
$s_k(\whu^m_{M,N})$ belong to $B$.
Then by the Lipschitz assumption on $\whP_{M,N}$ and $\overline{P}$ (i.e., 
\Cref{assumption:PP}(iii)) we have that 
\begin{equation*}
    \left| (P - \whP_{M,N}) (s_k(u^m))  \right|^2 \le C(B) \left( \| \whP_{M,N} \|^2_\PP 
    + \| \overline{P} \|^2_\PP \right) \| s_k (u^m - \whu^m_{M,N}) \|^2_2.
\end{equation*}
Applying  \eqref{disp-5} to control $\|\whP_{M,N} \|^2_{\PP}$
and assuming $\rho_{m,N} < \rho_0$,
we can write 
\begin{equation*}
\begin{aligned}
        &\left| (P - \whP_{M,N}) (s_k(u^m))  \right|^2  \\ 
        & \qquad \le C 
        \left( \| P \|^2_\PP 
    + \| \overline{P} \|^2_\PP +  \editstart \frac 1 M \editend \sum_{m=1}^M \rho_{m, N}^\gamma 
    \editstart
    \| u^m \|^2_{\UU^2}
    \editend
    \right) \| s_k (u^m - \whu^m_{M,N}) \|^2_2.
\end{aligned}
\end{equation*}
It remains for us to control the error $\| s_k (u^m - \whu^m_{M,N}) \|^2_2$. Thanks 
to \Cref{assumption:UU}(ii) and  
\editstart
Sobolev embedding theorem,
\editend
we have for any $\gamma'$ satisfying 
$d/2 + \text{order}(\P)< \gamma' \le \gamma$ that 
$ \| s_k(v) \|_2 \le C \| v \|_{H^{\gamma'}}$. This, together with statement (a) gives the bound 
\begin{equation}\label{disp-6}
\begin{aligned}
        \left| (P - \whP_{M,N}) (s_k(u^m))  \right|^2 & \le C 
        \left( \| P \|^2_\PP 
    + \| \overline{P} \|^2_\PP +
    \editstart
    \frac 1 M \sum_{m=1}^M \rho_{m, N}^\gamma  
    \| u^m \|^2_{\UU^2}
    \right)  \\ 
    & \qquad \times \rho_{M, N}^{2(\gamma -\gamma')} \left( \| P \|_\PP^2 + 
    \frac 1 M\sum_{m=1}^M \| u^m \|_\UU^2\right) \\
    & \le  C \rho_{M, N}^{2(\gamma - \gamma')}  
    \left( \| P \|_\PP^2 + \| \overline{P} \|_\PP^2 
    + \frac 1 M\sum_{m=1}^M \| u^m\|_{\UU^2}^2 \right)^2,
    \editend
    \end{aligned}
\end{equation}
where we assumed $\rho_{m,N} < 1$ for the second display.

At this point, we have shown that $P$ and $\whP_{M,N}$ are close 
on the discrete set $S$. A direct application of 
the noisy Sobolev sampling inequality \Cref{prop:sobolev-sampling-inequality}(b) under the assumption that $\varrho_{M,N}(B) \le \varrho_0$ gives  
\begin{equation}\label{disp-7}
\| P - \whP_{M,N} \|_{L^\infty(B)} 
\le C {\varrho_{M, N}(B)}^{\eta - \frac{Q + d}{2}} 
\| P - \whP_{M,N} \|_{H^\eta(B)} + 2 \, \| (P - \whP_{M,N} ) |_S \|_\infty.
\end{equation}
Thanks to \Cref{assumption:PP}(ii) and the bound \eqref{disp-5}
we can further bound the Sobolev norm in the first term on the right-hand side, 
\begin{equation*}
\begin{aligned}
\| P - \whP_{M,N} \|_{H^\eta(B)} 
& \le \| P - \whP_{M,N} \|_{H^\eta(\mS)} 
\le \| P \|_\PP + \| \whP_{M,N} \|_\PP \\
& \le C \left( \| P \|_\PP^2 +  
\editstart
\frac 1 M \sum_{m=1}^M \rho_{m, N}^\gamma 
 \| u^m \|^2_{\UU^2} 
 \editend
 \right)^{1/2}.
\end{aligned}
\end{equation*}
Taking the supremum over  $m$ under the sum, substituting 
this bound alongside \eqref{disp-6} into \eqref{disp-7}, 
and using the assumption that 
\editstart
$\rho_{m,N} \le 1$, 
\editend
we obtain the desired inequality

\begin{equation*}
\begin{aligned}
    &\| P - \whP_{M,N} \|_{L^\infty(B)} 
    \le C \Bigg[ \varrho_{M, N}(B)^{\eta - \frac{Q + d}{2}} 
    \left( \| P \|_\PP^2 +
    \editstart
    \frac {\rho_{M, N}^\gamma} M \sum_{m=1}^M  \| u^m \|^2_{\UU^2} 
    \editend
    \right)^{1/2} \\
& \qquad +   
\editstart
\rho_{M, N}^{\gamma - \gamma'} 
\editend
\left( \| P \|_\PP^2 + \| \overline{P} \|_\PP^2 
    + \editstart
    \frac 1 M
    \editend
    \sum_{m=1}^M \| u^m\|_{\UU^2}^2  \right)^{1/2} 
\Bigg]\\
    & \qquad
     \le C \left( {\varrho_{M, N}(B)}^{\eta - \frac{Q + d}{2}}
    + \editstart
    \rho_{M, N}^{\gamma - \gamma'} 
    \editend
    \right) 
    \left( \| P \|_\PP^2 + \| \overline{P} \|_\PP^2 + 
    \editstart
     \frac 1 M
     \editend
     \sum_{m=1}^M \| u^m \|_{\UU^2}^2 \right)^{1/2}.
\end{aligned}
\end{equation*}
{}
\end{proof}

\begin{remark}
\editstart
We emphasize that the above error bounds are quite strong, and give pointwise
control over the errors of the $\widehat u^m_{M,N}$ and $\widehat P_{M,N}$
over subsets of their domains, provided that the training data is sufficiently
space-filling. This is in line with the theory of scattered data approximation
and suggests various avenues for experimental design of equation learning
algorithms. Furthermore, the bound for $\widehat P_{M,N}$ implies that the
learned PDE is accurate for new input-output pairs $(u,f)$ that are close to
the training data in an appropriate sense, i.e., if $\Phi(u,\cdot)\in B$.
\editend
\end{remark}

\begin{remark}
    Our error analysis  can be extended in various directions 
    to incorporate other types of 
    problems that may be encountered in practice. 
   \editstart
   For example, we give worst case error bounds that may be too pessimistic.  It is 
   therefor  natural  to consider a probabilistic model,
 i.e., the GP regression approach, that allows us to obtain 
 error bounds in expectation or with high-probability that may be more practical.
 Such error bounds 
    already exist for GP regression 
    \editend
    in the literature and we refer the 
    interested reader to \cite{moriarty2023convergence, wynne2021convergence} for further reading. We also highlight that our proof technique can be applied 
    beyond the particular instance of equation learning to any compositional 
    model of multiple RKHS functions such as those encountered in 
    computational graph completion \cite{owhadi2022computational}. 
    Finally, we can view our equation learning problem as a 
    nonlinear  total least squares problem \cite{beck2006solution, beck2006finding} connecting our analysis to a broader set of approximation 
    problems with noise in both  inputs and outputs.
\end{remark}

\subsection{Error analysis for equation-mediated operator learning}\label{sec:operator-learning-error-analysis}

We now extend our equation learning error bound to an operator learning bound. 
\editstart
Operator learning error rates are typically stated as convergence guarantees for learned operators between function spaces, rather than bounds for a single fixed input. They often control the worst-case error over a compact input set, as in neural-operator universal approximation results \cite{kovachki2023neuraloperator}, or the average error under an input-function distribution, as in DeepONet analyses of approximation and generalization error \cite{lanthaler2022deeponeterrors}.
% In our notation, this would amount to obtaining an average or supremum error bound for $\wh \P^{\dagger}$ over some space of operators or over a certain set of 
% right hand side functions $f$.
% function space containing $\PP$.
Through our equation-mediated operator learning method, we leverage the learned $\whP$ as a local differential operator, then numerically ``invert" it to evaluate $\wh\P^{\dagger}$, thereby obtaining an operator learning error bound that is uniform over a set of right hand side functions.
 \editend
% \textcolor{red}{I have a personal problem in this section: in the methodology, the pseudo inverse is introduced and it is referred to throughout the experiments. Why is there no mention of it, at least connecting what there is in the theory, to that operator? In particular, the theory assumes the equation learning error holds, but shouldn't normally the $\widehat u$ in here be estimated from \eqref{eq:operator-learning}? Even the operator learning error measures of experiments has a $\P^\dagger$ in there. In theory, we treat operator learning as if it is done \emph{after} equation learning is done, but the methodology doesn't explicitly say that and just introduces a variational formulation for operator learning. If the operator learning is only done \emph{after} equation learning, what is the point of it, as in, what insight does the operator learning provide us with?} \textcolor{teal}{add a lemma for this; see proposition 3.6. of GP-PDE.}

\editstart
\begin{theorem}\label{prop:operator-learning-1-step-KEqL}
    Suppose \Cref{thm:1-step-error-estimate} is satisfied  and take $F \subset C(\mY)$ and $A \subset \UU$ such that 
    $A \supseteq \wh \P_{M,N}^{-1}(F) \cup \P^{-1}(F)$\footnote{In the sense of preimages.}
    and $\Phi(A) :=\{ \Phi(u,y) \mid u \in A, \quad y \in \mY\} \subset B$.
     Further, suppose  that $\P^{-1}$ is uniformly Lipschitz 
     on $\P(A)$, i.e., there exists a constant $L(A) \ge 0$
    so that for any pair $f, f' \in \P(A)$,  
    \begin{equation*}
         {\| \P^{-1}(f) - \P^{-1}(f') \|_\UU} \le  L(A) {\| f - f' \|_\FF}.
    \end{equation*}
    Then, with $C >0$ denoting the constant in \Cref{thm:1-step-error-estimate}, we have that 
    \begin{equation*}
\begin{aligned}
        & \sup_{f \in F} \sup_{\whu \in \wh \P_{M,N}^{-1}(f)}
        \| \whu - \P^{-1}(f) \|_\UU    \\ 
         & \qquad \qquad \le C \cdot L(A) \cdot  
         \left( \varrho_{M, N}^{\eta - \frac{Q + d}{2}}
    + \rho_{M, N}^{\gamma - \gamma'} \right) 
    \left( \| P \|_\PP^2 + \| \overline{P} \|_\PP^2 + 
     \frac 1 M\sum_{m=1}^M \| u^m \|_{\UU^2}^2 \right)^{1/2}.
\end{aligned}
    \end{equation*}
\end{theorem}
%$\Phi(A) :=\{ \Phi(u,y) \mid u \in A, \quad y \in \mY\} \subset B$

The proof of this theorem is given in \Cref{sm:operator-learning-analysis}.
Simply put, the bound tells us that the operator learning problem 
    associated to ``solving" $\widehat\P(u) = f$ essentially inherits the 
    same rate of convergence as the equation learning problem so long as 
    the new right-hand-side function $f$ is sufficiently
    close to the training data. However, 
    note that the assumption that $\P^{-1}(F) \cup \wh \P_{M,N}^{-1}(F) \subset A$ such 
    that $\P^{-1}$ is Lipschitz on the image of $A$ under $\P$ is quite strong 
    since it is possible that $\wh \P_{M,N}^{-1}(F)$ is very large or even empty.
    However, we anticipate that if $\wh \P$ is accurate enough and $F$ 
    is  concentrated
    around the training data, then the assumption may be verified. 
    Additionally, $\whP^\dagger(f)$ gives a regularized 
    solution to $\whP(u) = f$ and the gap due to 
    the regularization is not accounted for.
    Regardless, our result is a first step towards
    the analysis of equation-mediated operator learning and 
    a more in-depth  analysis  remains open.

% \begin{remark}
%      In general, our methodology, convergence results, and experiments hold for non-pointwise settings. However, the error estimates both for equation learning and operator learning are pointwise. Although we did not prove non-pointwise error estimates, techniques for proving such estimates can be found in Theorem 4.11. and Chapter 14 of \cite[Thm.~4.11 and Ch.~14]{owhadi2019operator}. 
% \end{remark}
% {\color{red} YJ TODO: REWRITE THIE REMARK TO BAMDAD: I think the description at the beginning of the section is pretty good, so I'll comment the remark.}
\editend

% \subsection{Connection to existing results}
% The theoretical analysis of equation learning
% has been largely focused on the case of sparse recovery methods
% \cite{zhang2019convergence} 
% while more recent articles have also considered fundamental 
% limitations of equation learning such as  
% identifiability \cite{scholl2023uniqueness} and 
% PDE learning from a single trajectory  \cite{he2024much}. \textcolor{red}{We are also writing this around line 630. Maybe we could remove one if it is not necessary to have both here.}
% More precise results have been obtained for ODEs in \cite{russo2024convergence}.
%  Our analysis 
% is different from these works in a number of directions: 
% (1) to our knowledge  previous works do not cover 
% the case of the 1-step methods or the RKHS formulation; 
% (2) quantitative worst-case error bounds 
% such as ours are first of their kind; 
% (3) some of the stronger guarantees in the literature
% are tailored to specific PDEs while our results apply to generic
% nonlinear PDEs under sufficient smoothness assumptions.
% In connection to operator learning we note that most of the works 
% in this direction are focused on the function space regression 
% framework \cite{kovachki2024operator} and, to our knowledge, 
% our results above are a first of their kind in this direction.

\section{Discussion}\label{sec:conclusion}
We presented the 1- and 2-step formulations of KEqL as an 
algorithm for 
\editstart
PDE learning, state reconstruction, and equation-mediated operator learning
in scarce data regimes. 
\editend
While our algorithms, experiments, and theoretical analyses
were focused on an RKHS formulation, our abstract framework also unifies other types of models including neural nets or sparse regression methods 
such as SINDy. Additionally, our exposition highlights the common optimal recovery thread running 
through 
equation learning, PDE solvers, and  operator learning.

Our theoretical results provide quantitative error bounds for our 
algorithms while our numerical experiments demonstrate significant gains in 
accuracy and robustness in comparison to existing methods in the literature
\editstart
in scarce data regimes. Our formulation also leads to efficient Gauss-Newton
 methods for  training our models.
\editend
% Most notably, our methods appear to be a lot less sensitive to choice of hyperparameters and demonstrate more consistent performance compared to 
% equivalent neural net methods. 

Various avenues of future research and extensions of the KEqL framework can be 
identified: (a) we outlined KEqL for a single PDE but 
its extension to systems of PDEs is an obvious next step since 
many physical processes of interest are governed by systems of equations;
(b) our approach to \editstart
equation-mediated operator learning,
\editend relies on a 
PDE solver for each evaluation of $\widehat{\P}^\dagger$ which can be expensive. 
Therefore it may be interesting to investigate the emulation of this process 
to obtain a cheap solver that can be deployed for real time predictions;
(c) In many of our experimental results, we observed very competitive 
accuracy for 1-step KEqL, however, we still observed a relative
error barrier (around $10^{-2}$ in \Cref{fig:darcy_main}) which 
is not in line with our theoretical guarantees; we suspect these 
issues may arise due to ill-conditioning of the problems and 
various approximations made in the algorithms in order to scale the computations; 
(d) While our theoretical analysis addresses asymptotic convergence rates, 
it does not apply to scarce data regimes where 1-step KEqL is most useful and therefore a 
non-asymptotic analysis in this regime  would be of great interest; 
\editstart
and finally (e) dictionary based methods such as SINDy are able to obtain 
human-interpretable equations that can aid in scientific discovery while our
RKHS formulation does not readily present an interpretable 
form for the equation; the extraction of such an interpretable form from 
the KEqL solution, perhaps as a post-processing step, remains an open research direction.
\editend

\editstart
\section*{Acknowledgements}
The authors are grateful to anonymous reviewers whose comments and questions 
helped us improve an earlier version of this article. 
\editend

%%%%%%%%%%%%%%%%%%%%%%%%%%%%%%%%%%%%%%%%%%%%%%%%%%%%%%%
%%%%%%%%%%%%% Appendices (formerly the Supplementary Materials) %%%%%%
%%%%%%%%%%%%%%%%%%%%%%%%%%%%%%%%%%%%%%%%%%%%%%%%%%%%%%%

\appendix

The appendices contain additional material that supplements the main text. 
\editstart
\Cref{sec:SINDy-PINN-review} gives a brief review of SINDy and PINN-SR in our notation.
\Cref{app:algorithms} contains further implementation details 
for our algorithms while \Cref{app:numerics} contains additional results 
from our numerical experiments in the main body. 
\Cref{app:theory-details} gives additional theoretical 
results.
\editend

\editstart
\section{A brief review of SINDy and PINN-SR}\label{sec:SINDy-PINN-review}
Here we briefly summarize the SINDy/PDE-FIND and PINN-SR algorithms of \cite{brunton2016discovering,rudy2016datadriven}
and \cite{chen2021physics} in the notation of \Cref{sec:proposed-method}. We will keep the review 
at an abstract level and avoid specific algorithmic details that can be found in the above references.

\subsection{SINDy and PDE-FIND}
SINDy is routinely presented and implemented as a 2-step method in our terminology although the recent article \cite{hsu2025joint} presents its extension to a 1-step method. Most developments of SINDy consider 
observations $u^m(Y^m)$ given on uniform grids and proceed to interpolate the data with piecewise polynomial features in order to obtain finite difference approximations to the $L_q u$. We view this as an special 
instance of the optimal recovery problems \eqref{u-optimal-recovery} and \eqref{u-optimal-recovery-noisy}
with appropriate kernels. For non-structured observations we readily use \eqref{u-optimal-recovery}
and \eqref{u-optimal-recovery-noisy}. Thus, 2-step KEqL and SINDy share their first steps modulo 
the choice of the model class for the $\whu^m$. 

The difference between the two methods arises in the second step and in particular the model class for 
the learned equation. While KEqL uses a kernel $\PK$ to estimate the nonlinear function $P$, SINDy 
estimates this function over a dictionary. To be more precise, consider feature maps 
$\{ \psi_j \}_{j=1}^J$ where $\psi_j: \mS \mapsto \R$ and write $P = \sum_{j=1}^J \beta_j \psi_j$. Then 
we solve the sparse regression problem 
\begin{equation*}
    \widehat{\beta} := \argmin_{\beta \in \R^J} \| \beta \|_1 
    + \frac{1}{2\sigma_P^2} 
    \sum_{m=1}^M \left\| \left( \sum_{j=1}^J \beta_j \psi_j\right) \circ \Phi( \whu^m, Y) - f^m(Y) \right\|_2^2,
\end{equation*} 
in place of \eqref{P-optimal-recovery-noisy}. The 1-norm penalty on $\beta$ encourages sparsity which 
results in $\whP = \sum_{j=1}^J \widehat{\beta}_j \psi_j$ that only involves a few of the features $\psi_j$. 
If we choose the specific kernel $\PK(s, s') = \sum_{j=1}^J \psi_j(s) \psi_j(s')$ then the 2-step KEqL and 
SINDy methods only differ in the fact that SINDy uses the 1-norm penalty while KEqL will use a 
2-norm penalty on the coefficients $\beta$.

\subsection{PINN-SR}
The PINN-SR method is a direct competitor for  1-step KEqL  that utilizes a neural net model 
for state reconstruction along with a SINDy dictionary model for  $P$. Simply put, 
we solve the optimization problem 
\begin{equation*}
\begin{aligned}
     (\wh{\theta}^m, \wh{\beta}) = & \argmin_{ \theta^m \in \R^\Theta,  \beta \in  \R^J} 
 \| \beta \|_1 
 + \lambda 
 \sum_{m=1}^M  \frac{1}{2 \sigma_u^2} \| \text{NN}(Y^m; \theta^m) - u^m(Y^m) \|_2^2  \\
 & + \frac{1}{2 \sigma_{\partial \mY}^2} 
 \left\| \text{NN}(Y^{\partial \mY}; \theta^m) 
 - g^m(Y^{\partial \mY}) \right\|_2^2 \\
 & + \frac{1}{2 \sigma_P^2} \left\| \left(\sum_{j=1}^J \beta_j \psi_j\right) \circ \Phi(\text{NN}(\cdot; \theta^m), Y) - f^m(Y) \right\|_2^2,
\end{aligned}
\end{equation*}
where we wrote $\text{NN}(y; \theta)$ to denote a scalar valued neural net that takes $y$ as input 
and has parameters (weights and biases) denoted with $\theta$.
In practice one often uses the same architecture to approximate
each $u^m$ but with different weights $\theta^m$ as implied 
by our notation. One  could also use different architectures
per state if needed.
Thus the PINN-SR method deviates from 
1-step KEqL in two major ways: (1) the equation is learned over a prescribed dictionary with a sparsity 
prior and (2) the states $u^m$ are approximated via neural nets as in PINNs as opposed to our RKHS model classes.
\editend

\section{Details of algorithms}\label{app:algorithms}

\editstart
In this section, we describe the details of numerical implementation of 1-step KEqL, in particular 
algorithms for solving \eqref{one-step-representer-relaxed}. Afterwards we briefly 
discuss our procedures and heuristics for the design and choice of kernels 
and hyperparameters for both 1-step and 2-step KEqL.
\editend

\subsection{Implementation of 1-step KEqL}
\editstart
We begin by outlining an LM algorithm for solving \eqref{one-step-representer-relaxed}
that is practical for small or moderate sized problems. Due to large kernel matrices 
that algorithm has  limited scalability. 
To address these issues we outline various optional strategies 
that are implemented in our numerical experiments, including: a Nystr\"om approximation to our model for $P$ using a reduced basis; applying a  change of variables in the optimization
algorithms using Cholesky factors of kernel matrices to improve conditioning; and 
fast matrix calculations by leveraging block arrow structures. 
\editend

\subsubsection{An LM algorithm for 1-step KEqL}\label{LM}
 Solving \eqref{one-step-representer-relaxed}
 is difficult in practice as a small perturbation 
in the $\alpha^m$ can lead to a large perturbation in the
derivatives of the corresponding functions 
$v^m =  \UK(\phi, \cdot)^T \alpha^m$. 
This large variation further
translates into a sharp change in the $S^m(\alpha^m)$ 
points leading to instabilities. 
To mitigate this issue, 
we propose a LM algorithm that will approximate the objective function of \eqref{one-step-representer-relaxed} with a sequence of quadratic minimization problems; 
see \cite[Ch.~10]{nocedal2006numerical} for an overview of similar methods.

Define $\JJ^m(\pmb{ \alpha}, \beta):= \PK( S(\pmb \alpha), S^m(\alpha^m))^T \beta + \overline{P}(S^m(\alpha^m)) - f^m(Y)$ so that the last term in \eqref{one-step-representer-relaxed} is simply $\frac{1}{2\sigma_P^2} \| \JJ^m(\pmb \alpha, \beta) \|_2^2$. Then we compute a minimizing sequence 
$(\pmb \alpha_{(j)}, \beta_{(j)})_{j=1}^\infty$ 
given by the scheme 
\begin{equation}\label{LM-update-rule}
\begin{aligned}
    (\pmb \alpha_{(j+1)},&  \beta_{(j+1)}) 
    =  \argmin_{{ \pmb \alpha} \in (\R^{QK})^M, \: \beta \in \R^{MK}}    
      \beta^T \PK\left(S(\pmb \alpha_{(j)}), S(\pmb \alpha_{(j)})\right) 
      \beta  \\
  &   +   \sum_{m=1}^M  \Bigg[ 
  \editstart
  \frac{\lambda}{M}
  \editend
  (\alpha^m)^T \UK(\phi,\phi) \alpha^m    + 
    \frac{1}{2 \sigma_u^2} \| \UK(\phi, Y^m)^T \alpha^m - u^m(Y^m) \|_2^2 \\
&  \qquad  + \frac{1}{2 \sigma_P^2} \left\|\JJ^m(\pmb \alpha_{(j)}, \beta_{(j)}) + 
\nabla \JJ^m(\pmb \alpha_{(j)}, \beta_{(j)})
\begin{pmatrix}
\pmb \alpha - \pmb \alpha_{(j)} \\
\beta - \beta_{(j)}
\end{pmatrix}
     \right\|_2^2 \Bigg] \\ 
& +\lambda_{(j)} \Bigg[ (\beta - \beta_{(j)})^T \PK\left(S(\pmb \alpha_{(j)}), S(\pmb \alpha_{(j)})\right) 
      (\beta - \beta_{(j)} ) \\ 
  &\qquad \qquad     + \sum_{m=1}^M (\alpha^m - \alpha^m_{(j)})^T \UK(\phi, \phi) (\alpha^m - \alpha^m_{(j)})  \Bigg].
\end{aligned}
\end{equation}
The first two terms are identical to \eqref{one-step-representer-relaxed}
with $\pmb \alpha$ fixed at the previous value $\pmb \alpha_{(j)}$,
 the third term is unchanged, and 
the forth term is a local quadratic approximation obtained 
by linearizing $\JJ^m$, with $\nabla \JJ^m$ denoting  the Jacobian. 
The last term acts as  damping that ensures that our next estimate $(\pmb \alpha_{(j+1)}, \beta_{(j+1)})$ does 
not deviate too far from the current values. The $\lambda_{(j)} >0$ is
inversely proportional to a step-size and is updated using a standard 
gain ratio heuristic that compares the decrease produced in the true objective to the decrease observed in the quadratic approximation.

We highlight that naive implementations of  \eqref{LM-update-rule} lead to 
 effective algorithms, but the computational costs due to  matrices $\PK(S,S)$ 
and $\UK(\phi,\phi)$ become prohibitive when $N, M$ are large. 
To address these bottlenecks we implement different computation 
techniques such as inducing point approximations and block matrix inversion.

\subsubsection{Reduced 1-step KEqL: an efficient approximation}\label{sec:1.5-step-learning} 
We briefly introduce an approximate 1-step KEqL that, 
at a small cost to accuracy,
leads to a better conditioned and  more efficient formulation. Let us approximate 
\eqref{one-step-representer-relaxed} 
 by restricting the $v^m$'s to the subspace 
 $\text{span} \{ \UK(y_1, \cdot), \dots, \UK(y_K, \cdot)  \} \subset \UU$, i.e., 
 we write $v^m = \UK(Y, \cdot)^T \alpha^m$ for $\alpha^m \in \R^K$ as opposed 
 to $\R^{QK}$ for 1-step KEqL.  Thus this reformulation 
 effectively constitutes a reduced basis/feature map formulation leading to 
 the following analogous problem:
 \begin{equation}\label{1.5-step-optimal-recover}
\begin{aligned}
     (\wh{\pmb \alpha}, \wh{\beta}) = 
     & \argmin_{ \pmb \alpha \in (\R^{K})^M, \: \beta \in \R^{MK}} 
 \beta^T \PK(S(\pmb \alpha), S(\pmb \alpha) ) \beta 
 +  \sum_{m=1}^M 
 \editstart
 \frac{\lambda}{M} 
 \editend
 (\alpha^m)^T \UK(Y, Y) \alpha^m  \\
 & \qquad + \frac{1}{2\sigma_u^2} \| \UK(Y, Y^m)^T \alpha^m  - u^m(Y^m) \|_2^2 \\
 & \qquad + \frac{1}{2\sigma_P^2} \| \PK(S(\pmb \alpha), S^m(\alpha^m) )^T \beta + \overline{P}(S^m(\alpha^m) -  f^m(Y)\|_2^2,
\end{aligned}
\end{equation}
 \\
with $S(\pmb \alpha) = \cup_{m=1}^M S^m(\alpha^m)$ as before 
but with $S^m(\alpha^m) = \Phi(U(Y, \cdot)^T \alpha^m, Y)$.

Observe that if we take $Y^m =Y$
for all $m$ then the first constraint above completely identifies the 
$\alpha^m$ and the reduced 1-step KEqL 
coincides with 2-step KEqL. Therefore this method interpolates 
between the 1-step and 2-step methods.
Furthermore, we can apply the same idea to the $\PK(S, s)$ feature 
maps and choose a reduced basis for representing the learned 
equation $\whP$ by, for example, subsampling the points in $S$ 
in \eqref{1.5-step-optimal-recover}. Regardless, the LM algorithm of  
 \Cref{LM} remains applicable.

\subsection{Nystr\"om approximations and reduced bases for $P$}\label{app:algorithms:nystromapproximation}
It is often possible to replace the canonical representer basis 
$\PK(S(\pmb \alpha), s)$ for $P$ with a fixed set of $I \ll MK$ inducing points $S_I$ \cite{leibfried2020tutorial}.
These inducing points can be chosen by randomly sub-sampling the elements of $S(\pmb \alpha)$ at every iteration of the LM algorithm, but this 
can lead to technical difficulties since the point cloud $S(\pmb \alpha)$
changes from one iteration to the next, meaning that our reduced 
basis for $P$ also needs to change. Instead, we simply choose $S_I$
by subsampling the point clouds 
$\Phi(\whu^m, Y)$ where the $\whu^m$ denote our estimates of the $u^m$ from 
2-step KEqL. Hence, our reduced basis for $P$ is pre-computed before 
the LM algorithm is implemented. 
Note, that this approximation becomes exact whenever the number of inducing points $I \geq \dim(\PP)$, for example, in the case of polynomial kernels which have finite rank, choosing a sufficiently 
large number of inducing points ensures exact solution of the problem. 

% In more general circumstances (such as the Darcy flow example), taking $N_{\mathrm{inducing}}=200$ is sufficient for the additional inducing point error to be negligible. This is very similar to the approximation made when applying the reduced version of 1-step method, as that subsamples the usual representer basis which would include differential operators applied to the kernel, to only the kernel functions themselves.

By fixing $S_I$, our finite dimensional optimization problem becomes 
\begin{equation}\label{inducing_opt}
  \begin{aligned}
     (\wh{\pmb \alpha}, \wh{\beta}) = 
     & \argmin_{{ \pmb \alpha} \in (\R^{QK})^M, \: \beta \in \R^{I}}    
      \beta^T \PK(S_I, S_I) \beta 
+  
\editstart
\frac{\lambda}{M}
\editend
\sum_{m=1}^M   (\alpha^m)^T \UK(\phi,\phi) \alpha^m \\ 
& \qquad\qquad  +  \sum_{m=1}^M 
    \frac{1}{2\sigma_u^2} \| \UK(\phi, Y^m)^T \alpha^m - u^m(Y^m) \|_2^2 \\
& \qquad\qquad + \frac{1}{2\sigma_P^2} \| \PK(S_I, S^m(\alpha^m))^T \beta - f^m(Y) \|_2^2, 
  \end{aligned}
\end{equation}
Notice that the matrix $P(S_I, S_I)$ is now fixed and of much smaller 
size than the full matrix $P(S(\pmb \alpha), S(\pmb \alpha))$. 

\subsection{Cholesky change of variables}\label{app:algorithms:choleskychangeofvariables}
\editstart
Our next optional strategy is to precondition our LM updates to further 
improve the convergence of the optimizer.
\editend
Taking the Cholesky factorizations $\PK(S_I, S_I)=C_PC_P^T$ and $\UK(\phi,\phi) = C_UC_U^T$, we can define the transformed variables 
\begin{equation*}
    w^m = C_U^T \alpha^m,\quad z=C_P^T\beta.
\end{equation*}
Under this transformation, we can re-write \eqref{inducing_opt} as,
\begin{equation}\label{transformed_opt}
  \begin{aligned}
     (\wh{\pmb w}, \wh{z}) = 
     & \argmin_{{ \pmb{w}} \in (\R^{QK})^M, \: z \in \R^I}   
      \|z\|^2 
+ 
\editstart
\frac{ \lambda}{M}
\editend
\sum_{m=1}^M   \|w^m\|^2 \\ 
& \qquad  +  \sum_{m=1}^M 
    \frac{1}{2\sigma_u^2} \| \UK(\phi, Y^m)^T C_U^{-T} w^m - u^m(Y^m) \|_2^2 \\
& \qquad + \frac{1}{2\sigma_P^2} \| \PK(S_I, S^m(w^m))^T C_P^{-T}z - f^m(Y) \|_2^2. 
  \end{aligned}
\end{equation}
While the LM algorithm is generally invariant under such changes of variables this greatly improves the conditioning of the optimization problems in 
practice. Indeed, we found that this change of variables enables first order methods, such as gradient descent, to also produce reasonable solutions
when otherwise they would have failed. 

To further clarify our implementation of the LM algorithms we abstract the objective into a general nonlinear least squares problem
by rewriting \eqref{app:algorithms:choleskychangeofvariables} as
\begin{equation}\label{nls_opt}
  \begin{aligned}
     (\wh{\pmb w}, \wh{z}) = 
     & \argmin_{{ \pmb{w}} \in (\R^{QK})^M, \: z \in \R^{I}}   
      \|z\|^2 
+  \editstart
\frac 1 M
\editend
\sum_{m=1}^M   \|w^m\|^2 + \|F(z,\pmb{w})\|^2 \\
& \text{where }\quad F(z,\pmb{w})=\begin{bmatrix}F_{1}\left(z,w^{1}\right)\\
F_{2}\left(z,w^{2}\right)\\
\vdots\\
F_{M}\left(z,w^{M}\right)
\end{bmatrix}\in \R^{M(N+K)}, \text{ and} \\
&F_{m}(z,w_{m})=\begin{bmatrix}
\sqrt{\frac{1}{2\sigma_u^2}}\left(\UK(\phi, Y^m)^T C_U^{-T} w^m - u^m(Y^m)\right)\\
\sqrt{\frac{1}{2\sigma_P^2}}\left(\PK(S_I, S^m(w^m))^T C_P^{-T}z - f^m(Y)\right)
\end{bmatrix}.
  \end{aligned}
\end{equation}
For this problem we obtain the LM updates
\begin{equation}\label{newLM-update-rule}
\begin{aligned}
    (\pmb{w}_{(j+1)}, z_{(j+1)}) 
    = & \argmin_{{ \pmb{w}} \in (\R^{QK})^M, \: z \in \R^{I}}   
      \|z\|^2 
+  
\editstart
\frac 1M 
\editend
\sum_{m=1}^M   \|w^m\|^2 + \|F_{(j)}(z,\pmb{w})\|^2
\\
&+ \lambda_{(j)}\left(\|z - z_{(j)}\|^2 
+  
\editstart
\frac 1M
\editend
\sum_{m=1}^M   \|w^m - w^m_{(j)}\|^2\right)
\\
 &\hspace{-10ex}\text{where}\quad F_{(j)}(z,\pmb{w}) = \nabla F(z_{(j)},\pmb{w}_{(j)})\begin{bmatrix}z-z_{(j)}\\
\pmb{w}-\pmb{w}_{(j)}
\end{bmatrix}
 + F(z_{(j)},\pmb{w}_{(j)}),
\end{aligned}
\end{equation}
and $\lambda_{(j)}$ is a proximal regularization parameter to ensure global convergence.
In many cases, applying this iteration with a reasonable update rule 
for $\lambda_j$ and solving the least squares subproblems by computing Cholesky factorizations of the normal equations is sufficient to achieve accurate solutions. When issues with the conditioning of the normal 
equations arise, we use an SVD solver which is more accurate but 
has a higher computational cost.

\subsubsection{A heuristic for choosing $\lambda_{(j)}$}\label{app:algorithms:dampingparameterheuristic}
We now discuss the adaptation of the damping parameter $\lambda_{(j)}$, which is crucial for both ensuring stability of the optimization algorithm far from a minimum, and providing fast convergence near minima. 
Following \cite{aravkin2024levenberg,gavin2019levenberg},
define the \textit{gain ratio} $\rho_{(j)}$ to be the ratio of the decrease of the true objective value to the decrease predicted by the linearized objective,
\begin{equation*}
    \begin{aligned}
    \rho_{(j)} & := \frac{\theta_{(j)}}{\theta'_{(j)}}, \\
    \theta_{(j)}& := \left(\|z_{(j+1)}\|^2 
+  
\editstart
\frac 1M
\editend
\sum_{m=1}^M   \|w^m_{(j+1)}\|^2 + \|F(z_{(j+1)},\pmb{w}_{(j+1)})\|^2\right)  \\ 
& \quad  -  \left(\|z_{(j)}\|^2  + 
\editstart
\frac 1M
\editend
\sum_{m=1}^M   \|w^m_{(j)}\|^2 + \|F(z_{(j)},\pmb{w}_{(j)})\|^2\right), \\ 
    \theta'_{(j)}& := \left(\|z_{(j+1)}\|^2 
+ 
\editstart
\frac 1M
\editend
\sum_{m=1}^M   \|w^m_{(j+1)}\|^2 + \|F_{(j)}(z_{(j+1)},\pmb{w}_{(j+1)})\|^2\right)  \\ 
& \quad  -  \left(\|z_{(j)}\|^2  + 
\editstart
\frac 1M
\editend
\sum_{m=1}^M   \|w^m_{(j)}\|^2 + \|F(z_{(j)},\pmb{w}_{(j)})\|^2\right),
    \end{aligned}
\end{equation*}
% $$
% \theta_{(j)} = \|z_{(j+1)}\|^2 
% +  \sum_{m=1}^M   \|w^m_{(j+1)}\|^2
% $$
% $$
% \theta_{(j)}' = \|F(z_{(j+1)},\pmb{w}_{(j+1)})\|^2
% $$
% \begin{equation*}
%     \rho_{(j)} := \frac
%     {\left(\|z_{(j+1)}\|^2 
% +  \sum_{m=1}^M   \|w^m_{(j+1)}\|^2 + \|F(z_{(j+1)},\pmb{w}_{(j+1)})\|^2\right) - 
% \left(\|z_{(j)}\|^2 
% +  \sum_{m=1}^M   \|w^m_{(j)}\|^2 + \|F(z_{(j)},\pmb{w}_{(j)})\|^2\right)}
%     {\left(\|z_{(j+1)}\|^2 
% +  \sum_{m=1}^M   \|w^m_{(j+1)}\|^2 + \|F_{(j)}(z_{(j+1)},\pmb{w}_{(j+1)})\|^2\right) - \left(\|z_{(j)}\|^2 
% +  \sum_{m=1}^M   \|w^m_{(j)}\|^2 + \|F(z_{(j)},\pmb{w}_{(j)})\|^2\right)}.
% \end{equation*}
For constants $0<c_0<c_1<c_2<1$ and an adaptation multiplier $b>1$, we set 
\begin{equation*}
    \lambda_{(j+1)}=\begin{cases}
b\lambda_{(j)} & \rho_{(j)}<c_{1}\\
\lambda_{(j)} & \rho_{(j)}\in(c_{1},c_{2})\\
\frac{1}{b}\lambda_{(j)} & \rho_{(j)}>c_{2}.
\end{cases}
\end{equation*}
In the case that $\rho_{(j)}<c_0$, we reject the step,  
set $(z_{(j+1)},\pmb{w}_{(j+1)}) = (z_{(j)},\pmb{w}_{(j)})$, and attempt to compute the next step with the increased damping parameter. 

\subsection{Leveraging block-arrow sparsity 
\editstart structure \editend}\label{app:algorithms:blockmatrixstructure}
Instantiating and computing solutions to
the regularized linear least squares problems within the LM iterations can be difficult in large scale problems where $N, M$ are large.
In such cases, the solution to \eqref{newLM-update-rule} can be computed directly in $O((QK+N_{\mathrm{I}})^3 M)$ operations, 
by leveraging sparsity of the gradients $\nabla F$,
rather than the naive $O((QK+N_{\mathrm{I}})^3 M^3)$. 
Observe that 
\begin{equation*}
    \nabla F(z_{(j)},\pmb{w}_{(j)})
%     \begin{bmatrix}z-z_{(j)}\\
% w^{1}-w^{1}_{(j)}\\
% \vdots\\
% w^{M}-w^{M}_{(j)}
% \end{bmatrix}
= 
    \begin{bmatrix}\nabla_{z}F_{1} & \nabla_{w_{1}}F_{1} & 0 & \cdots & 0\\
\nabla_{z}F_{2} & 0 & \nabla_{w_{2}}F_{1} & \cdots & 0\\
\vdots & \vdots & \vdots & \ddots & \vdots\\
\vdots & 0 &  &  & 0\\
\nabla_{z}F_{M} & 0 & 0 & \cdots & \nabla_{w_{M}}F_{M}
\end{bmatrix}
% \begin{bmatrix}z-z_{(j)}\\
% w^{1}-w^{1}_{(j)}\\
% \vdots\\
% w^{M}-w^{M}_{(j)}
% \end{bmatrix}
\end{equation*}
This form of $\nabla F$ induces a block arrowhead structure for the regularized normal equations for the least squares problems 
within the LM update,  
\begin{equation*}
    \nabla F^T \nabla F + (1+\lambda_{(j)})I = \begin{bmatrix}* & * & \cdots & * & *\\
* & *\\
\vdots &  & *\\
* &  &  & \ddots\\
* &  &  &  & *
\end{bmatrix} = \begin{bmatrix}A_P & B_{UP}^{T}\\
B_{UP} & D_{U}
\end{bmatrix} 
\end{equation*}
where $A_P\in \R^{N_I \times N_I}$ and $D_U \in \R^{QKM \times QKM}$ is a block diagonal matrix  with blocks of size $QK$. From here, block elimination can be applied to solve the normal equations
\begin{equation*}
\begin{bmatrix}A_{P} & B_{UP}^{T}\\
B_{UP} & D_{U}
\end{bmatrix}\begin{bmatrix}\delta z\\
\delta \pmb{w}
\end{bmatrix}=\begin{bmatrix}F_{P}\\
F_{U}
\end{bmatrix}
\end{equation*}
given by the formulae
\begin{align*}
    C &=\left(A_{P}-B_{UP}^{T}D_{U}^{-1}B_{UP}\right), \\
     \delta z &= C^{-1}\left(F_{P}-B_{UP}^{T}D_{U}^{-1}F_{U}\right), \\
    \delta \pmb{w} &= D_{U}^{-1}\left(F_{U}-B_{UP} \delta z \right),
\end{align*}
where we note that the Schur complement $C$ and the matrix $A_{P}$ are of the moderate size $N_I \times N_I$, and \editstart linear systems involving $D_U$ can be solved efficiently because of \editstart the block diagonal structure. 

\subsection{Choosing kernels and hyper-parameters}
\label{sm:choosing-kernels}
\editstart
In this section we briefly describe our choices of the kernels $\UK, \PK$ and our 
approach to tuning their hyper-parameters.

\subsubsection{Choosing $\UK$}
Since Sobolev spaces are a natural solution space for 
many PDEs we choose the Mat{\'e}rn kernel family 
 for $\UK$. In particular, we consider 
the anisotropic Mat{\'e}rn kernels $
\UK_{\text{Mat{\'e}rn}}(y, y') := \frac{2^{1-\nu}}{{\Gamma}(\nu)} \left( \sqrt{2\nu} \| y - y'\|_\Sigma \right)^{\nu} K_\nu \left( \sqrt{2\nu} \| y -y' \|_\Sigma \right), $
where $\nu >0$ is a parameter, $\Gamma$ is the standard gamma function, and $K_\nu$
is the modified Bessel function of the second kind, and 
$\| y - y'\|^2_\Sigma := (y -y')^T \Sigma^{-1} (y - y')$ for a PDS matrix $\Sigma \in \R^{d \times d}$. 
Under mild conditions (see \cite[Ex.~2.6,~2.8]{kanagawa2018gaussian}) 
 the RKHS of the Mat{\'e}rn kernel  
is norm equivalent to the Sobolev space $H^{\nu + d/2}$. Thus it is crucial that the 
regularity parameter $\nu$  is chosen carefully in light of the order of the PDE.
In the limit $\nu \to \infty$ the Mat{\'e}rn kernel converges to the 
radial basis function (RBF) kernel 
$\UK_{\rm RBF}(y, y') := \exp\left( -\frac{1}{2} \| y - y' \|_\Sigma^2 \right)$ which has an infinitely 
smooth RKHS. In addition to this kernel we also use the first order rational quadratic (RQ) kernel, defined as $\UK_{\rm RQ}(y,y') := \left(1 + \|y - y'\|_{\Sigma}^2\right)^{-1}$. 
 We often take $\Sigma$ to be a diagonal matrix 
$\Sigma = \text{diag}(\gamma)$ for a vector of lengthscales $ \gamma \in \R^d_{>0}$  
 to be chosen via maximum likelihood estimation (MLE), cross validation, or hand tuning. 
 % $L_q$, i.e., if 
% the $L_q$'s are $k$-th order differential operators then we need to take 
% $\nu > k$ 
% so that, by Sobolev embedding, the maps $\delta_y \circ L_q$ are continuous linear functionals 
% on $\UU$.

% In this case the RKHS  is embedded in $C^\infty$, eliminating concerns 
% over the existence of requisite partial derivatives in the PDE. This extra 
% flexibility, however, comes at the expense of possibly worse conditioned 
% kernel matrices with exponentially decaying singular valued
% as opposed to algebraic decay for Mat{\'e}rn kernels.

% Finally, we note that further structure about the solution to the PDE can be 
% incorporated in the design of the kernel $\UK$ should such knowledge be 
% available. A standard example is periodic boundary conditions in which 
% case we can use a compositional kernel to impose that structure;
% see \cite{duvenaud2014automatic}.

% In most PDE of elliptic or parabolic type we take $\UK$ to be a Mat{\'e}rn kernel of the form

\subsubsection{Choosing $\PK$}\label{sec:choosing-P}
% In many practical settings we may have prior knowledge about the unknown function $P$
% that can be encoded within our choice of the kernel $\PK$ to improve 
% accuracy and data efficiency of our algorithm. 
Since many PDEs in physical sciences 
have polynomial nonlinearities \cite[Sec.~1.2]{evans2022partial}, 
a natural choice for $\PK$ would be a 
polynomial kernel; this is also the dominant model for construction of 
dictionaries in SINDy-type algorithms. Here we consider 
$\PK_{\rm poly}(s,s') = \left((s-c)^TB(s'-c)+1\right)^{\text{deg}},$
for $s, s' \in \R^{Q + d}$,
where $B \in \R^{Q + d \times Q + d}$ is a matrix akin to $\Sigma$ in the previous section that allows us to 
scale different input coordinates or to introduce correlations, and $c \in \R^{Q + d}$ is a fixed vector introducing bias. We treat $B, c$ 
as hyper-parameters for this kernel.  
In extreme scenarios, 
if no a priori knowledge about $P$ exists, then one can 
take $\PK$ to be a Mat{\'e}rn or RBF kernel as 
in the case of  $\UK$ above. 
% \begin{equation*}
% \PK_{\rm RBF}(s, s') := \exp\left( -\frac{1}{2} \| s - s' \|_B^2 \right), 
% \qquad s, s' \in \R^{Q+d}.
% \end{equation*}
% The choice of a Mat{\'e}rn kernel is also valid here. 

We are particularly interested in learning PDEs with 
variable coefficients, such as the variable coefficient diffusion model \eqref{nonlinear-heat-eqn}. In such cases
we may have a priori knowledge that $P$ has polynomial dependence on a subset of inputs 
(the $L_1u(y), \dots, L_Qu(y)$ values)
but general nonlinear dependence on other parameters (the $y$ values). We 
encode such structures using hybrid product kernels of the form 
$\PK_{\text{hybrid}}(s,s') = \PK_{\rm RBF}(s_{:d},s_{:d}')\PK_{\rm poly}(s_{d:},s_{d:}')$ 
where we used $s_{:d} \in \R^d$ to denote the first $d$ coordinates of $s$
and $s_{d:} \in \R^{Q}$ to denote the remaining  coordinates.

\subsubsection{Procedures for tuning hyper-parameters}
We employ various heuristics to help us in tuning kernel parameters as 
well as various weights in our loss functions. These are described in detail for each of 
our experiments below but we review the overall approach.

We primarily employ a two step heuristic for tuning our kernel lenghscales and hyper-parameters 
for 1-step KEqL where, even in very scarce data regimes, we first employ 2-step KEqL and 
use MLE estimators to tune $\UK$ and $\PK$. The underlying kernel ridge regression 
problems for 2-step KEqL are standard and many off-the-shelf packages, e.g., \texttt{skleanr},
include functionality for MLE estimation of kernel parameters. We then fix these hyper-parameters 
in 1-step KEqL. 

We largely tune the rest of our hyper-parameters, i.e., $\lambda, \sigma_P, \sigma_u$ by 
hand however, we choose $\sigma_P, \sigma_u$ to be as small as possible while maintaining 
numerical stability since we wish to impose the PDE constraint and the data 
observations as accurately as possible (we mainly work in the interpolation regime). 
Following this heuristic $\lambda$ is the main parameter that we tune by hand without 
any particular heuristic. 

\editend

\section{Details of numerical experiments}\label{app:numerics}

We now present details of the numerical experiments summarized in the main body of the paper, covering various details such as data generation processes, selection of observation and collocation points, hyperparameter tuning,
and additional results and observations. Further implementation details can be found in our GitHub repository\footnote{\url{https://github.com/TADSGroup/kernelequationlearning}} where we collect code for regenerating 
our numerical results and figures. 

\subsection{The Duffing oscillator}\label{app:numerics:ODE-example}
In this example we compared the performance of 1-step and 2-step KEqL 
along with SINDy for learning a 1D nonlinear ODE. Our focus 
is on the performance of the methods in state reconstruction and 
extrapolation/forecasting of the dynamics.

Let us recall our problem setup where we took $\mY = (0,50)$ and 
considered the Duffing oscillator
\begin{equation}\label{duffing_ODE_SI}
\left\{
\begin{aligned}
        \P(u) &= \partial_{t}^2 u  - 3u +3u^3 +0.2 \partial_t u= \cos(2t), && t \in \mY, \\
        u(0) & = \partial_t u(0) = 0. &&
\end{aligned}
\right.
\end{equation}

To generate the training data for this example we solved the ODE numerically using a Dopri5 solver in Python with adaptive step size and initial value $1e-3$. The numerical solution was then subsampled on $N=32$ observation points that were also picked on a uniform collocation grid $Y$ of size $1000$
in $\mY$. The test data set for filtering the solution was generated similarly but on a finer uniform grid of $5000$ observation points. 
To generate the test datasets for forecasting and extrapolation errors, we used the same numerical solver, but, we introduced three new initial conditions, each satisfying $u(0) = 0$, while the initial time derivative $\partial_t u(0)$ was set to $0.5$, $1$, and $-1$, respectively.

\editstart
Kernels and hyperparameters were chosen as follows: 

\begin{itemize}
    \item For the kernels, we selected $\UK$ as a RQ kernel 
and $\PK$ as RBF, both with $\Sigma = \text{Id}$.  The length scale for the RQ kernel was set to $1$ and hand-tuned. The chosen identity matrix $\Sigma$ for $\PK$ means that the lengthscales in every feature direction, ${u,\partial_t u}$, were hand-tuned to be $1$.

\item We employed a fixed basis with 
$I = 500$ inducing points for learning the equation (see \Cref{app:algorithms:nystromapproximation}). 

\item In the loss function \Cref{app:algorithms:nystromapproximation}, we used
\begin{equation}
\label{eq: hyperparams}
\lambda = M, \qquad \sigma_u^2 = \theta_u N M , \qquad \text{and} \qquad \sigma_P^2 = \theta_P MI,
\end{equation}
with  $\theta_u = 1e-11$ and $\theta_P = 5e-10$, selected manually. Values are small because we tuned them by starting with a near-zero value and then gradually increasing them until the PDE residual increased significantly.

\end{itemize}

% For the loss function \Cref{app:algorithms:nystromapproximation}, we used
% \begin{equation}
% \label{eq: hyperparams}
% \lambda = 1, \qquad \sigma_u^2 = \sqrt{\frac{NM}{\theta_u}}, \qquad \text{and} \qquad \sigma_P^2 = \sqrt{\frac{MI}{\theta_P}},
% \end{equation}
% \textcolor{red}{This needs to change accordingly as well. If it is easy to get new runs, should we change these to the correct $\frac 1M$ values as well and see if things are consistent with before?}
% with  $\theta_u = 5e-8$ and $\theta_P = 1e-9$, selected manually.
The loss history for the optimization of 1-step KEqL can be seen in \Cref{fig: duffing_loss_hist} for learning the Duffing ODE.
We see that the algorithm has converged in about 
200 iterations. 

\begin{figure}[htp]
    \centering
    \begingroup
    \footnotesize
    \begin{overpic}[scale=0.4]{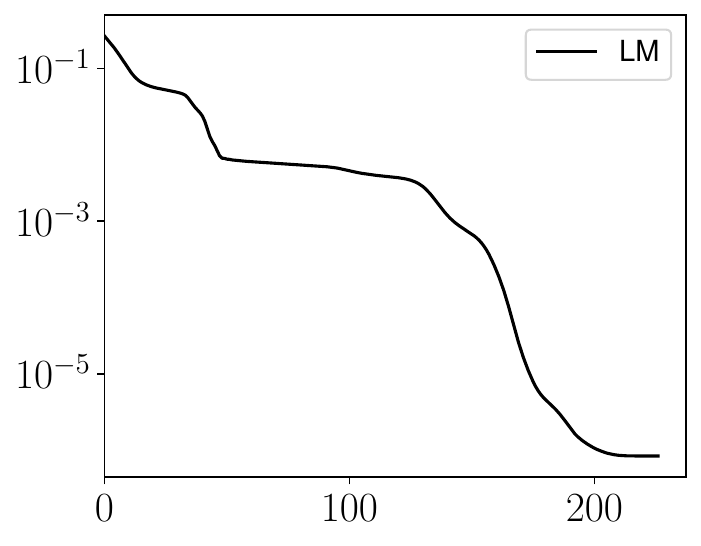}
    \put(-6,36){\rotatebox{90}{{ Loss}}}
    \put(50,-5){\rm Steps}
    \end{overpic}
    \endgroup
    \caption{Convergence history of LM for 1-step KEqL for the Duffing ODE \eqref{duffing_ODE}.}
    \label{fig: duffing_loss_hist}
\end{figure}

\subsection{The Burgers' PDE}\label{app:numerics:burgers}
In this example we compared 1-step KEqL with
SINDy and the PINN-SR algorithm.
Here we describe 
the choice of initial conditions, generation of 
training and test samples, hyperparameters, and optimization details used for this example. 

Let us recall our problem setting. We took $\mY = (0,1] \times (0,1)$ and  considered the Burgers' PDE 
\begin{equation}\label{burgers-PDE-SI}
\left\{
\begin{aligned}
        &\P(u) = \partial_t u  + \vartheta u \partial_x u - \nu \partial_{x}^2 u = 0
        \text{ for }(t,x) \in \mY, \\ 
        &u(0, x) = u_0(x), 
        \text{ and }
        u(t, 0)  = u(t, 1) = 0. 
\end{aligned}
\right.
\end{equation}

We remind the reader that whenever we state in any of the subsequent experiments that \eqref{burgers-PDE} was solved, we mean that a second-order Strang splitting method with a small step size was used, followed by spline interpolation to obtain a solution that can be evaluated at any point of its domain. Additionally, when using the 1-step KEqL method, we always assumed the use of the standard implementation of the algorithm described in \eqref{transformed_opt}.

\subsubsection{Experiments with increasing 
number of observations}
\label{sec:burgers-experiments-with-increasing-observations}

Here, we used the PDE \eqref{burgers-PDE} with coefficients $\vartheta = 1$ and $\nu = 0.01$ and prescribed initial conditions (IC) $u_0$ from a Gaussian process with Karhunen–Loève expansion 
\begin{equation}
    \label{eq: burgers_karhunen_loeve_exp}
    u_0(x) = \sum_{j=1}^{50} \frac{1}{j^2} \sin(j \pi x) Z_j, \quad \text{with } Z_j \sim N(0,1), \quad 
    x \in [0,1].
\end{equation}
The condition used for the fixed IC case is depicted in \Cref{fig: burgers_fixed_vary_IC_FC}(A) while some samples of ICs are shown in \Cref{fig: burgers_fixed_vary_IC_FC}(B-D) for the case when ICs vary also in the experiment. 

\begin{figure}[htp]
    \centering
    \vspace{2ex}
    \begingroup
    \footnotesize
    \setlength{\tabcolsep}{0pt}
    \begin{tabular}{cccc}
        \begin{overpic}[scale=0.25]{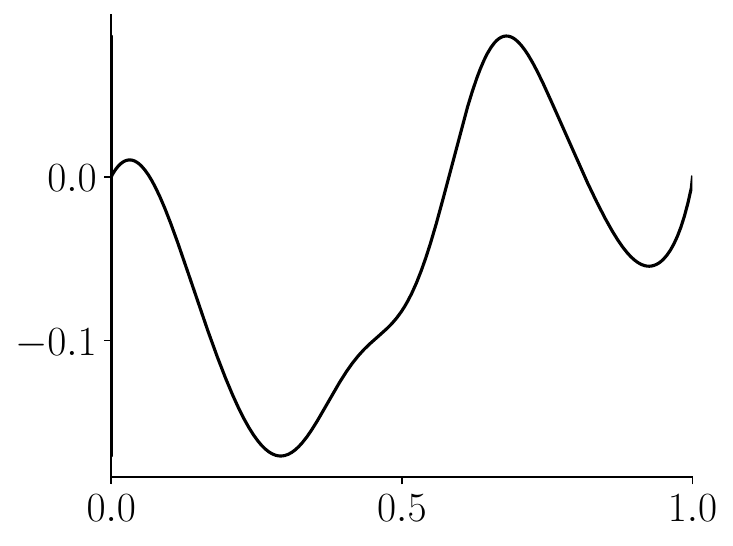} 
        \put(0,37){\rotatebox{90}{{ $u$}}}
        \put(52,-2){{$x$}}
        \put(24,82){{ (A) Fixed IC}}
        \end{overpic}
        &
        \begin{overpic}[scale=0.25]{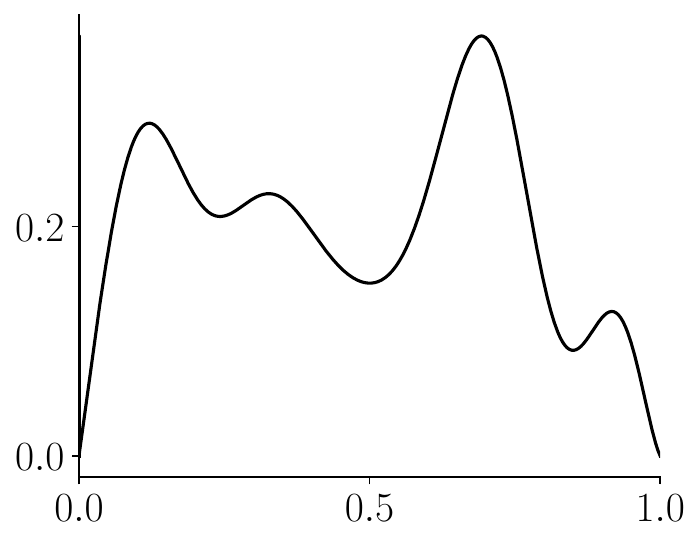}
        \put(52,-2){{$x$}}
        \put(10,86){{ (B) Varying IC 1}}
        \end{overpic}
        & 
        \begin{overpic}[scale=0.25]{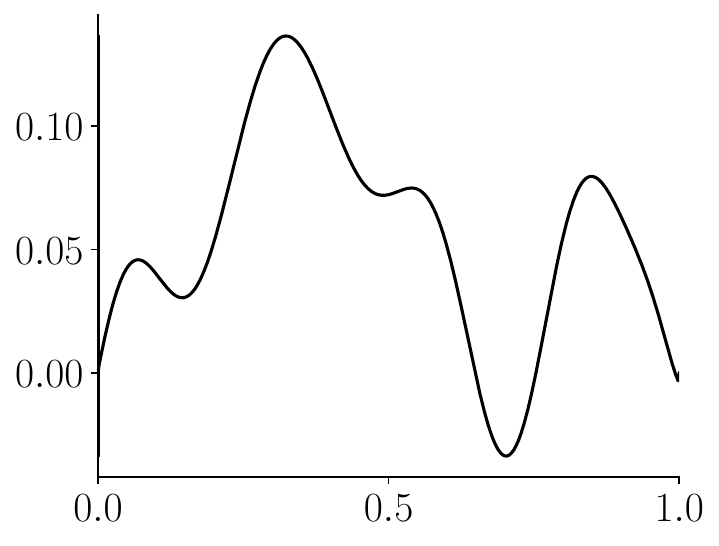}
        \put(49,-2){{$x$}}
        \put(15,84.5){{(C) Varying IC 2}}
        \end{overpic}
        & 
        \begin{overpic}[scale=0.25]{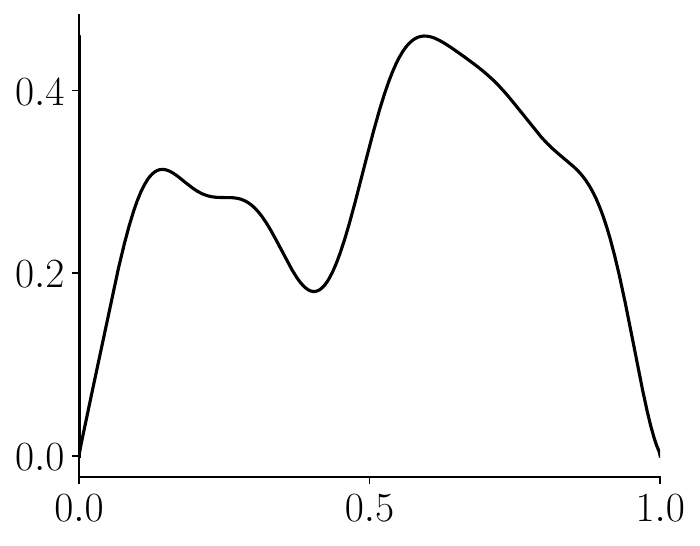}
        \put(49,-2){{$x$}}
        \put(15,87){{(D) Varying IC 3}}
        \end{overpic}
    \end{tabular}

    \endgroup
        \caption{Initial conditions sampled from the process \eqref{eq: burgers_karhunen_loeve_exp} for examples of Burgers' PDE \eqref{burgers-PDE} increasing number of observations: 
(A) Condition used in the fixed IC case; (B-D) Sample conditions for the varying IC case.
}
    \label{fig: burgers_fixed_vary_IC_FC}
\end{figure}

We then solved each PDE and subsampled the observed data for values $N_\mY = 10, 30, 50, 100, 200, 300, 400, 500, 600$ in the interior and  $N_{\partial \mY} = 131$ at the boundary from a set of $26 \times 31$ Chebyshev collocation points in space-time. In \Cref{fig: training_data_Burgers}, we present example training data overlaid on the contour plots of the solutions for the fixed initial condition, as shown in \Cref{fig: burgers_fixed_vary_IC_FC}(A), while varying the number of observations $N_\mY = 20, 50, 200$. Additionally, we include instances of the training data for the case where the initial condition varies, corresponding to the ICs introduced earlier in \Cref{fig: burgers_fixed_vary_IC_FC}(B-D).
Convergence history of our LM optimization algorithm for some sample cases of learning the equation is shown in \Cref{fig: burgers_loss_hist}(A-B).

\begin{figure}[htp]
    \centering
    \begingroup
    \footnotesize
    \setlength{\tabcolsep}{0pt}
    \begin{tabular}{lcl}
        \begin{overpic}[scale=0.33]{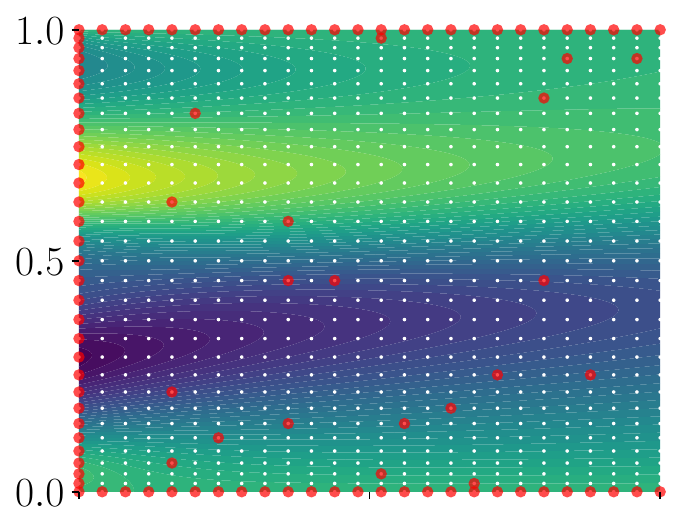} 
        \put(36,77){{ $N_\mY=20$}}
        \put(-2,38){\rotatebox{90}{{$x$}}}
        \put(-10,20){\rotatebox{90}{{(A) Fixed IC}}}
        \end{overpic}
        &
        \hspace{1ex}\raisebox{0.05 \height}{
        \begin{overpic}[scale=0.33]{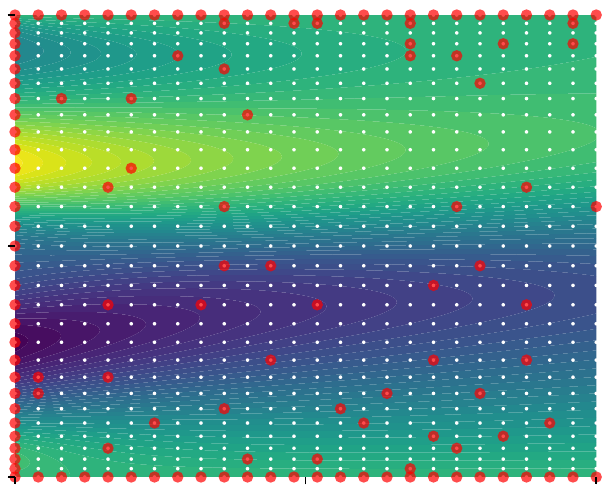}
        \put(32,81){{ $N_\mY=50$}}
        \end{overpic}
        }
        & 
        \hspace{1ex}
        \raisebox{0.05 \height}{
        \begin{overpic}[scale=0.33]{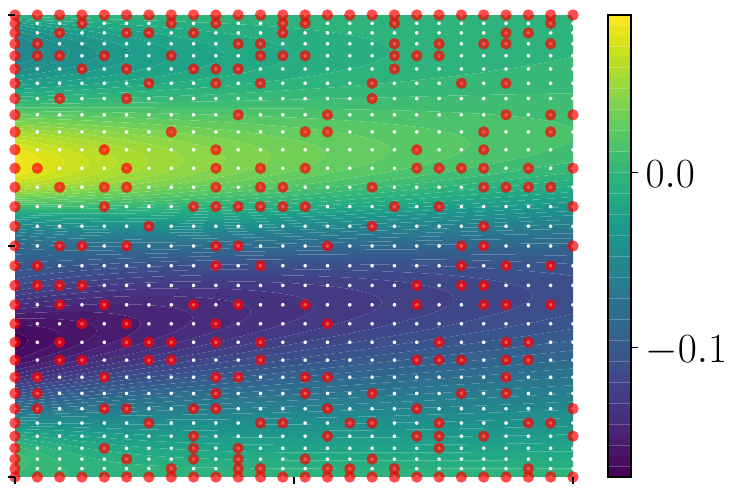}
        \put(23,66.5){{ $N_\mY=200$}}
        \end{overpic}
        }
        \\
        \begin{overpic}[scale=0.33]{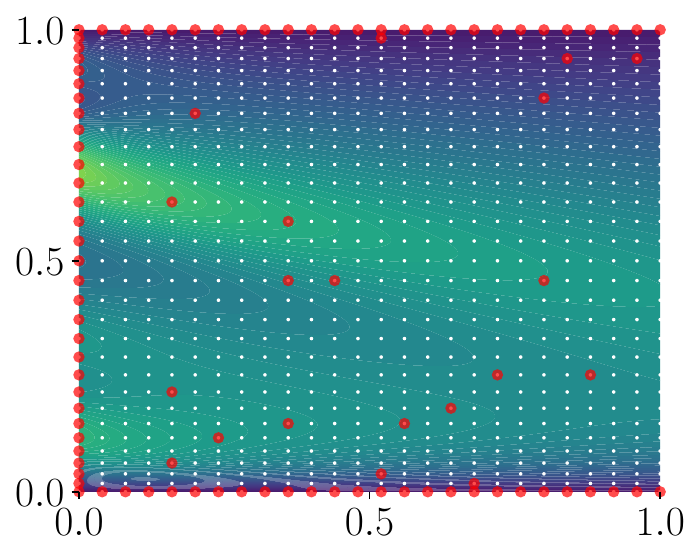} 
        \put(-2,39){\rotatebox{90}{{ $x$}}}
        \put(49,-2){{ $t$}}
        \put(-10,20){\rotatebox{90}{{(B) Varying IC}}}
        \end{overpic}
        &
        \hspace{0ex}
        \begin{overpic}[scale=0.33]{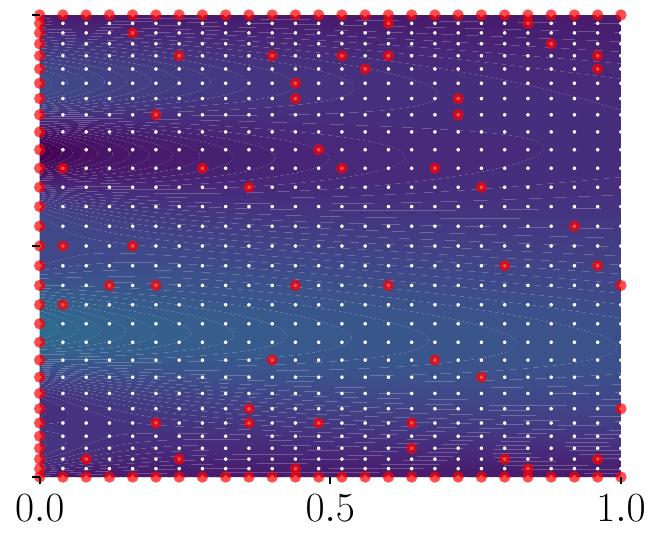}
        \put(46,-2){{ $t$}}
        \end{overpic}
        & 
        \hspace{0.5ex}
        \begin{overpic}[scale=0.33]{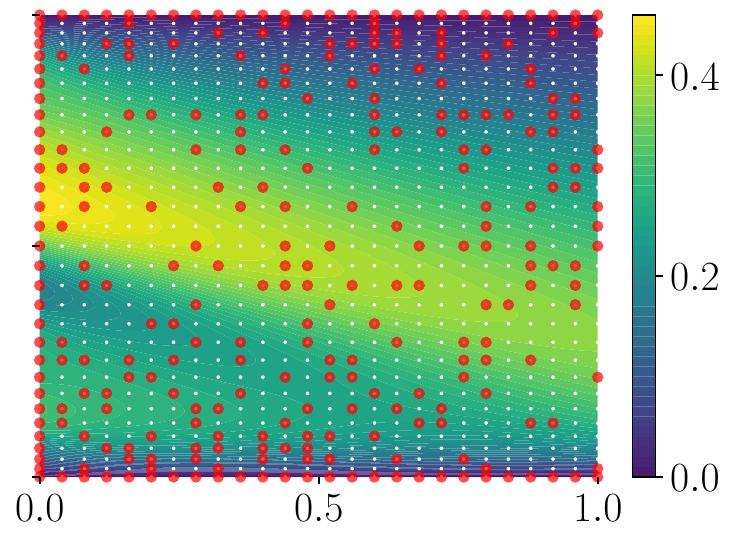}
        \put(42,-2){{$t$}}
        \end{overpic}
    \end{tabular}

    \endgroup
    \caption{Samples of training solutions along with their $N$ observation points (red) and the collocation points (white) for Burgers' PDE \eqref{burgers-PDE} with ICs sampled from the GP \eqref{eq: burgers_karhunen_loeve_exp}: (A) Solution obtained using a fixed IC depicted in \Cref{fig: burgers_fixed_vary_IC_FC}(A); (B) Three solutions obtained from solving the PDE with different ICs depicted in \Cref{fig: burgers_fixed_vary_IC_FC}(B-D).}
    \label{fig: training_data_Burgers}
\end{figure}

\editstart
Kernels and hyperparameters were chosen as follows:
\begin{itemize}
    \item For the kernels, we selected $\UK$ as a RBF kernel with a diagonal matrix $\Sigma$, where its entries were determined using MLE in 2-step KEqL. In both cases for this experiment, we chose $\PK$ to be a polynomial kernel of degree 2, with the shift parameter set to $c = \operatorname{mean}(S_{2:})$ and the scaling matrix given by $B = \operatorname{diag}(\operatorname{cov}(S_{2:}))^{-1}$, where $\operatorname{cov}(A)$ represents the sample covariance of $A$. Specifically,

\begin{equation} \label{eq: S_to_define_B}
S = \Phi(\whu_{2}, Y^1) := 
\begin{bmatrix}
    \Phi(\whu_{2}, y_1) \\
    \vdots
    \\
    \Phi(\whu_{2}, y_N)
\end{bmatrix}
\in \mathbb{R}^{N \times 5},
\end{equation}
where 
 $\whu_2$ is the learned solution obtained from the first step of the 2-step KEqL method, using the MLE-fitted hyperparameters mentioned above.
 \item We employed a fixed basis with 
$I = 26 \times 31$ inducing points for learning the equation (see \Cref{app:algorithms:nystromapproximation}).
 \item Following the definitions in \eqref{eq: hyperparams}, we set $\lambda = 1$, $\theta_u = 5e-14$ and $\theta_P = 5e-12$ as manually chosen hyperparameters, following the heuristic described in \Cref{app:numerics:ODE-example}.
\end{itemize}

\editend

Recall that for SINDy, we used the same kernel $\UK$ for derivative approximation and employed a polynomial library for the dictionary terms, which matched the features of the degree-2 polynomial used in KEqL. These same features were also used by PINN-SR. For PINN-SR, we fixed the hyperparameters across different experiments. For example, in the fixed IC case, where we aim to learn a PDE by varying the number of observation points, we selected a set of hyperparameters—such as network size, loss function weights, and iteration steps—and kept them the same for every experiment with different numbers of observation points. The same approach was applied when varying the initial condition. This strategy is consistent with the kernel method and SINDy, as we also maintained the same hyperparameters across different experiments for those methods.

\editstart
We also repeated this experiment with different neural network architectures for the same fixed and varying IC examples, shown in \Cref{fig:burgers_main_pinnsr-SM}, complementing \Cref{fig:burgers_main} to show that the results are not specific to a particular architecture choice, and that models vary a lot, and sometimes not being able to converge for other architectures of choice, it is again open to the user on choosing this architectures and we happened to choose what was advice in the original paper \cite{chen2021physics}.
\editend

\begin{figure}[htp]
\centering
\scriptsize
\begingroup
\setlength{\tabcolsep}{0pt}
\begin{tabular}{cc}
\vspace{4ex}
\adjustbox{valign=c}{
\begin{overpic}[width=0.32\textwidth]{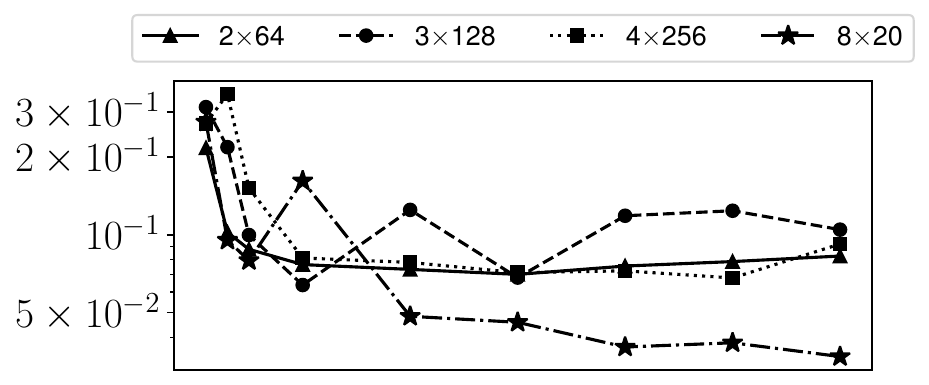}
\put(40,43){\scriptsize $\mathcal{R}_{\rm{rec}}(\rm train)$}
\put(-12,-7){\rotatebox{90}{
\begin{minipage}{0.15 \textwidth}
\centering
(A) Fixed IC \\ Errors
\end{minipage}
}}
\end{overpic}
}
&
\adjustbox{valign=c}{
\begin{overpic}[width=0.32\textwidth]{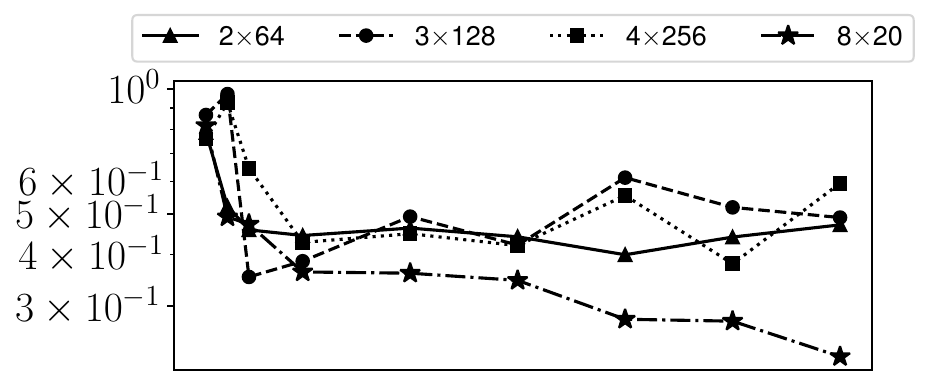}
\put(40,43){\scriptsize $\mathcal{R}_{\rm{eql}}\left(\rm train\right)$}
\end{overpic}
\hspace{8ex}
}
\\
\vspace{4ex}
\adjustbox{valign=c}{
\begin{overpic}[width=0.32\textwidth]{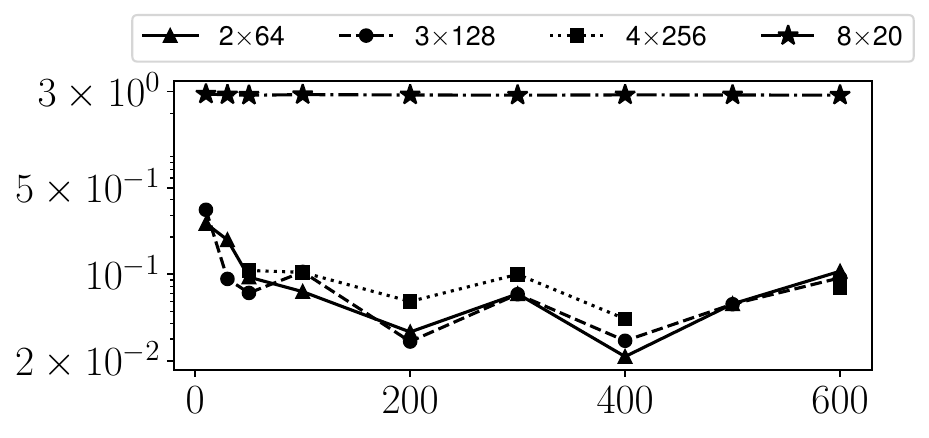}
\put(52,0){\scriptsize $N_\mY$}
\put(-12,-2){\rotatebox{90}{ 
\begin{minipage}{0.15 \textwidth}
\centering
(B) Varying IC \\ Errors
\end{minipage}
} }
\end{overpic}
}
&
\adjustbox{valign=c}{
\begin{overpic}[width=0.32\textwidth]{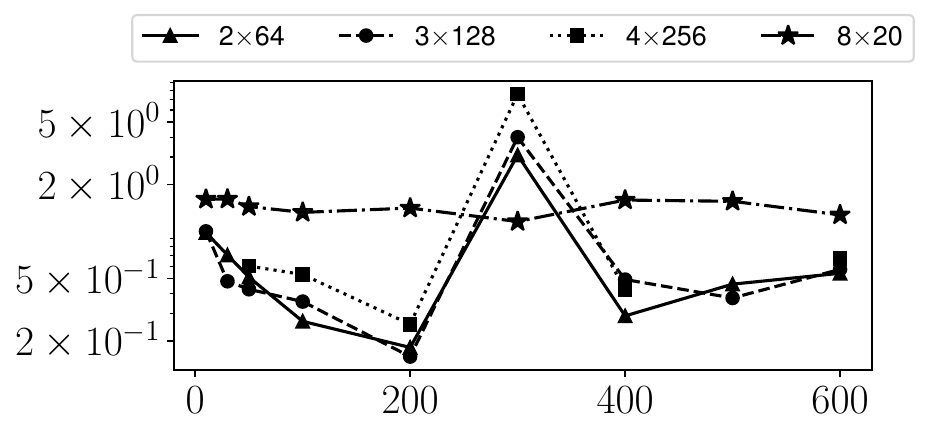}
\put(52,0){\scriptsize $N_\mY$}
\end{overpic}
\hspace{8ex}
}
\end{tabular} 
\endgroup
\caption{
\editstart
Numerical results for PINN-SR recovery of Burgers' PDE \eqref{burgers-PDE} using different neural network architectures (e.g., $2 \times 64$ denotes a fully connected feed-forward network with 2 layers of 64 neurons each), complementing \Cref{fig:burgers_main} with additional architectures of the method:
(a) State reconstruction and equation learning errors on the training functions, using $M=1$ training pairs and varying the number of interior observations $N_\mY$;
(b) Same experiment as panel (a) but with randomized initial conditions.
The reference architecture, used as the baseline and averaged over the randomized data, is the $8 \times 20$ one, which was used in \Cref{fig:burgers_main}.
A missing marker of a plot indicates that the experiment diverged.
\editend
}
\label{fig:burgers_main_pinnsr-SM}
\end{figure}

\subsubsection{An experiment with a smoothed 
shock} \label{subsubsec: exp_shock}

Here, we considered the PDE \eqref{burgers-PDE} with coefficients $\vartheta = 5$ and $\nu = 0.001$, using the prescribed IC
\begin{equation}
u_0(x) = \frac{1}{4} \left(x \sin(\pi x) - \sin(5 \pi x) - \sin(3 \pi x)\right).
\label{eq: ic_burgers_one_shot}
\end{equation}
We then solved the PDE, and the solution's behavior at the initial and final times is shown in \Cref{fig: burgers_IC_FC}(A). We then subsampled \(N_\mathcal{Y}=60\) interior observation points and \(N_{\partial \mathcal{Y}}=131\) boundary points from a selected \(26 \times 31\) Chebyshev collocation grid in space-time, as illustrated in \Cref{fig:burgers_shock_data}.
\begin{figure}[htp]
    \centering
    \begin{overpic}[scale=0.5]{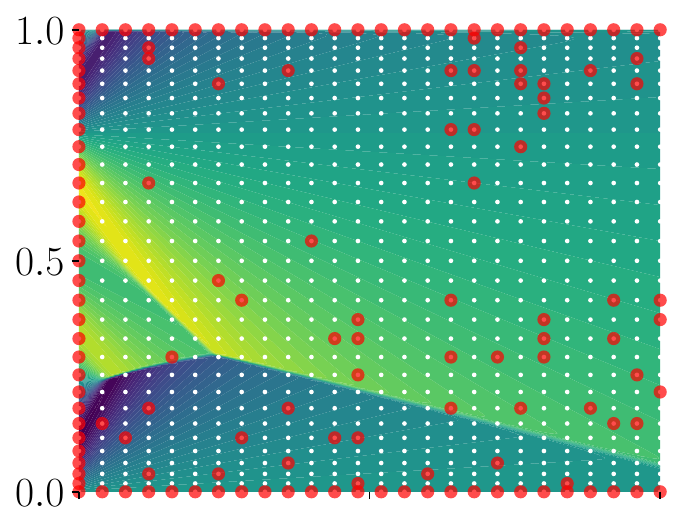}
        \put(-3,38){\rotatebox{90}{$x$}}
        \put(53.5,-1.5){{$t$}}
        \put(99,4){{$1.0$}}
    \end{overpic}
    \caption{Solution of the Burgers’ PDE \eqref{burgers-PDE} with initial condition \eqref{eq: ic_burgers_one_shot}. Sampled observation points are shown in red, while collocation points are overlaid in white.}
    \label{fig:burgers_shock_data}
\end{figure}

Kernels and hyperparameters were chosen as in \Cref{sec:burgers-experiments-with-increasing-observations}, where once again we used the hyperparameter information from the 2-step KEqL methodology to guide the selection of hyperparameters for the 1-step KEqL, using  $N = N_{\mY} + N_{\partial\mY}$ observations. For SINDy we used the same kernel $\UK$ for derivative approximation and employed a polynomial library for the dictionary terms, which matched the features of the degree-2 polynomial used in KEqL. These same features were also used by the PINN-SR method.

We clearly observe that the 1-step KEqL method remains competitive in shock state reconstruction. This is evident from the relative state reconstruction errors presented in \Cref{tab: burgers_errs_table}, where the 1-step KEqL achieves comparable results in shock recovery to those of SINDy and PINN-SR. Convergence history of the optimization algorithm for learning the equation is shown in \Cref{fig: burgers_loss_hist}(C).

\subsubsection{The effect of dictionaries for Burgers with a smoothed 
shock} \label{subsubsec: exp_shock_ablation}
\editstart
Here, we consider once more the PDE \eqref{burgers-PDE} with coefficients $\vartheta = 5$ and $\nu = 0.001$, subject to the initial condition \eqref{eq: ic_burgers_one_shot}, with same observations subsampled from the $26 \times 31$ Chebyshev collocation grid in space-time.

To learn the dynamics, we adopt the same kernels and hyperparameters for each method as in \Cref{subsubsec: exp_shock}, but now test two variants per method, one reduced feature set in which the degree-2 polynomial features are restricted to $\{u, \partial_x u\}$ (excluding $\partial_{x}^2 u$), and one expanded in which the feature set is $\{u, \partial_x u, \partial_x^2 u, \partial_{tx} u, \partial_{t}^2 u\}$. Results are reported in \Cref{tab: burgers_errs_table}.

Overall, the results across all methods share the same order of error, with the expanded feature set yielding more accurate results than the reduced set. We note that, upon removal of the diffusion feature, the PINN-SR method reduces to predicting an equation containing only the $u$ term, in contrast to the 1-step kernel method, which retains a linear combination over the full feature set. The kernel interpolant is unchanged as it is independent of the step of inferring the dynamics.

\editend

\begin{figure}[htp]
    \centering
    \begingroup
    \footnotesize
    \setlength{\tabcolsep}{0pt}
    \begin{tabular}{lcl}
        \begin{overpic}[scale=0.33]{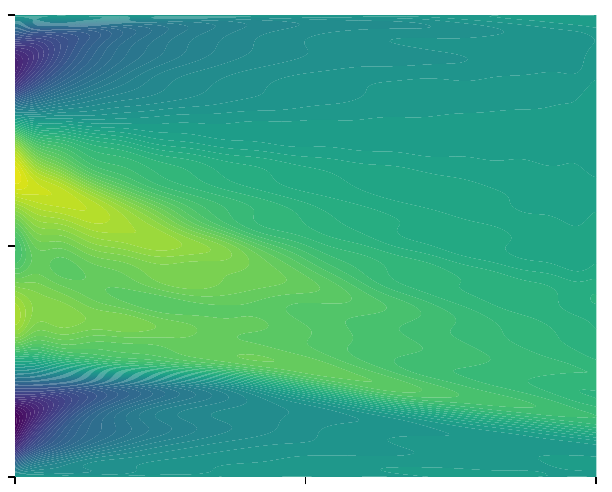} 
        \put(36,82){{1-step}}
        \put(-5,38){\rotatebox{90}{{$x$}}}
        \put(-18,20){\rotatebox{90}{{(A) Ex. 1a}}}
        \put(-10,75){{$1.0$}}
        \put(-10,0){{$0.0$}}
        \end{overpic}
        &
        \hspace{1ex}\raisebox{0. \height}{
        \begin{overpic}[scale=0.33]{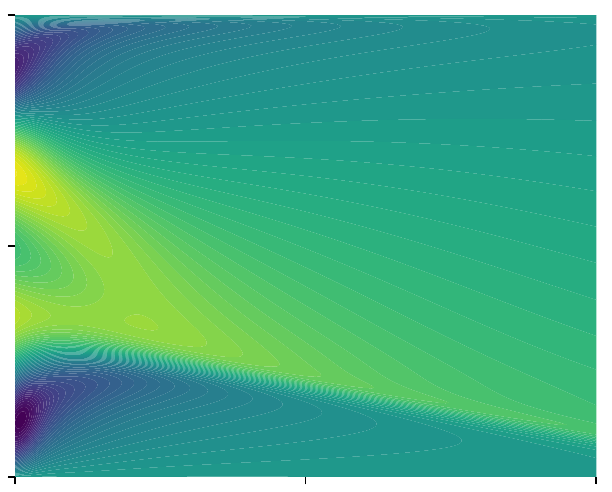}
        \put(32,81){{PINN-SR}}
        \end{overpic}
        }
        & 
        \hspace{-0.3ex}
        \raisebox{0.0 \height}{
        \begin{overpic}[scale=0.33]{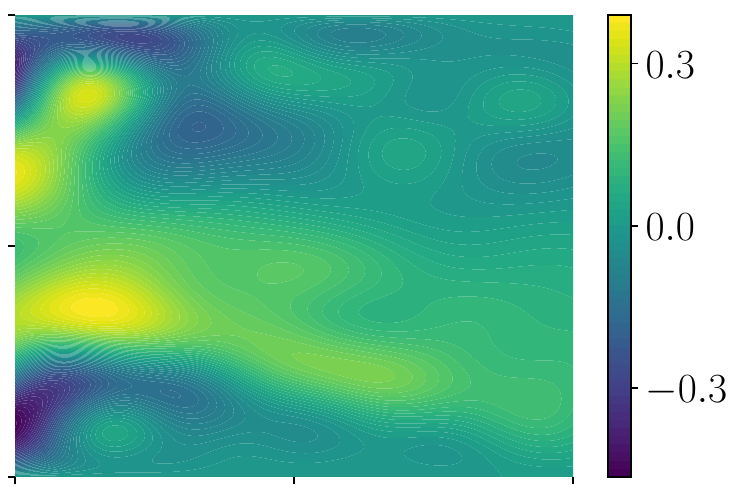}
        \put(19,66.5){{ Interpolant}}
        \end{overpic}
        }
        \\
        \begin{overpic}[scale=0.33]{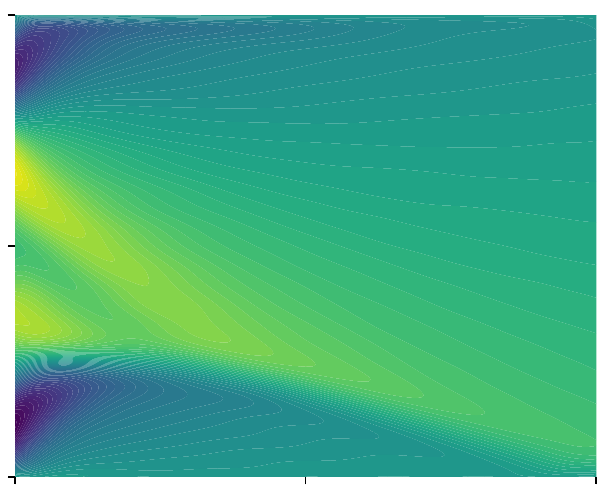} 
        \put(-5,35){\rotatebox{90}{{ $x$}}}
        \put(-10,75){{$1.0$}}
        \put(-18,20){\rotatebox{90}{{(B) Ex. 1b}}}
        \end{overpic}
        &
        \hspace{0ex}
        \begin{overpic}[scale=0.33]{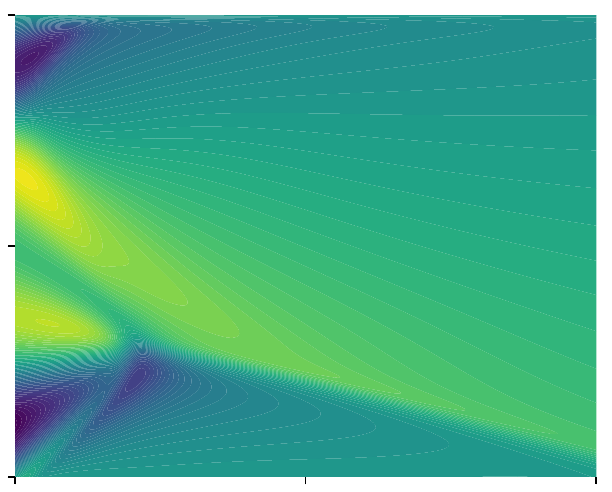}
        \end{overpic}
        & 
        \hspace{0.5ex}
        \begin{overpic}[scale=0.33]{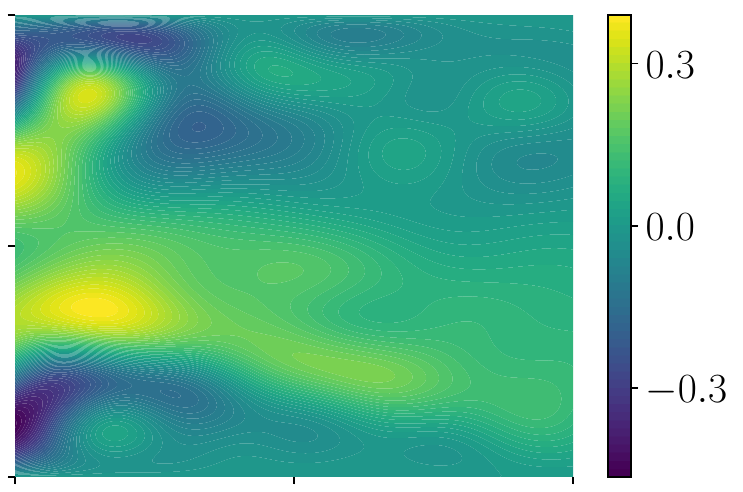}
        \end{overpic}
        \\
        \begin{overpic}[scale=0.33]{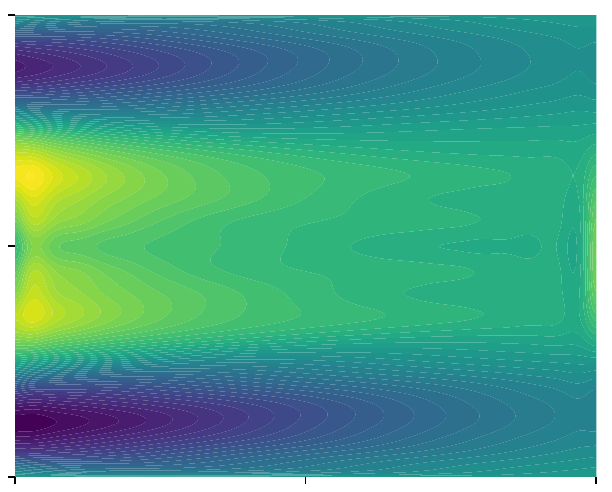} 
        \put(-5,38){\rotatebox{90}{{$x$}}}
        \put(-18,20){\rotatebox{90}{{(C) Ex. 2a}}}
        \put(-10,75){{$1.0$}}
        \put(-10,0){{$0.0$}}
        \end{overpic}
        &
        \hspace{1ex}\raisebox{0. \height}{
        \begin{overpic}[scale=0.33]{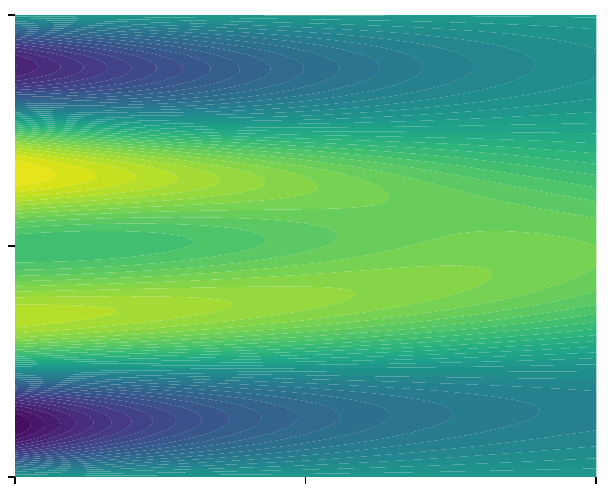}
        \end{overpic}
        }
        & 
        \hspace{-0.3ex}
        \raisebox{0.0 \height}{
        \begin{overpic}[scale=0.33]{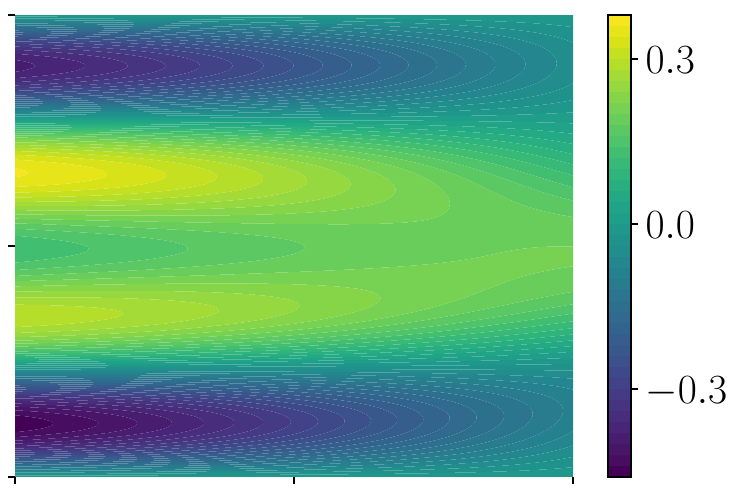}
        \end{overpic}
        }
        \\
        \begin{overpic}[scale=0.33]{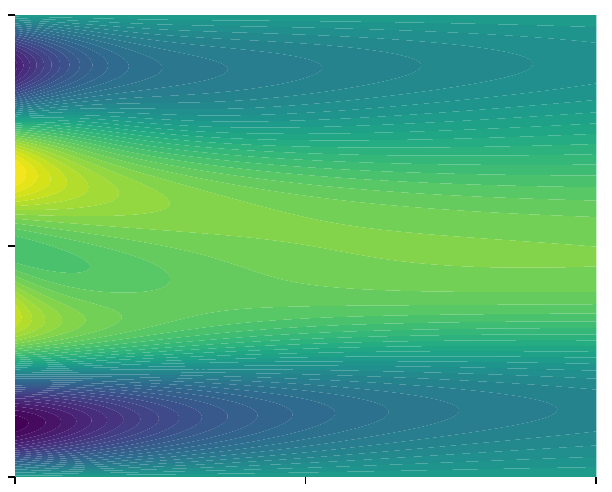} 
        \put(-5,35){\rotatebox{90}{{ $x$}}}
        \put(-10,75){{$1.0$}}
        \put(-8,-6){{$0.0$}}
        \put(93,-6){{$1.0$}}
        \put(45,-8){{ $t$}}
        \put(-18,20){\rotatebox{90}{{(D) Ex. 2b}}}
        \end{overpic}
        &
        \hspace{0ex}
        \begin{overpic}[scale=0.33]{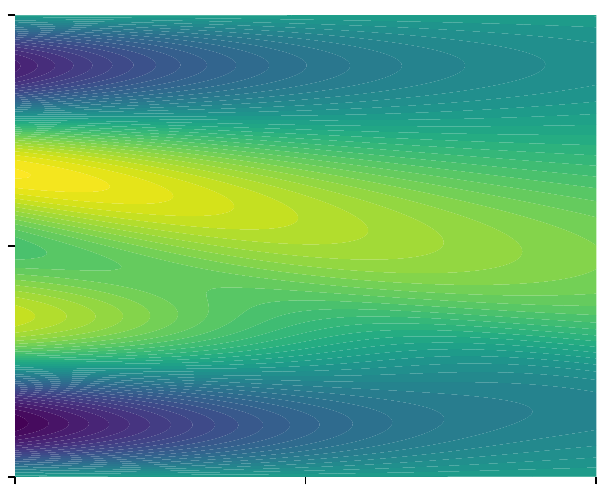}
        \put(0,-6){{$0.0$}}
        \put(93,-6){{$1.0$}}
        \put(46,-8){{ $t$}}
        \end{overpic}
        & 
        \hspace{0.5ex}
        \begin{overpic}[scale=0.33]{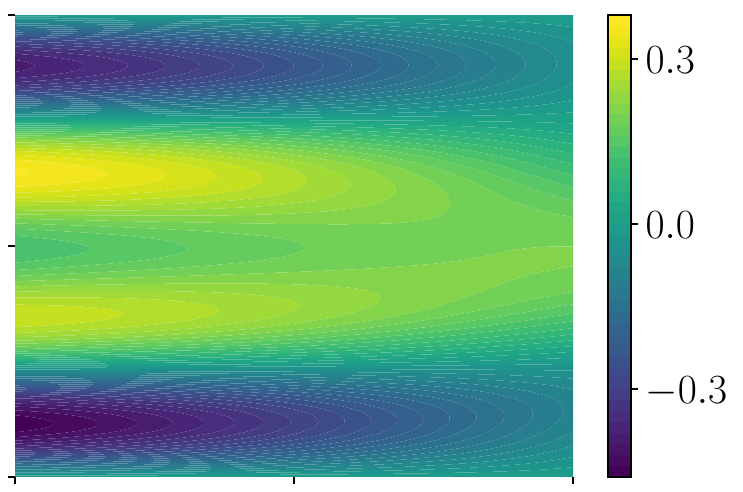}
        \put(0,-5){{$0.0$}}
        \put(72,-5){{$1.0$}}
        \put(39,-6){{$t$}}
        \end{overpic}
    \end{tabular}

    \endgroup
    \caption{Representative numerical results for the ablation studies on Burgers' PDE \eqref{burgers-PDE} with initial condition \eqref{eq: ic_burgers_one_shot}, examining the effect of removing/adding features across the 1-step KEqL, PINN-SR, and kernel interpolation methods for the smoothed shock in \Cref{subsubsec: exp_shock_ablation} shown in Ex.~1a--b, and a smoother transition case in \Cref{subsubsec: one-shot-ablation-study} shown in Ex.~2a--b: (A,C) solutions recovered using the reduced feature set $\{1, u, \partial_x u\}$; (B,D) solutions recovered using the expanded feature set $\{1, u, \partial_x u, \partial_x^2 u, \partial_{tx} u, \partial_{t}^2 u\}$.}
    \label{fig: ablation_filtering_results}
\end{figure}

\subsubsection{A one-shot learning experiment} \label{subsubsec: one-shot}

Here, we considered the PDE \eqref{burgers-PDE} with coefficients $\vartheta = 0.5$ and $\nu = 0.01$, using the prescribed IC given by \eqref{eq: ic_burgers_one_shot}. We then solved the PDE, whose solution behavior at initial and final times are shown in \Cref{fig: burgers_IC_FC}(B). Here we subsampled $N=100$ points placed on the boundary, see first column in \Cref{fig:burgers_main}(A). 
The placement of the points amounts to observing the initial and end conditions along with the boundary values in the physical domain.
To generate data for operator learning, we additionally solved the same PDE with a new initial condition, $u_0(x) = -x \sin(2 \pi x)$. The solution behavior at the initial and final times is shown in \Cref{fig: burgers_IC_FC}(C).

\begin{figure}[htp]
    \centering
    \vspace{2ex}
    \begingroup
    \footnotesize
    \setlength{\tabcolsep}{0pt}
    \begin{tabular}{ccc}
        \begin{overpic}[scale=0.33]{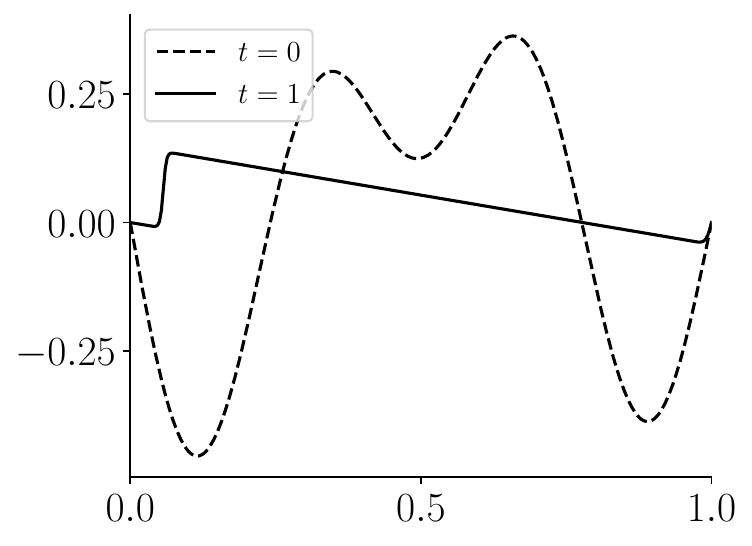} 
        \put(0,37){\rotatebox{90}{{ $u$}}}
        \put(54,-2){{$x$}}
        \put(40,75){{ (A) Ex. 1}}
        \end{overpic}
        &
        \begin{overpic}[scale=0.33]{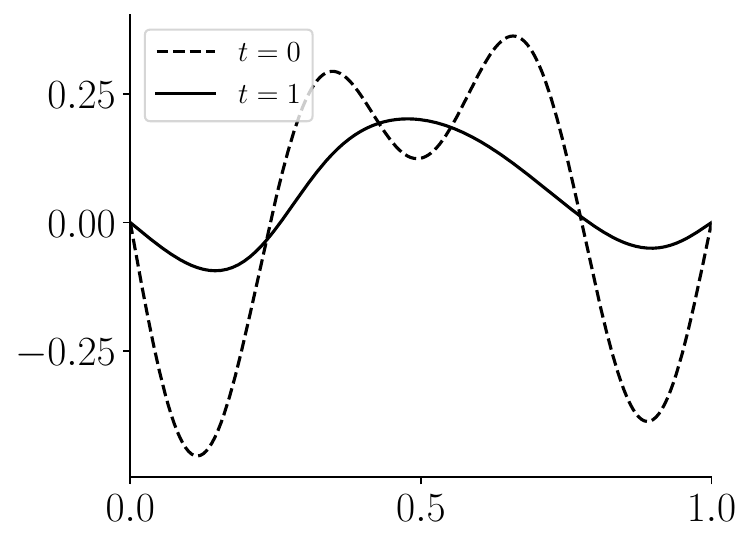}
        \put(54,-2){{$x$}}
        \put(40,75){{ (B) Ex. 2}}
        \end{overpic}
        & 
        \begin{overpic}[scale=0.33]{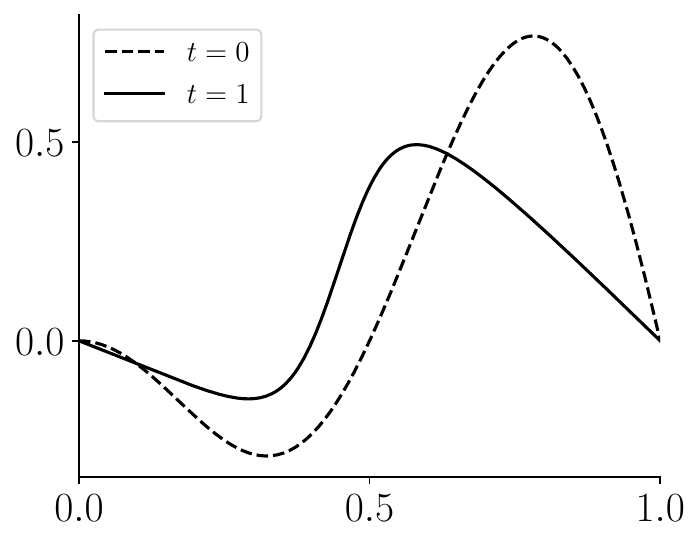}
        \put(51,-2){{$x$}}
        \put(35,80){{(C) New IC}}
        \end{overpic}
    \end{tabular}

    \endgroup
        \caption{Behavior of solutions at initial and final time for one-shot learning examples of Burgers' PDE \eqref{burgers-PDE}: 
(A) Shock development (\Cref{subsubsec: exp_shock}); (B) Smooth transition  (\Cref{subsubsec: one-shot}); (C) Solution behavior to the PDE in (B) for the depicted new initial condition.
}
    \label{fig: burgers_IC_FC}
\end{figure}

\editstart
Kernels and hyperparameters were chosen as follows:
\begin{itemize}
    \item For the kernels, we hand-tuned $\UK$ as a RBF kernel with a diagonal matrix $\Sigma = 0.0125 I$. We chose $\PK$ to be a polynomial kernel of degree 2, with the shift parameter set to $c = \operatorname{mean}(S_{2:})$ and the scaling matrix given by $B = \operatorname{diag}(\operatorname{cov}(S_{2:}))^{-1}$, where $\operatorname{cov}(A)$ represents the sample covariance of $A$. Specifically,

\begin{equation} \label{eq: S_to_define_B-2}
S = \Phi(\whu_{2}, Y^1) := 
\begin{bmatrix}
    \Phi(\whu_{2}, y_1) \\
    \vdots
    \\
    \Phi(\whu_{2}, y_N)
\end{bmatrix}
\in \mathbb{R}^{N \times 5},
\end{equation}
where 
 $\whu_2$ is the learned solution obtained from the first step of the 2-step KEqL method, using the hand-tuned hyperparameters mentioned above.
 \item We employed a fixed basis with 
$I = 26 \times 26$ inducing points for learning the equation (see \Cref{app:algorithms:nystromapproximation}).
 \item Following the definitions in \eqref{eq: hyperparams}, we set $\lambda = 1$,  $\theta_u = 5e-18$ and $\theta_P = 5e-16$ as manually chosen hyperparameters, in order to nearly omit the RKHS norms in the loss function, and aligned with the heuristic described in \Cref{app:numerics:ODE-example}.
\end{itemize}
\editend
Recall that a similar rationale was used for selecting hyperparameters in SINDy and PINN-SR, including the choice of dictionary terms, as in the previous experiment discussed in \Cref{subsubsec: exp_shock}.

We clearly observe that the 1-step KEqL method outperforms the other methods in reconstructing the smooth solution. This is evident from results presented in \Cref{tab: burgers_errs_table}, where the one-step KEqL achieves a two-order-of-magnitude improvement over the other methods in state reconstruction and a one-order-of-magnitude improvement in the operator learning task for the smooth case.
The one-shot learning example for the smooth case further confirms that effective state reconstruction is essential for accurately capturing the operator. While PINN-SR and SINDy provided a reasonable reconstruction of the solution, the superior state reconstruction of the 1-step KEqL method resulted in a significantly better operator learning error for a new initial condition.
We recall that the variational solution proposed in \eqref{eq:operator-learning} was used to solve the recovered PDE in each of the methods for this experiment, employing the same kernel $\UK$, to report the operator learning results.
Convergence history of the optimization algorithm for learning the equation is shown in \Cref{fig: burgers_loss_hist}(D).

\begin{table}
    \centering
    \editstart
\begin{tabular}{cc|ccc}
     & & & Method &  \\
    \hline
     & Experiment & 1-step & PINN-SR & SINDy \\
    \hline
    \multirow{6}{*}{$\mathcal{R}_{\text{rec}}$} 
     & 1 & \num{3.9e-1} & \num{3.4e-1} & \num{6.2e-1} \\
     & {1 - Reduced features} 
        & \num{4.7e-1} & \num{3.6e-1} & \num{6.2e-1} \\
     & {1 - Expanded features} 
        & \num{3.7e-1} & \num{3.2e-1} & \num{6.2e-1} \\
     & 2 & \num{3.5e-3} & \num{1.5e-1} & \num{2.9e-1} \\
     & {2 - Reduced features} 
        & \num{3.6e-1} & \num{1.8e-1} & \num{2.9e-1} \\
     & {2 - Expanded features} 
        & \num{1.6e-1} & \num{1.5e-1} & \num{2.9e-1} \\
    \hline
    \multirow{3}{*}{$\mathcal{R}_{\text{opl}}$}
     & 2 & \num{1.9e-2} & \num{4.7e-1} & \num{5.4e-1} \\
     & {2 - Reduced features} 
        & \num{6.4e-1} & \num{4.2e-1} & \num{5.4e-1} \\
     & {2 - Expanded features} 
        & \num{2.0e-2} & \num{4.4e-1} & \num{5.4e-1} \\
    \hline  
\end{tabular}
\editend
\caption{
    \editstart
    Quantitative values of relative state reconstruction and operator learning errors for Burgers' PDE \Cref{burgers-PDE} with different ICs corresponding to the one-shot learning examples Ex. 1 (\Cref{subsubsec: exp_shock}), its ablation studies Ex.~1a--b, in which features are reduced and expanded (\Cref{subsubsec: exp_shock_ablation}), Ex. 2 (\Cref{subsubsec: one-shot}) in  \Cref{fig:burgers_main}(C-E), and  its ablation studies Ex.~2a--b, in which features are reduced and expanded (\Cref{subsubsec: one-shot-ablation-study}.)
    \editend
    }
    \label{tab: burgers_errs_table}
\end{table}

\begin{figure}[htp]
    \centering
    \vspace{3ex}
    \begingroup
    \setlength{\tabcolsep}{0pt}
    \begin{tabular}{cccc}
        \begin{overpic}[scale=0.26]{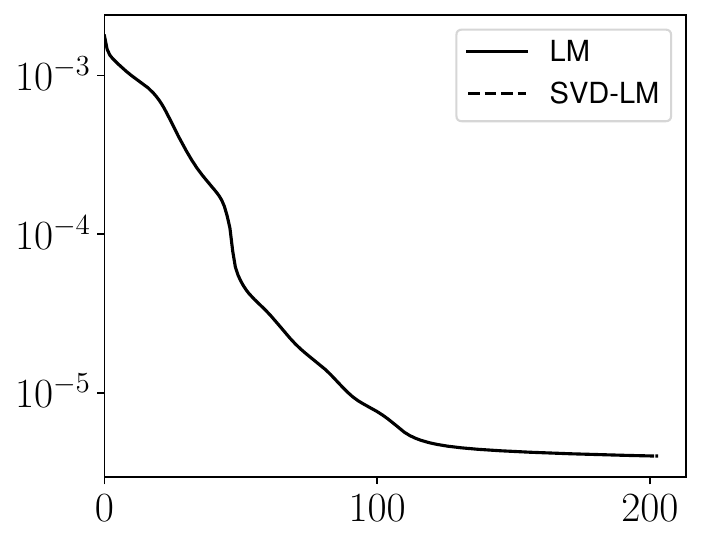}
        \put(47,-7){{\footnotesize \rm Steps}}
        \put(-6,33){\rotatebox{90}{{\footnotesize Loss}}}
        \put(30,82){{\footnotesize (A) Fixed IC}}
        \end{overpic}
        &
        \begin{overpic}[scale=0.26]{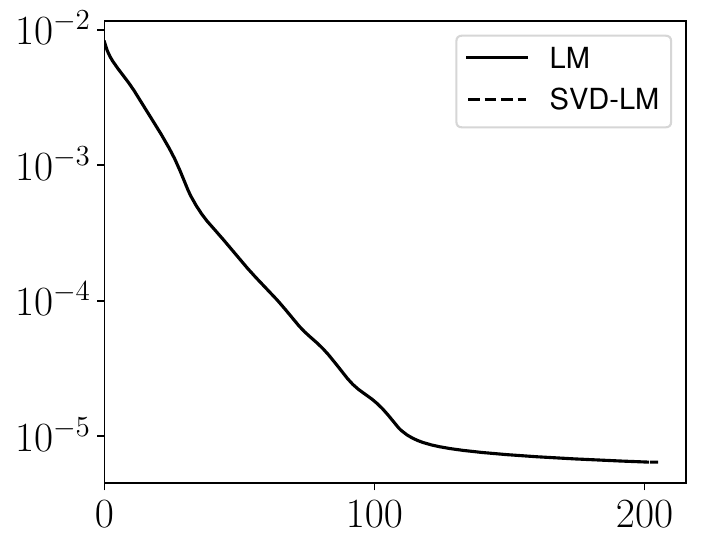}
        \put(47,-7){{\footnotesize \rm Steps}}
        \put(20,82){{\footnotesize (B) Varying ICs}}
        \end{overpic}
        &
        \begin{overpic}[scale=0.26]{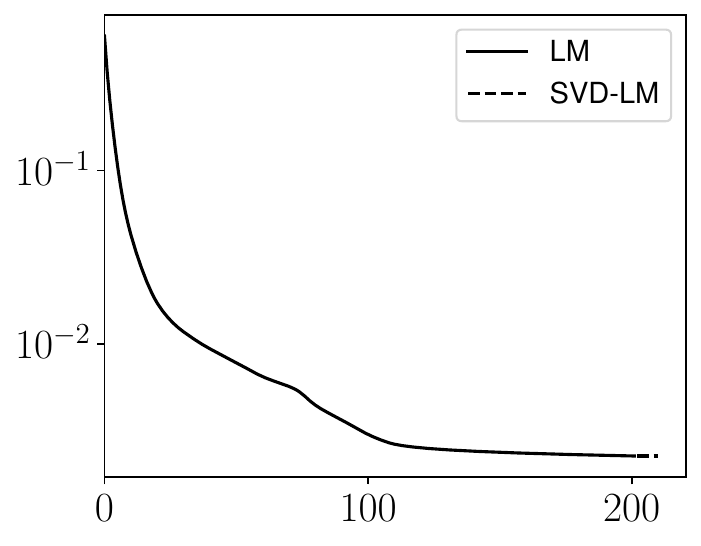}
        \put(47,-7){{\footnotesize \rm Steps}}
        \put(33,82){{\footnotesize (C) Ex. 1}}
        \end{overpic}
        &
        \begin{overpic}[scale=0.26]{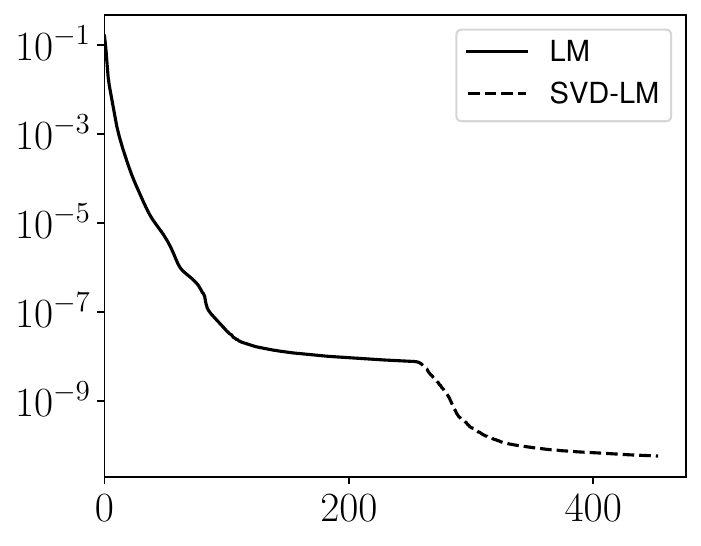}
        \put(44,-7){{\footnotesize \rm Steps}}
        \put(32,82){{\footnotesize (D) Ex. 2}}
        \end{overpic}
        
    \end{tabular}

    \endgroup
    \caption{Convergence history of the LM algorithm for 1-step KEqL for Burgers' PDE \eqref{burgers-PDE}: (A-B) Sample cases using fixed and varying IC, using $M=1$ training pairs with number of observations $N_\mY=20$ corresponding to the two cases depicted in the  first column in \Cref{fig: training_data_Burgers}; 
    (C) Shock development (\Cref{subsubsec: exp_shock}); (D) One-shot smooth case where only observed on the boundary (\Cref{subsubsec: one-shot}).}
    \label{fig: burgers_loss_hist}
\end{figure}

\subsubsection{The effect of dictionaries for the one-shot learning of Burgers' PDE} \label{subsubsec: one-shot-ablation-study}
\editstart
In this section we turn our attention to the choice of the dictionary of 
differential operators, i.e., the $L_q$'s, and their effect on 
the quality of the learned equation. We adopt the same kernels and hyperparameters as in \Cref{subsubsec: one-shot} for the one-shot learning problem, but now 
consider a reduced dictionary of differential operators $\{u, \partial_x u\}$ (excluding $\partial_x^2 u$), and an expanded dictionary $\{u, \partial_x u, \partial_x^2 u, \partial_{tx} u, \partial_t^2 u\}$. Results are reported in \Cref{tab: burgers_errs_table}, both for the state reconstruction errors and for the operator learning error obtained by solving the learned PDE with a new initial condition.

Overall, errors are markedly larger than in the fully-observed setting of \Cref{subsubsec: exp_shock}, reflecting the added difficulty of learning from boundary-only data, and the expanded feature set again yields more accurate results than the reduced one for the 1-step method, in both state reconstruction and operator learning. For PINN-SR, however, the expanded set only improves the state reconstruction error; operator learning is essentially unchanged between the two variants. This is consistent with the recovered coefficients where rather than identifying a qualitatively different equation once the diffusion feature is included, PINN-SR collapses to essentially the same equation in both variants, an approximately linear decay $\partial_t u \approx -1.16 u$ with no meaningful dependence on $\partial_x u$, $\partial_x^2 u$, or the nonlinear advection term $u\partial_x u$, whose coefficients remain at or near zero whether or not the corresponding features are supplied. In other words, under boundary-only observations PINN-SR fails to identify the diffusive or advective structure of the equation regardless of feature availability, in contrast to the 1-step kernel method, which retains a linear combination over the full feature set in both variants. As before, the kernel interpolant is unchanged across variants, since it is independent of the step of inferring the dynamics.
\editend

\subsection{Darcy's flow PDE}\label{sec:numerics:darcy-SI}
In this example we performed a systematic study of the 
performance of 1-step (its reduced version) and 2-step KEqL for learning an 
elliptic PDE with a variable diffusion coefficient.
In particular, we investigate the ID and OOD performance 
in terms of 
\editstart
state reconstruction,
\editend
equation learning, and operator learning.

Let us recall our problem setting. Here we took $\mY = (0,1) \times (0,1)$ and  considered the problem
\begin{equation}\label{darcy_PDE_SI}
\left\{
\begin{aligned}
        &\P(u) = \operatorname{div}\left(a\nabla u\right) = f, \text{ in } \mY, \\
        &u(y) = g(y), \text{ for } y\in \partial\mY,
\end{aligned}
\right.
\end{equation}
where $a(y) = \exp\left(\sin( \cos(y_1)+ \cos(y_2))\right)$.
To generate the training data we drew $M = 2,4,8,16,32$ functions $u \sim \mathcal{GP}(0,\UK_{\rm RBF})$ with $\Sigma = 0.5^2 I$ and took $f:= \P(u)$, also considering the value of $u$ at $\partial\mY$ as the boundary condition. Each $u$ was then subsampled observation points with $N_\mY = 2,4,8$ in the interior and fixed $N_{\partial \mY} = 56$ at the boundary that were picked randomly from a $15\times 15$ uniform collocation grid where the PDE was enforced. Some training tuples with their observation and collocation points are depicted in \Cref{fig:darcy_training_us_and_fs}. We recall that the ID data was drawn from the same distribution that the training data was drawn from. The test data sets were generated similarly, with the OOD data drawn from a less regular GP using an RBF kernel with $\Sigma = 0.4^2 I$. Some ID and OOD samples are shown at \Cref{fig: darcy_ID_OOD_contours}(A). The fine grid $Y_{\rm test}$ where we tested the equation and operator learning error was chosen to be a uniform grid of size $100^2$.

\begin{figure}[htp]
\vspace{3ex}
\begingroup
\setlength{\tabcolsep}{0pt}
\begin{tabular}{lcccl}
\centering

\begin{overpic}[scale=0.25]{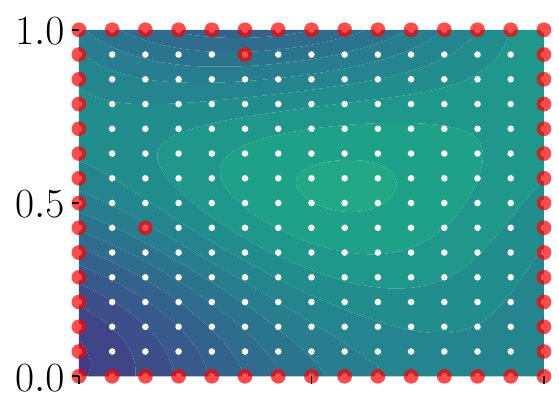}
\put(-6,34){\rotatebox{90}{{\scriptsize $x_2$}}}
\put(50,71){{\footnotesize $u^1$}}
% \put(-23,20){\rotatebox{90}{(B) Ex. 1}}
\end{overpic}

&

\raisebox{0.05 \height}{
\begin{overpic}[scale=0.25]{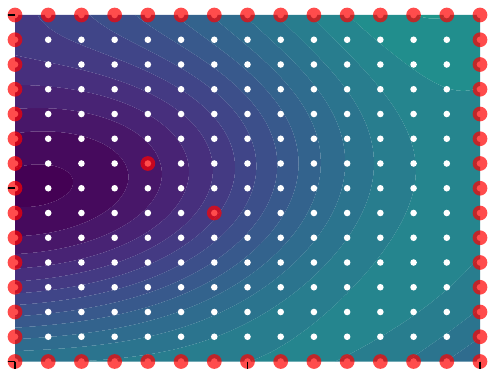}
\put(45,75){{\footnotesize $u^2$}}
\end{overpic}
}
&

\raisebox{0.05 \height}{
\begin{overpic}[scale=0.25]{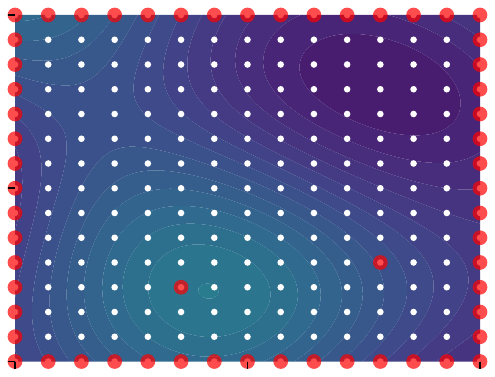}
\put(45,75){{\footnotesize $u^3$}}
\end{overpic}
}
&

\raisebox{0.05 \height}{
\begin{overpic}[scale=0.25]{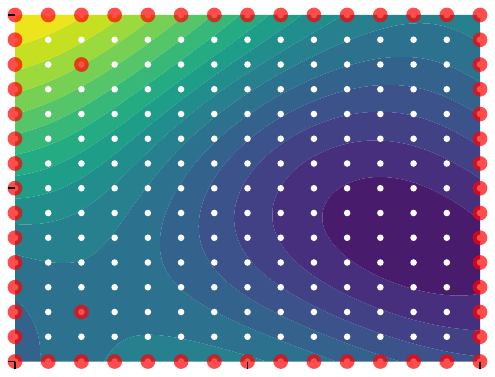}
\put(45,75){{\footnotesize $u^4$}}
\end{overpic}
}
&

\raisebox{0.05 \height}{
\begin{overpic}[scale=0.25]{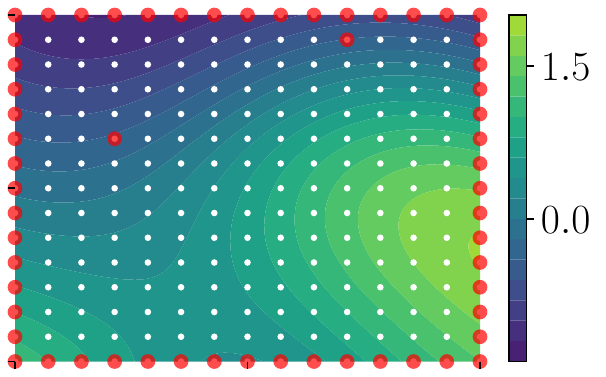}
\put(38,61){{\footnotesize $u^5$}}
\end{overpic}
}
\\

\begin{overpic}[scale=0.25]{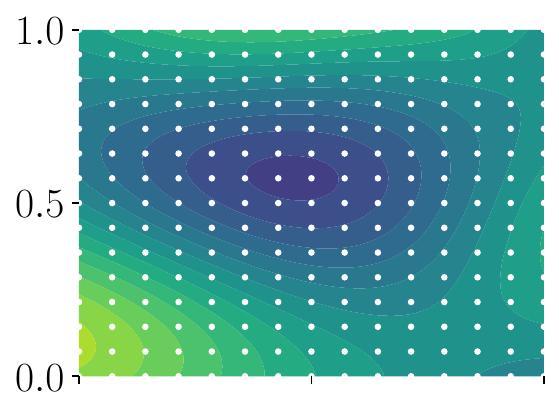}
\put(47,72){{\footnotesize $f^1$}}
\put(-6,34){\rotatebox{90}{{\scriptsize $x_2$}}}
\put(-26,18){\rotatebox{90}{{\footnotesize(A) Interior obs. $N_\mY=2$}}}
\end{overpic}
&
\raisebox{0.05 \height}{
\begin{overpic}[scale=0.25]{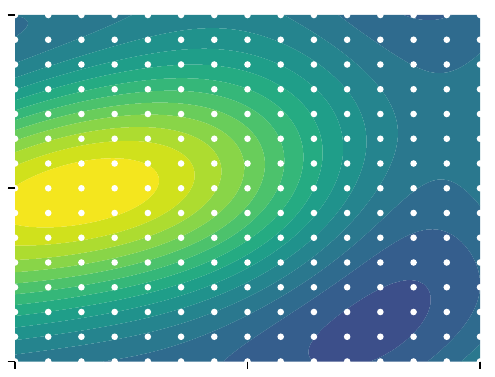}
\put(43,77){{\footnotesize $f^2$}}
\end{overpic}
}
&
\raisebox{0.05 \height}{
\begin{overpic}[scale=0.25]{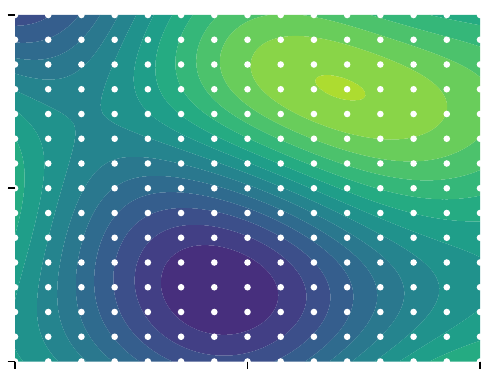}
\put(43,77){{\footnotesize $f^3$}}
\end{overpic}
}
&
\raisebox{0.05 \height}{
\begin{overpic}[scale=0.25]{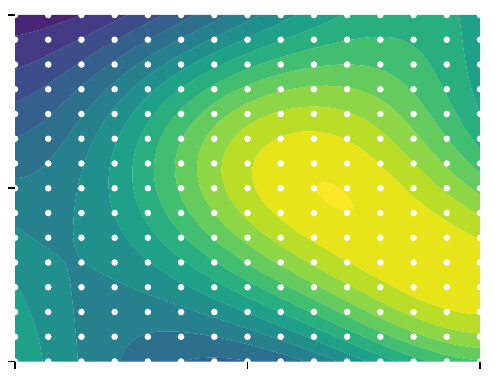}
\put(43,77){{\footnotesize $f^4$}}
\end{overpic}
}
&
\raisebox{0.05 \height}{
\begin{overpic}[scale=0.25]{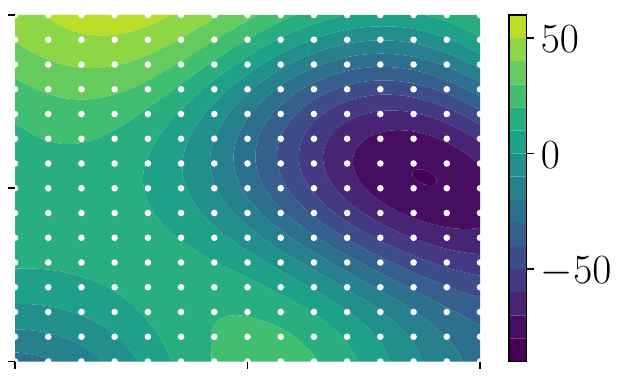}
\put(37,60){{\footnotesize $f^5$}}
\end{overpic}
}
\\
% \arrayrulecolor{gray}\hline
% \rule{-3pt}{17ex}
\begin{overpic}[scale=0.25]{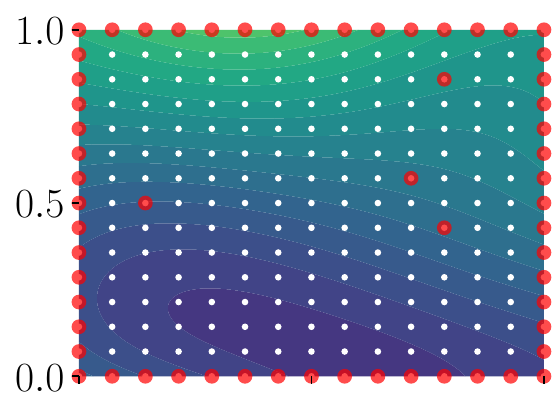}
\put(-6,34){\rotatebox{90}{{\scriptsize $x_2$}}}
\put(50,71){{\footnotesize $u^1$}}
% \put(-23,20){\rotatebox{90}{(B) Ex. 1}}
\end{overpic}
&
\raisebox{0.05 \height}{
\begin{overpic}[scale=0.25]{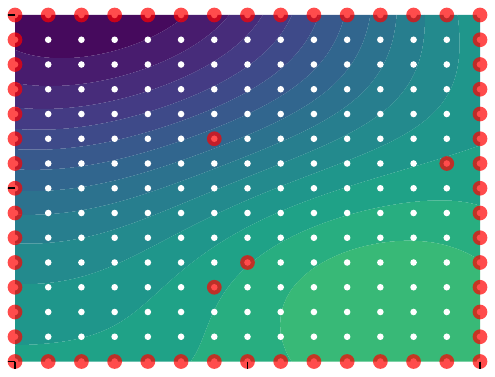}
\put(45,75){{\footnotesize $u^2$}}
\end{overpic}
}
&

\raisebox{0.05 \height}{
\begin{overpic}[scale=0.25]{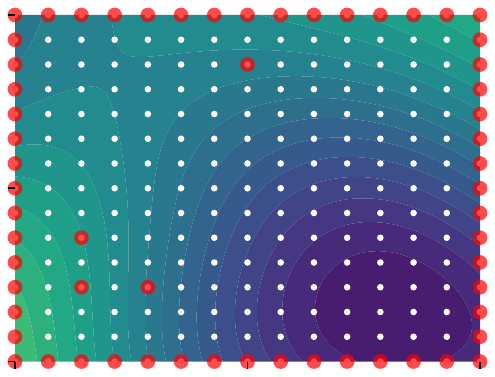}
\put(45,75){{\footnotesize $u^3$}}
\end{overpic}
}
&

\raisebox{0.05 \height}{
\begin{overpic}[scale=0.25]{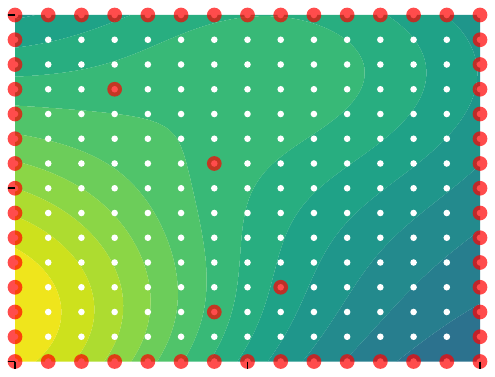}
\put(45,75){{\footnotesize $u^4$}}
\end{overpic}
}
&

\raisebox{0.05 \height}{
\begin{overpic}[scale=0.25]{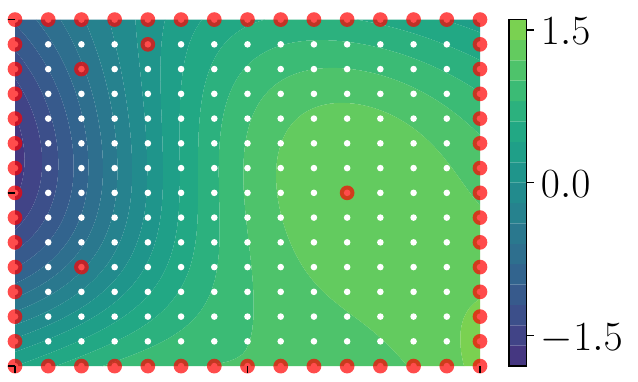}
\put(38,58){{\footnotesize $u^5$}}
\end{overpic}
}
\\

\begin{overpic}[scale=0.25]{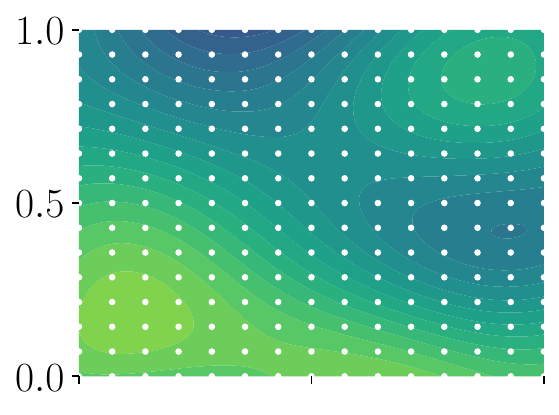}
\put(47,72){{\footnotesize $f^1$}}
\put(-6,34){\rotatebox{90}{{\scriptsize $x_2$}}}
\put(-26,20){\rotatebox{90}{{\footnotesize(B) Interior obs. $N_\mY=4$}}}
\end{overpic}
&
\raisebox{0.05 \height}{
\begin{overpic}[scale=0.25]{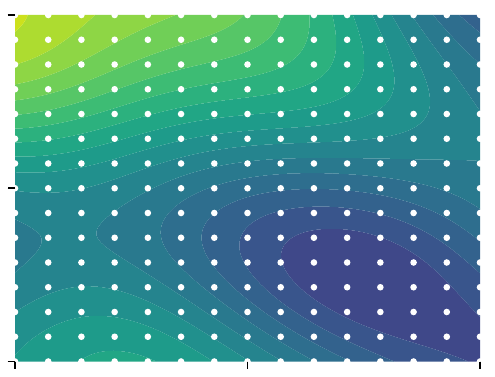}
\put(43,77){{\footnotesize $f^2$}}
\end{overpic}
}
&
\raisebox{0.05 \height}{
\begin{overpic}[scale=0.25]{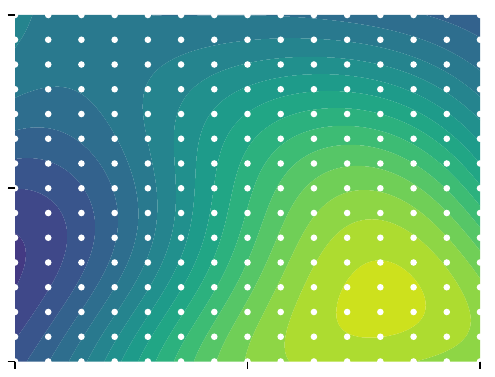}
\put(43,77){{\footnotesize $f^3$}}
\end{overpic}
}
&
\raisebox{0.05 \height}{
\begin{overpic}[scale=0.25]{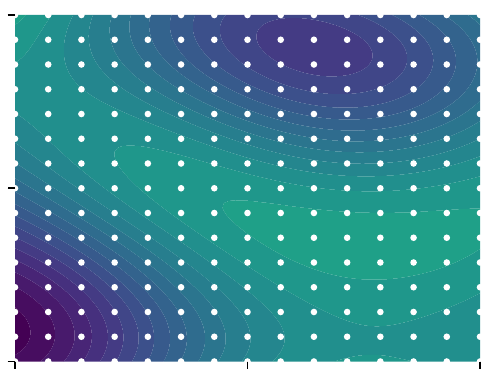}
\put(43,77){{\footnotesize $f^4$}}
\end{overpic}
}
&
\raisebox{0.05 \height}{
\begin{overpic}[scale=0.25]{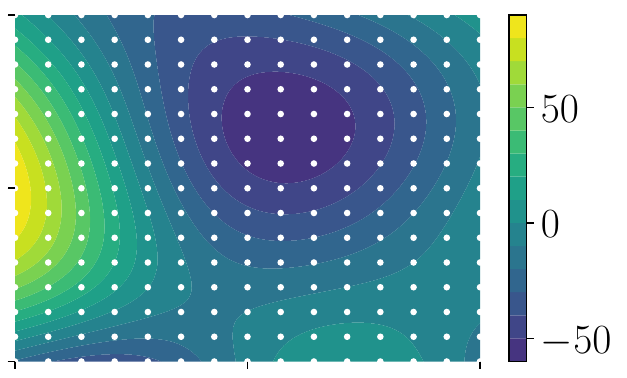}
\put(37,60){{\footnotesize $f^5$}}
\end{overpic}
}
\\
% \arrayrulecolor{gray}\hline
% \rule{-3pt}{17ex}
\begin{overpic}[scale=0.25]{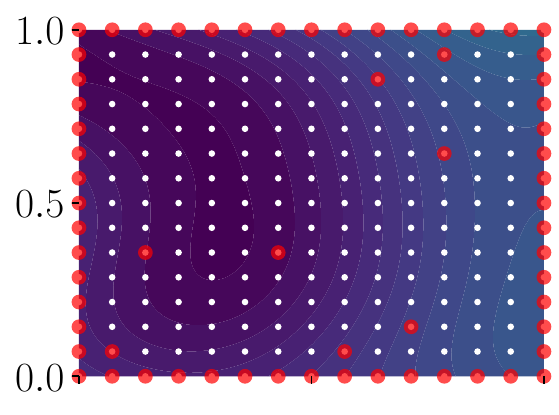}
\put(-6,34){\rotatebox{90}{{\scriptsize $x_2$}}}
\put(50,71){{\footnotesize $u^1$}}
% \put(-23,20){\rotatebox{90}{(B) Ex. 1}}
\end{overpic}
&
\raisebox{0.05 \height}{
\begin{overpic}[scale=0.25]{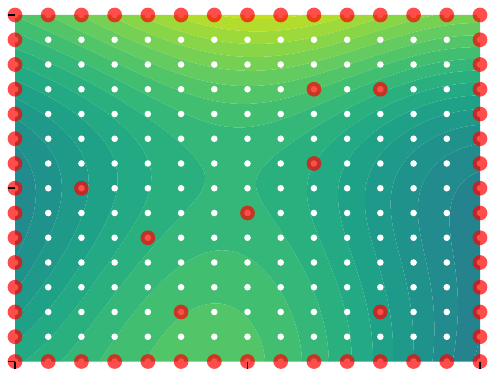}
\put(45,75){{\footnotesize $u^2$}}
\end{overpic}
}
&

\raisebox{0.05 \height}{
\begin{overpic}[scale=0.25]{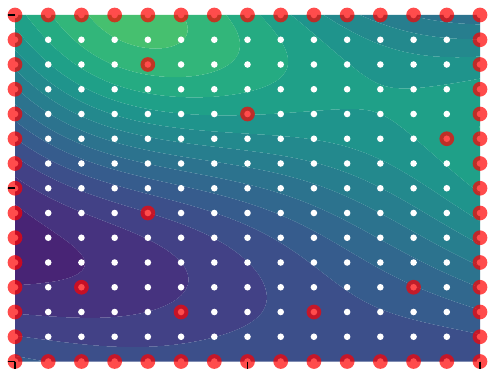}
\put(45,75){{\footnotesize $u^3$}}
\end{overpic}
}
&

\raisebox{0.05 \height}{
\begin{overpic}[scale=0.25]{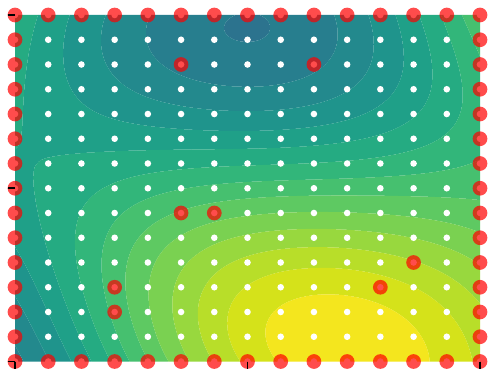}
\put(45,75){{\footnotesize $u^4$}}
\end{overpic}
}
&

\raisebox{0.05 \height}{
\begin{overpic}[scale=0.25]{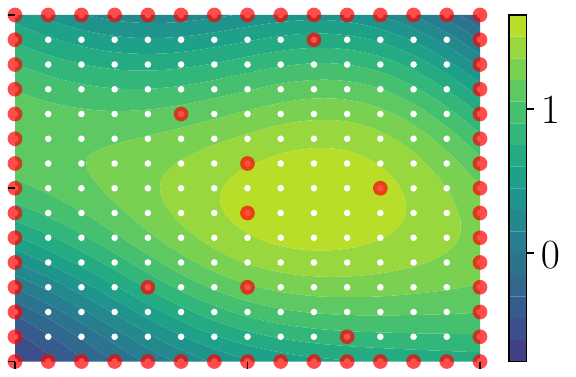}
\put(38,64){{\footnotesize $u^5$}}
\end{overpic}
}
\\
\begin{overpic}[scale=0.25]{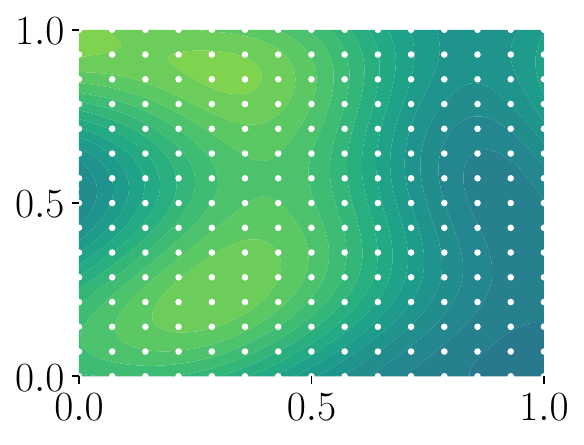}
\put(47,74){{\footnotesize $f^1$}}
\put(-6,37){\rotatebox{90}{{\scriptsize $x_2$}}}
\put(49,-3){{\scriptsize $x_1$}}
\put(-24,20){\rotatebox{90}{{\footnotesize(C) Interior obs. $N_\mY=8$}}}
\end{overpic}
&
\begin{overpic}[scale=0.25]{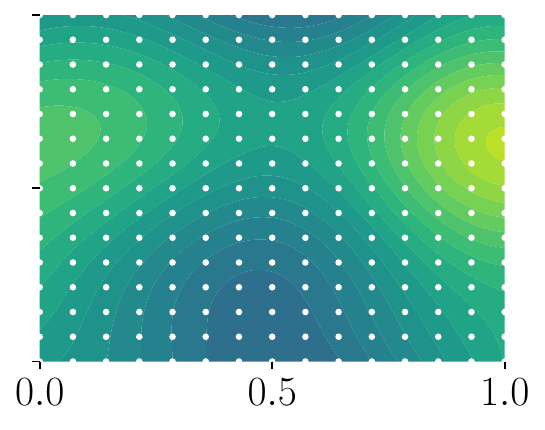}
\put(43,79){{\footnotesize $f^2$}}
\put(46,-3){{\scriptsize $x_1$}}
\end{overpic}
&
\begin{overpic}[scale=0.25]{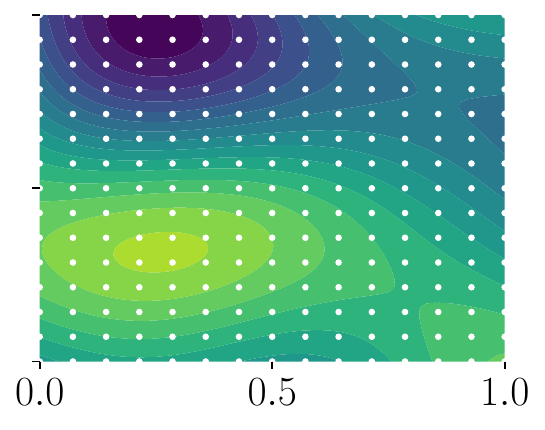}
\put(43,79){{\footnotesize $f^3$}}
\put(46,-3){{\scriptsize $x_1$}}
\end{overpic}
&
\begin{overpic}[scale=0.25]{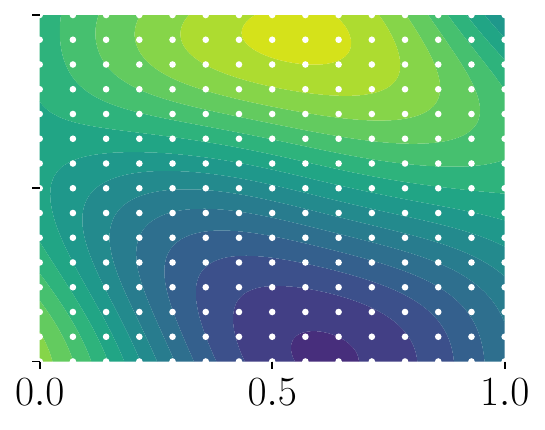}
\put(43,79){{\footnotesize $f^4$}}
\put(46,-3){{\scriptsize $x_1$}}
\end{overpic}
&
\begin{overpic}[scale=0.25]{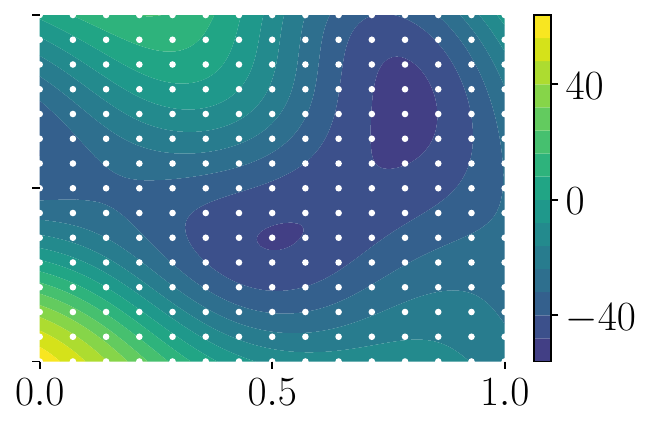}
\put(37,66){{\footnotesize $f^5$}}
\put(38,-3){{\scriptsize $x_1$}}
\end{overpic}
\end{tabular}
\endgroup
\caption{Samples of training tuples $(u^m,f^m)$ with different number of observation points for Darcy's flow 
PDE \eqref{darcy_PDE}: (A-C) Training tuples  observation points (red) and collocation points (white).}
\label{fig:darcy_training_us_and_fs}
\end{figure}

\editstart
Kernels and hyperparameters were chosen as follows: 
\begin{itemize}
    \item We used $\UK = \UK_{\rm RBF}$ with a hand-tuned value of $\Sigma = 0.5^2 I$ for learning $u$ and $\PK = \PK_{\rm hybrid}$ where an RBF with manually-selected $\Sigma = 0.4^2 I$ was used for the spatial variables $x_1,x_2$ and a polynomial kernel of first degree for the variables $\{L_1 u,\dots,L_6u\}$ where the shift was chosen to be $c = \operatorname{mean}(S_{2:})$ and the scaling matrix $B = \left(\operatorname{diag}(\operatorname{cov}(S_{2:}))\right)^{-1}$ where 
\begin{equation*}
S = \begin{bmatrix}
    \Phi(\whu^1_2, Y^1) \\
    \vdots \\
    \Phi(\whu^M_2, Y^M)
\end{bmatrix} \in \R^{\sum_{m=1}^M |Y^m| \times 8}.    
\end{equation*}
where $\whu^m_2$ is the learned solution obtained from the first step of the 2-step KEqL method, using the hand-tuned hyperparameters mentioned above.

\item To scale up equation learning when using many $M$ function tuples, we employed a fixed basis with 
$I = 200$ inducing points for learning the equation (see \Cref{app:algorithms:nystromapproximation}).

\item For the loss function \eqref{transformed_opt}, we followed the definitions in \eqref{eq: hyperparams} except we took $\sigma_P^2 = \theta_P I$ and set $\lambda = M$, $\theta_u = 1e-13$ and $\theta_P = 5e-13$ as hyperparameters, selected manually following the heuristic described in \Cref{app:numerics:ODE-example}.

\end{itemize}
\editend

Additionally, we leveraged the block arrowhead matrix structure within the optimization algorithm (see \Cref{app:algorithms:blockmatrixstructure}).

To further illustrate the equation and operator learning errors in relation to the error plot in \Cref{fig:darcy_main}, we show samples of the contour error plots in \Cref{fig: darcy_ID_OOD_contours} and the convergence history in  \Cref{fig: darcy_loss_hist} that uses the block arrowhead matrix structure LM-type algorithm for the case where the equation is learned using $M = 8$ training tuples and $N_\mY = 8$ observation points of the solution.  
We recall that the variational solution proposed in \eqref{eq:operator-learning} was used to solve the recovered PDE for both the 1-step and 2-step KEqL methods, employing the same kernel $\UK$, to report the operator learning results.

In \Cref{fig: darcy_ID_OOD_contours}(A), we show some ID and OOD solution samples used to test the learned equation for both the 1-step and 2-step KEqL methods. In \Cref{fig: darcy_ID_OOD_contours}(B), we present the equation learning errors, where the values typically differ by one order of magnitude between 1-step and 2-step KEqL in favor of the joint methodology, with larger errors appearing near the boundaries. Similarly, in \Cref{fig: darcy_ID_OOD_contours}(C), we observe that the 1-step method consistently outperforms the 2-step method by an order of magnitude. Notably, for this case the equation learning errors were an order of magnitude smaller than the operator learning errors.

\begin{figure}[htp]
    \centering
    \vspace{6ex}
    \begingroup
    \setlength{\tabcolsep}{4pt}
    % \begin{tabular}{l!{\color{gray}\vrule}cc!{\color{gray}\vrule}cc}
    \begin{tabular}{lcccc}
        \begin{overpic}[scale=0.17]{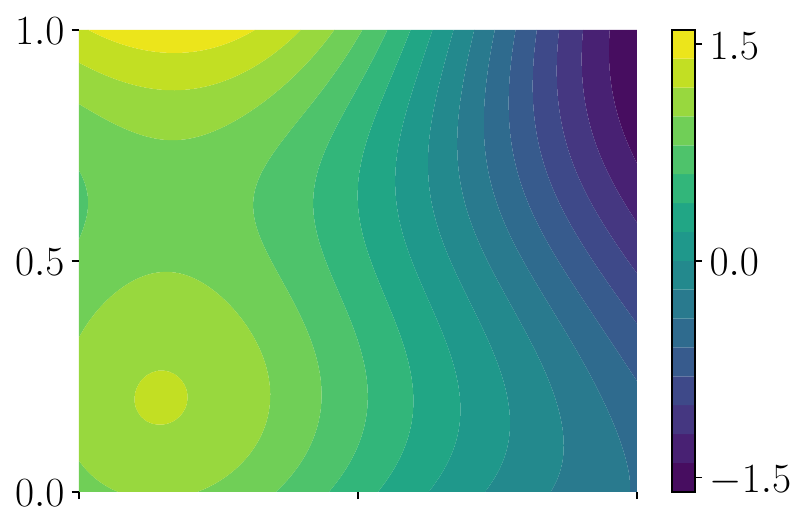}
        \put(-5,30){\rotatebox{90}{{\scriptsize $x_2$}}}
        \put(43,65){{\footnotesize $u$}}
        \put(20,85){{\small (A) Solutions}}
        \end{overpic}
        &
        \hspace{-0.8ex}\begin{overpic}[scale=0.17]{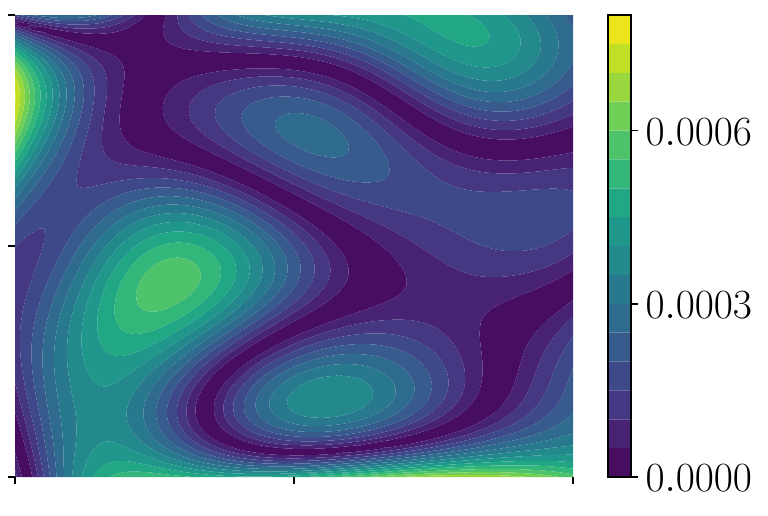}
        \put(25,68){{\footnotesize 1-step}}
        \put(55,90){{\small (B) $\frac{|\P(u) - \widehat\P(u)|}{\|\P(u)(Y_{\rm test})\|_2}$}}
        \end{overpic}
        
        &
        \hspace{-1ex}\begin{overpic}[scale=0.17]{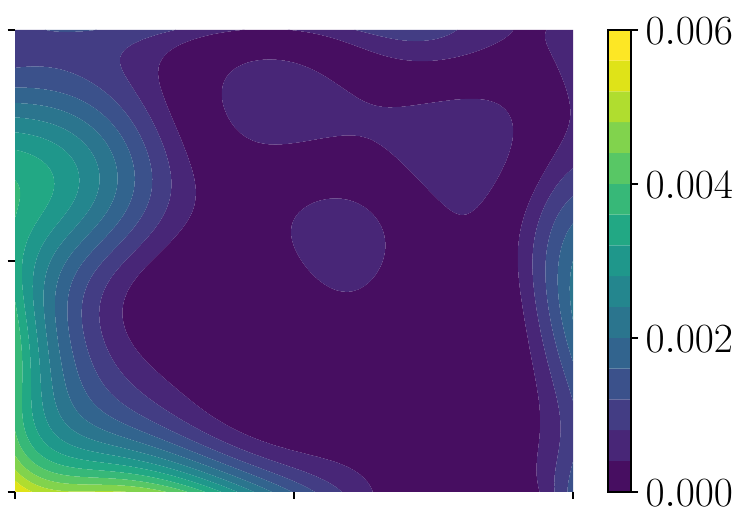}
        \put(25,69){{\footnotesize 2-step}}
        \end{overpic}
        &
        
        \hspace{-0.8ex}\begin{overpic}[scale=0.17]{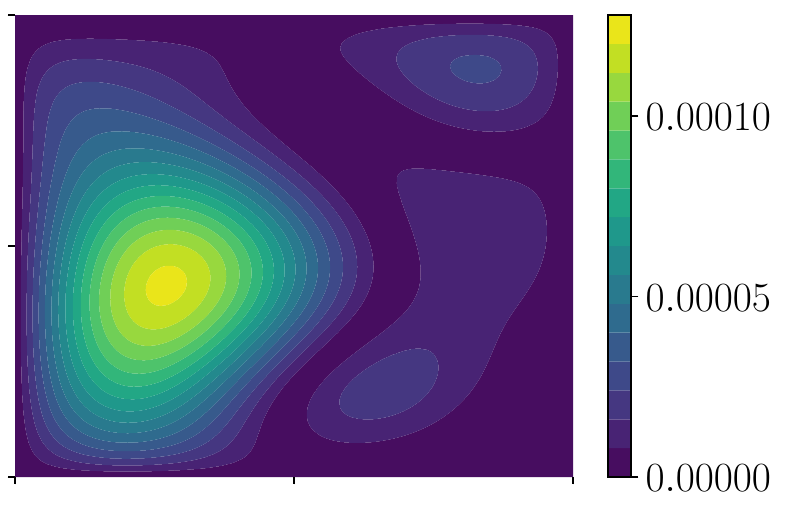}
        \put(24,66){{\footnotesize 1-step}}
        \put(45,88){{\small (C) $\frac{|u - \widehat{\P}^\dagger \P(u)|}{\|u(Y_{\rm test})\|_2}$}}
        \end{overpic}
        
        &
        \hspace{-0.8ex}\begin{overpic}[scale=0.17]{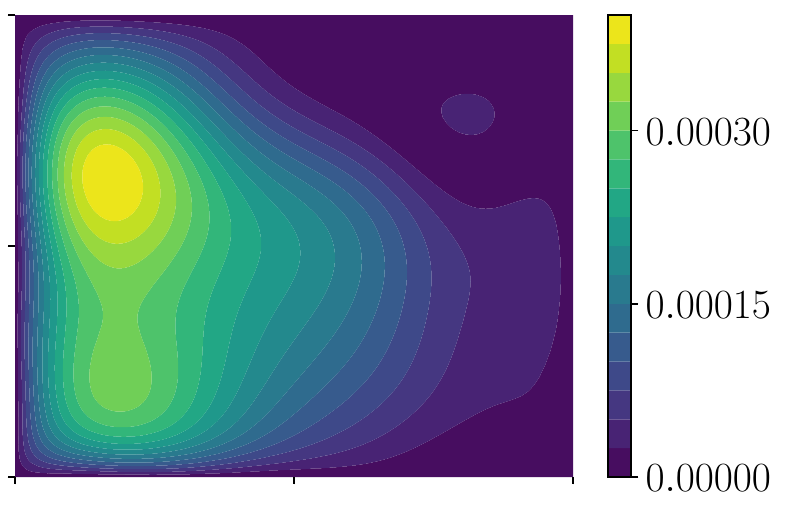}
        \put(25,68){{\footnotesize 2-step}}
        \end{overpic}
        \\
        \begin{overpic}[scale=0.17]{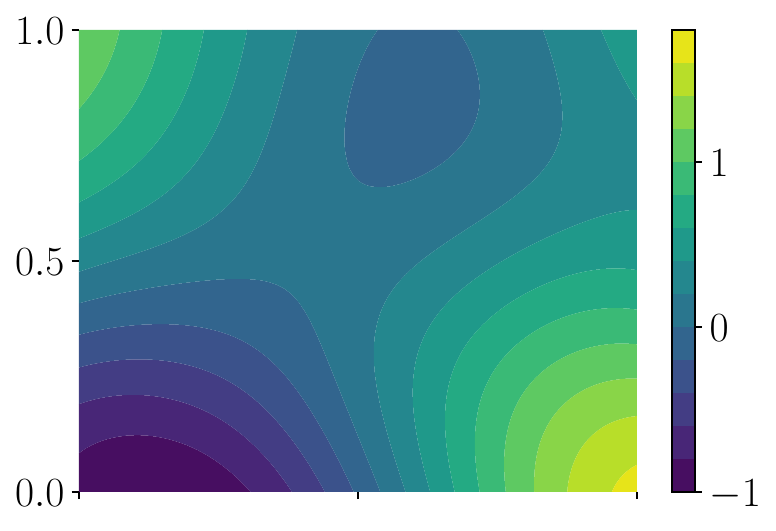}
        \put(-5,31){\rotatebox{90}{{\scriptsize $x_2$}}}
        \end{overpic}
        &
        
        \includegraphics[scale=0.17]{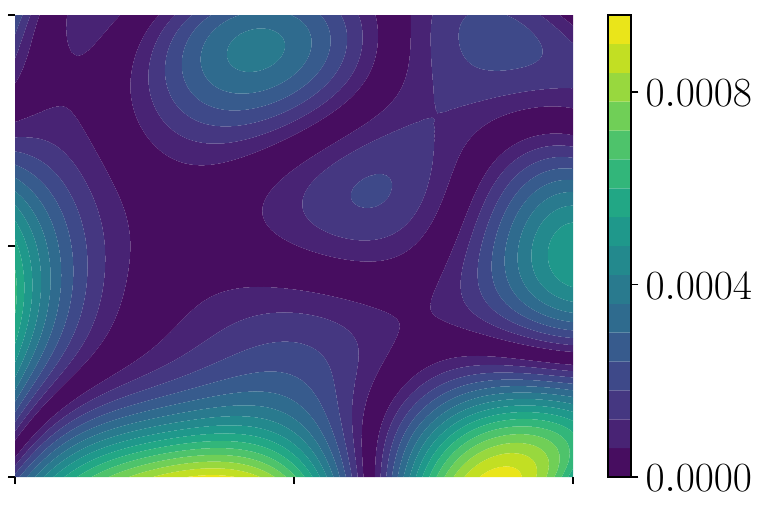}
        
        &
        \includegraphics[scale=0.17]{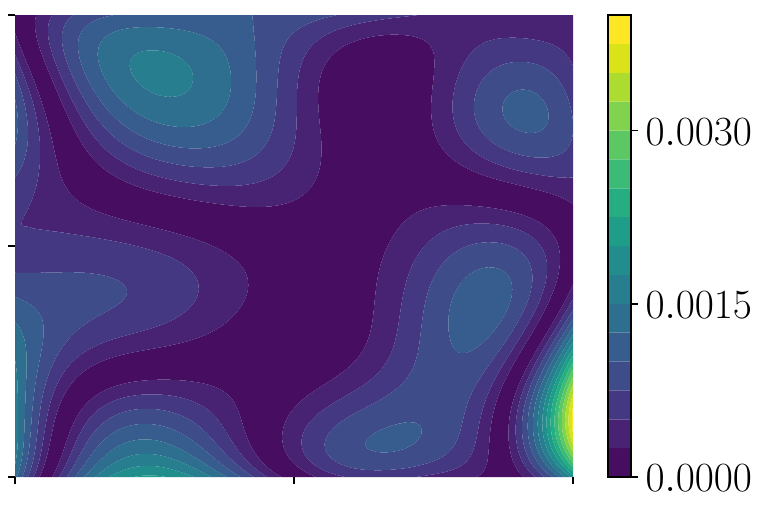}
        &
        
        \hspace{-3ex}\includegraphics[scale=0.17]{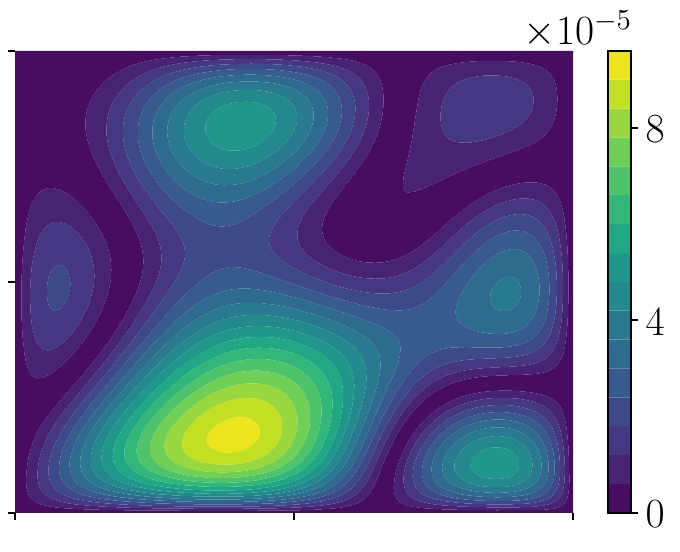}
        
        &
        \includegraphics[scale=0.17]{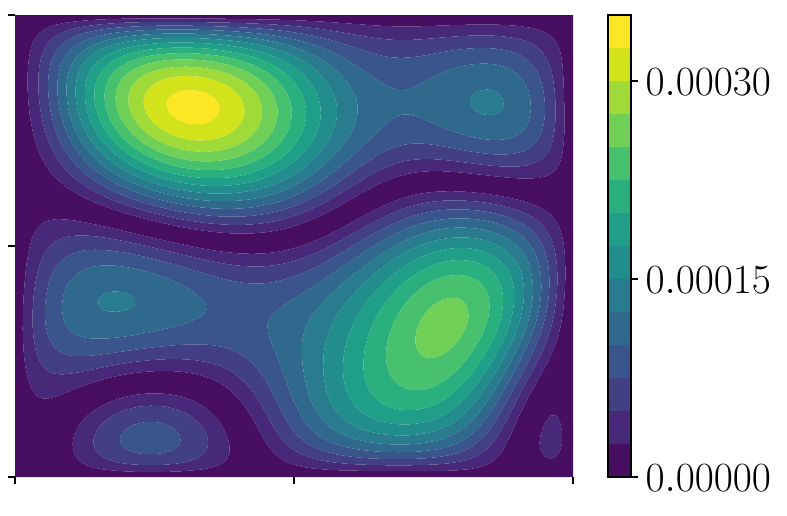}
        \\
        \begin{overpic}[scale=0.17]{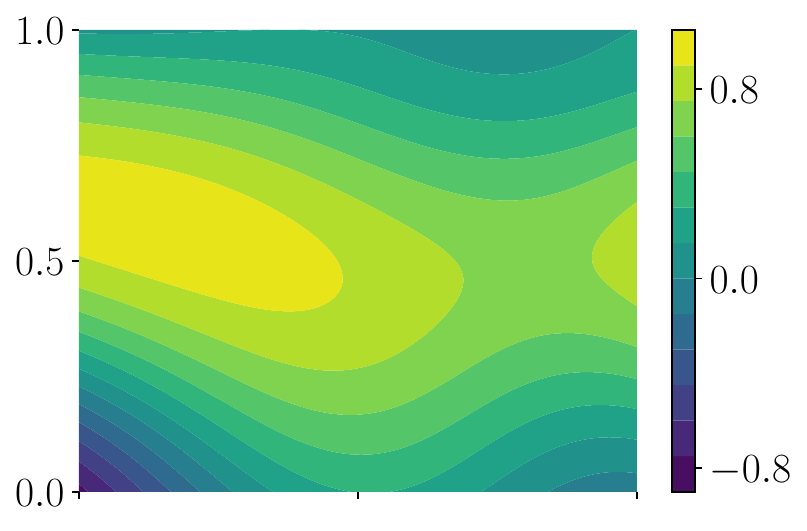}
        \put(-5,30){\rotatebox{90}{{\scriptsize $x_2$}}}
        \put(-20,68){\rotatebox{90}{{\small In-distribution (ID)}}}
        \end{overpic}
        &
        
        \includegraphics[scale=0.17]{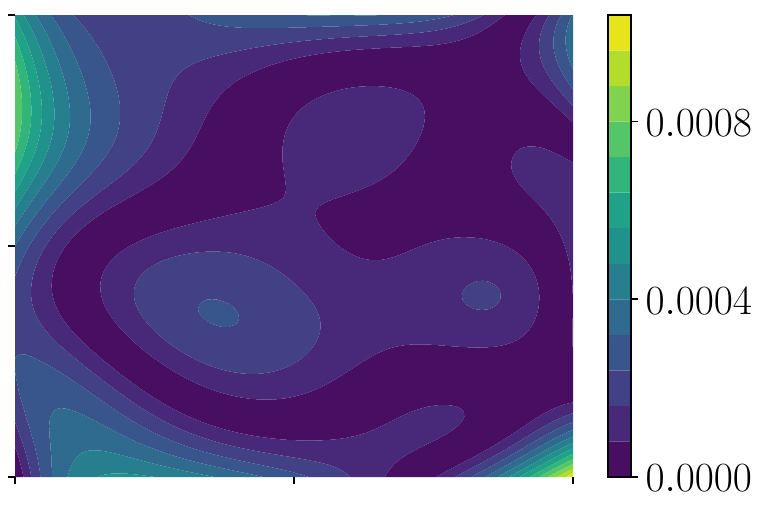}
        
        &
        \includegraphics[scale=0.17]{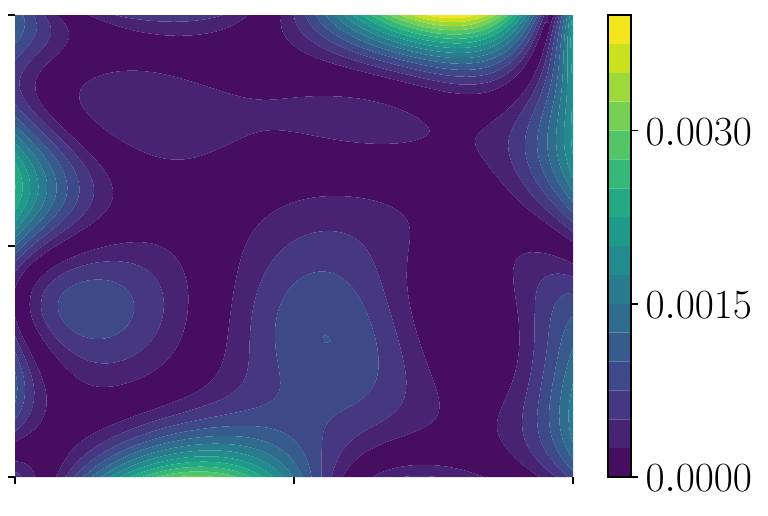}
        &
        
        \hspace{-3ex}\includegraphics[scale=0.17]{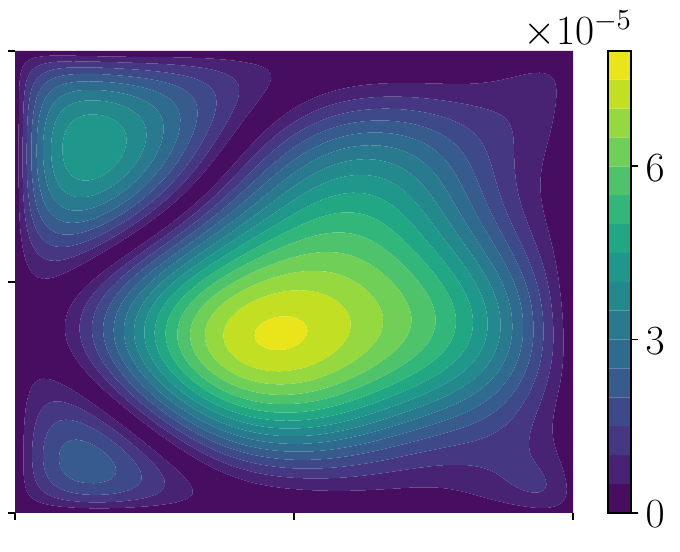}
        
        &
        \includegraphics[scale=0.17]{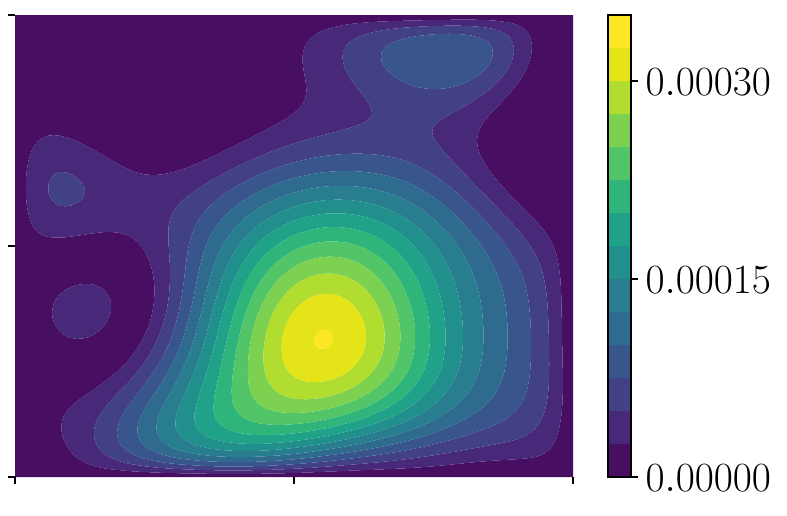}
        \\
        % \arrayrulecolor{gray}\hline
        % \rule{-3pt}{16ex}
        \begin{overpic}[scale=0.17]{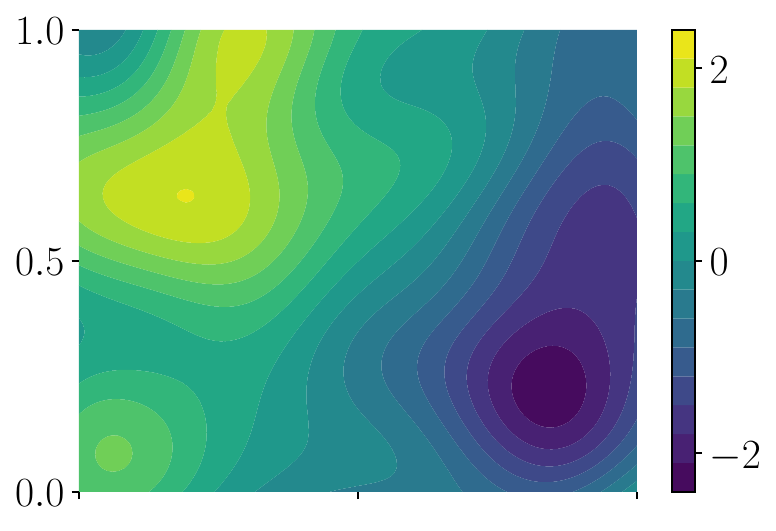}
        \put(-5,31){\rotatebox{90}{{\scriptsize $x_2$}}}
        \end{overpic}
        &
        \includegraphics[scale=0.17]{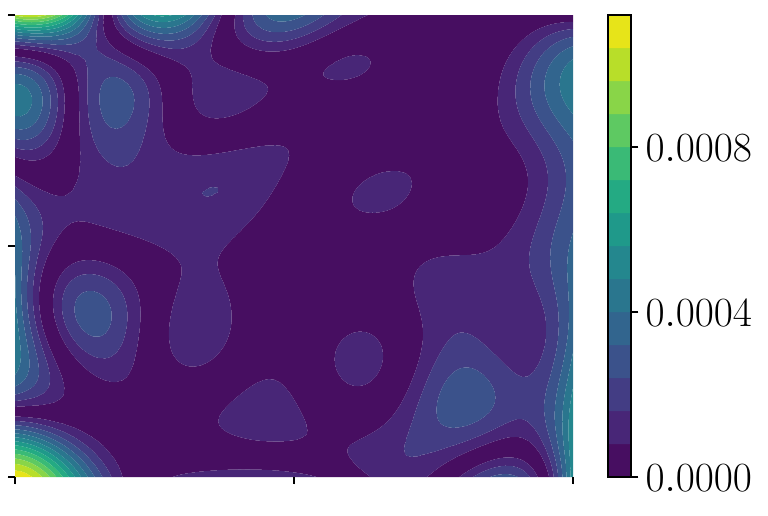}
        &
        \includegraphics[scale=0.17]{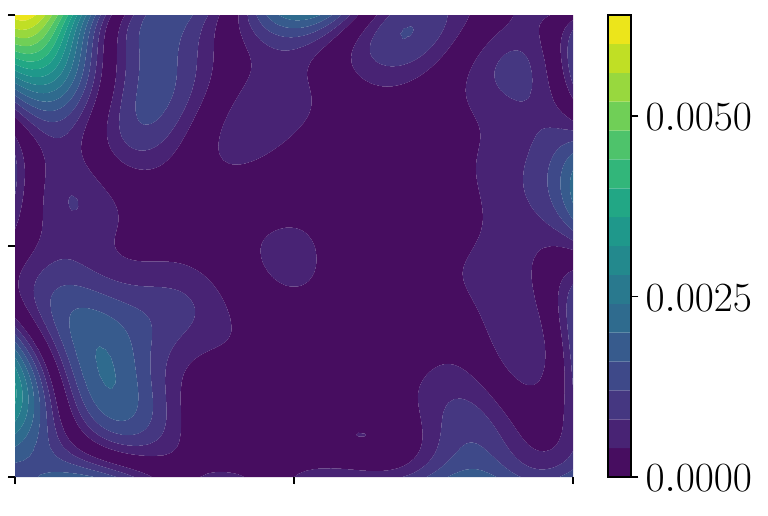}
        &
        \includegraphics[scale=0.17]{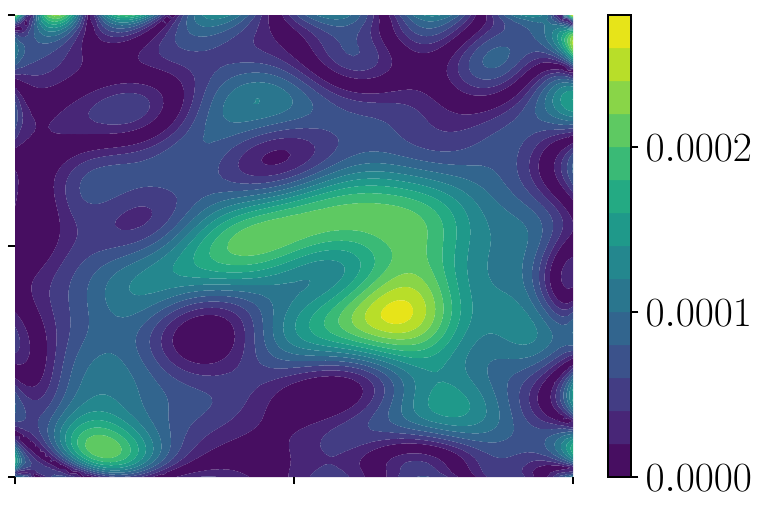}
        &
        \includegraphics[scale=0.17]{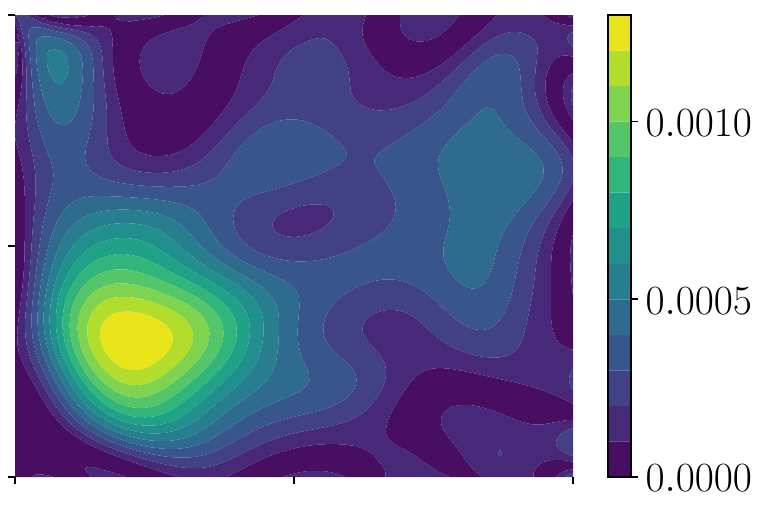}
        \\
        \begin{overpic}[scale=0.17]{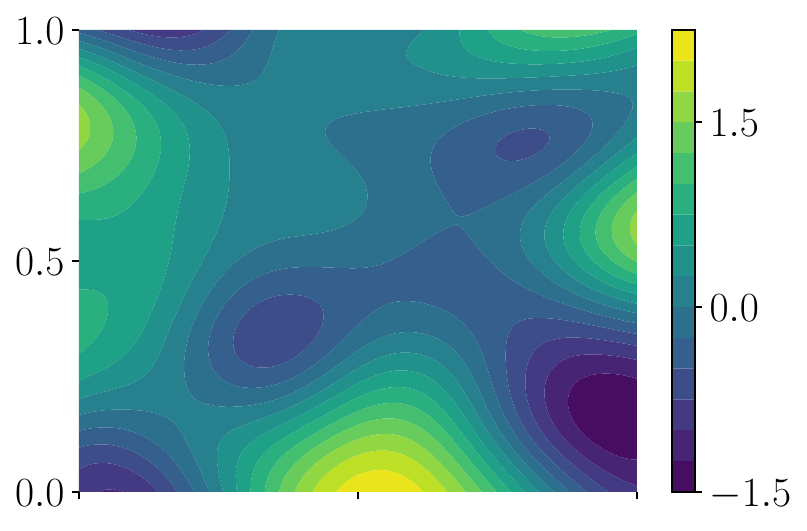}
        \put(-5,30){\rotatebox{90}{{\scriptsize $x_2$}}}        \end{overpic}
        &
        \includegraphics[scale=0.17]{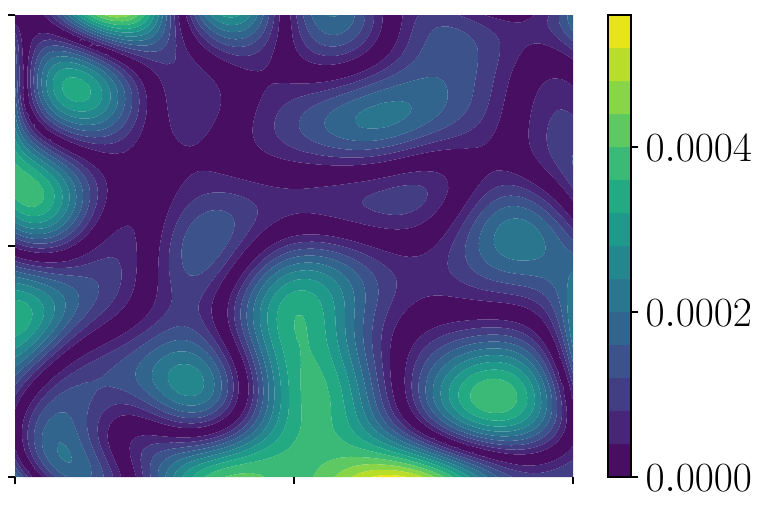}
        &
        \includegraphics[scale=0.17]{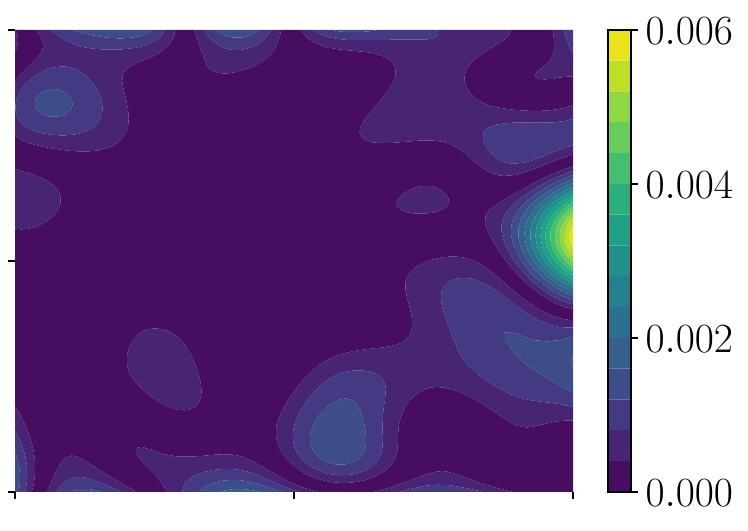}
        &
        \includegraphics[scale=0.17]{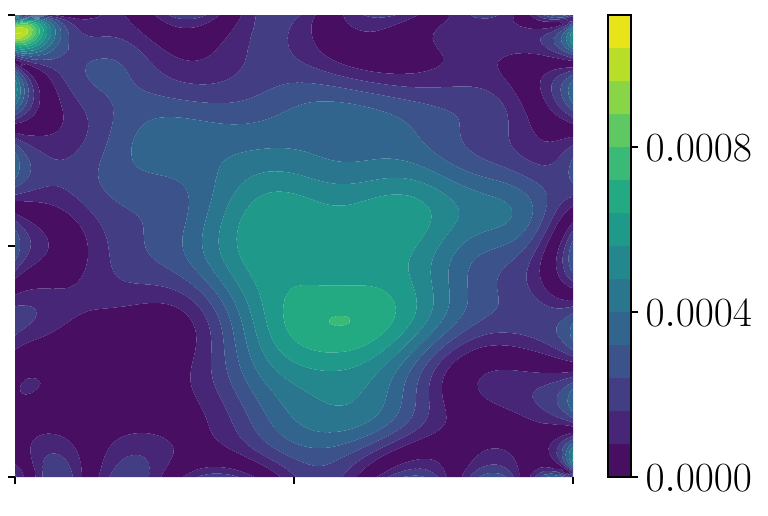}
        &
        \includegraphics[scale=0.17]{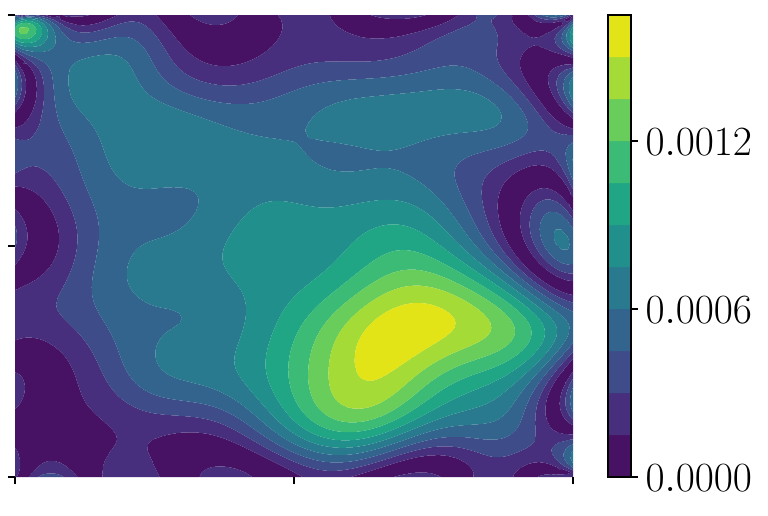}
        \\
        \begin{overpic}[scale=0.17]{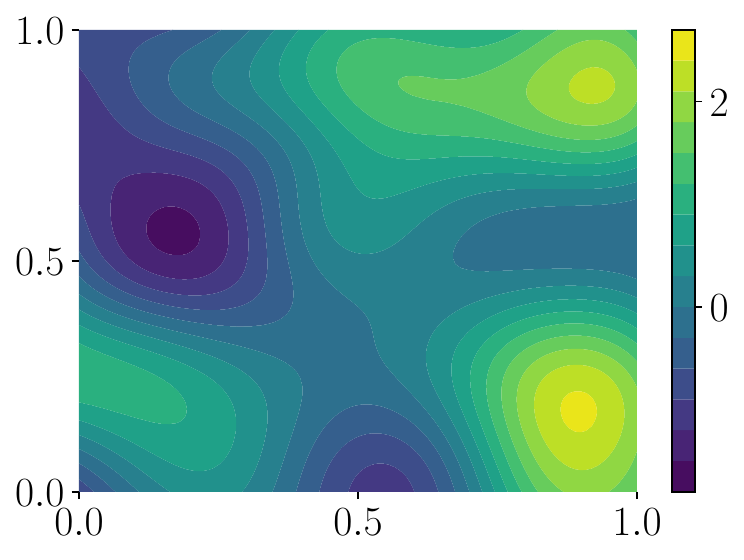}
        \put(-5,36){\rotatebox{90}{{\scriptsize $x_2$}}}
        \put(45,-3){{\scriptsize $x_1$}}
        \put(-20,65){\rotatebox{90}{{\small Out-of-distribution (OOD)}}}
        \end{overpic}
        &
        \hspace{-1.4ex}\begin{overpic}[scale=0.17]{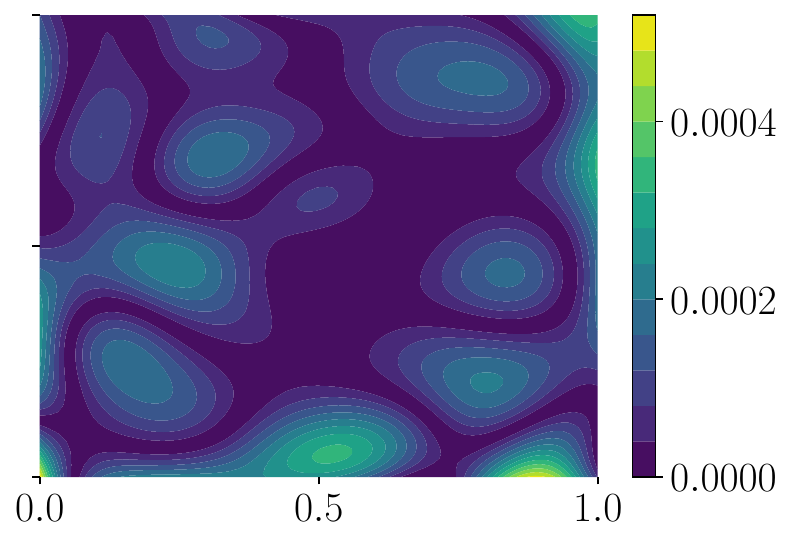}
        \put(37,-3){{\scriptsize $x_1$}}
        \end{overpic}
        &
        \hspace{-1.4ex}\begin{overpic}[scale=0.17]{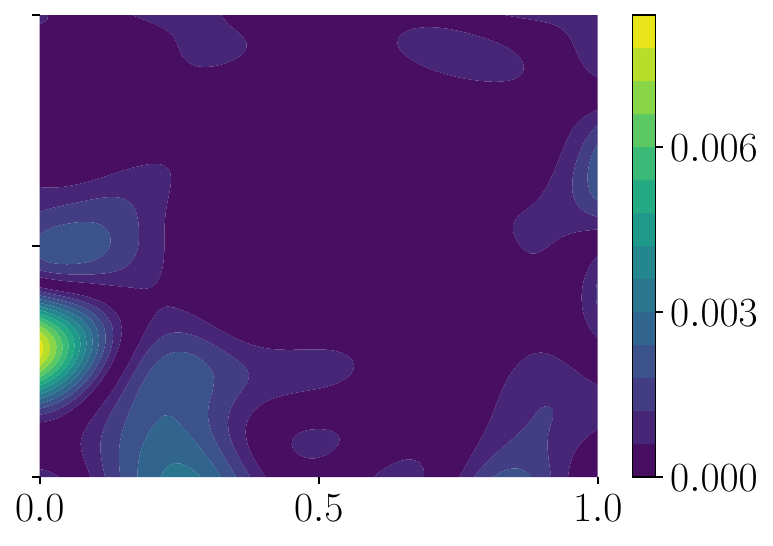}
        \put(38,-3){{\scriptsize $x_1$}}
        \end{overpic}
        &
        \hspace{-1.4ex}\begin{overpic}[scale=0.17]{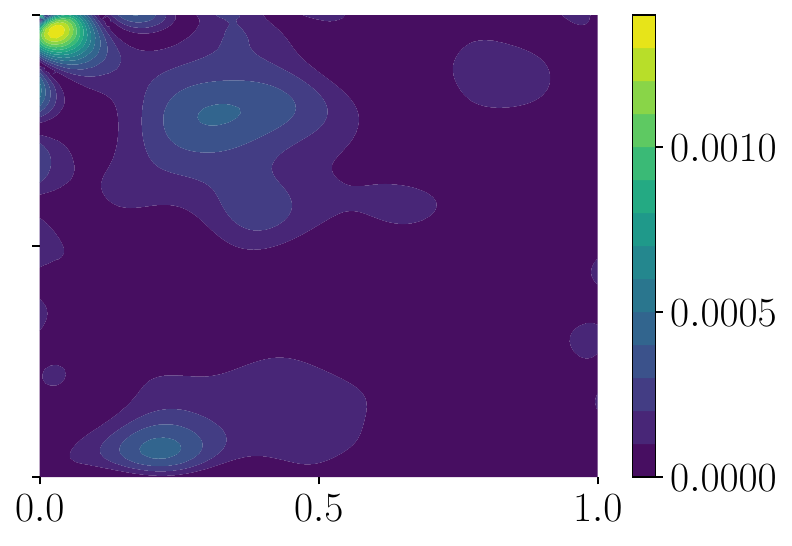}
        \put(37,-3){{\scriptsize $x_1$}}
        \end{overpic}
        &
        \hspace{-1.4ex}\begin{overpic}[scale=0.17]{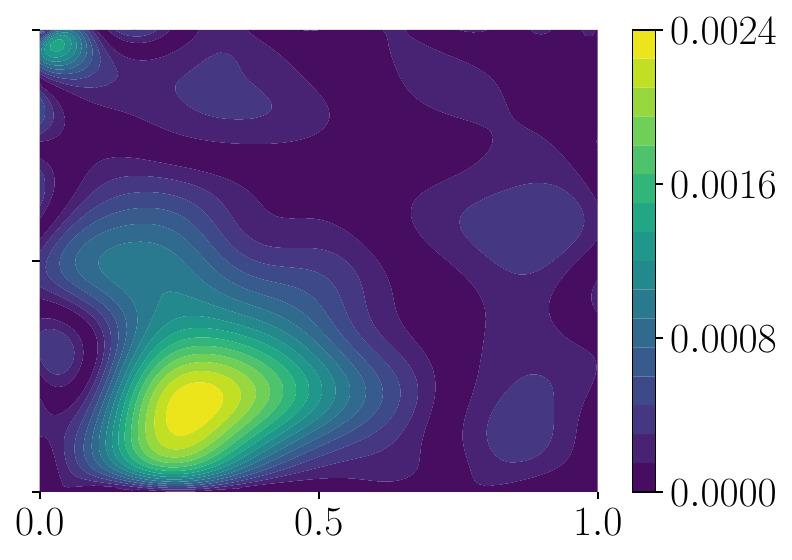}
        \put(37,-3){{\scriptsize $x_1$}}
        \end{overpic}
    \end{tabular}

    \endgroup
    \caption{Samples of test solutions $u$ (ID and OOD) and their contours of the equation and operator learning errors for 1-step and 2-step KEqL methods for Darcy's flow 
PDE \eqref{darcy_PDE}: (A) Sample solutions $u$ ID and OOD for the PDE ;(B) Contour plots for equation learning errors; (C) Contour plots for operator learning errors.}
    \label{fig: darcy_ID_OOD_contours}
\end{figure}

\begin{figure}[htp]
    \centering
    \begingroup
    \footnotesize
    \begin{overpic}[scale=0.4]{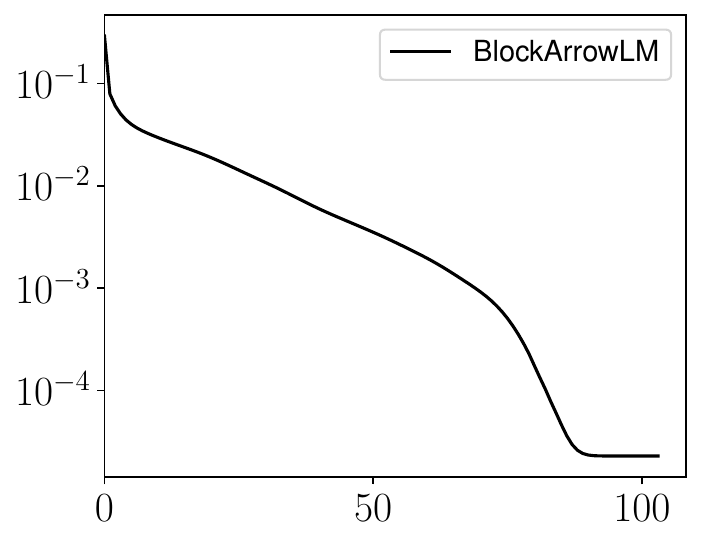}
    \put(-6,36){\rotatebox{90}{{ Loss}}}
    \put(50,-5){\rm Steps}
    \end{overpic}
    \endgroup
    \caption{Convergence history of the block arrowhead matrix structure LM-type algorithm for the reduced 1-step KEqL when using $M = 8$ solution tuples at $N_\mY =8$ observation points in the interior for Darcy's flow PDE \eqref{darcy_PDE}.}
    \label{fig: darcy_loss_hist}
\end{figure}

\subsubsection{An ablation study for Darcy flow PDE  with  scarce-data} 
\label{subsubsec: darcy-ablation-study}

\editstart
In this section we study the effect of the choice of the differential operators $L_q$
on the quality of the learned equation and the solution operator for the Darcy flow PDE.
We adopt the same kernels and hyperparameters  as in \Cref{sec:numerics:darcy-SI}, fixing $M=32$ sampled functions and $N_{\mathcal{Y}}=2$ observed interior points per function, but now we consider three different dictionaries of differential operators. The 
Exact dictionary $\{ u, \partial_{y_1}u,\partial_{y_2}u,\partial_{y_1}^2u,\partial_{y_2}^2 u\}$
that matches the partial derivatives that appear in the true operator;
a Reduced dictionary $\{u,\partial_{y_1}u,\partial_{y_2}u\}$ where the second order 
partial derivatives are omitted; and an Expanded dictionary $\{u,\partial_{y_1}u,\partial_{y_2}u,\partial_{y_1}^2u,\partial_{y_2}^2u,\partial_{y_1y_2}u\}$
where an additional cross derivative is included. We fit the model with these 
features and report the equation and operator learning errors in \Cref{tab: darcy_errs_table-SM}
for both 1-step and 2-step KEqL.

% test three variants of the degree-1 polynomial feature set used to build $\hat P$, an exact (E) set matching the true operator's structure, \; an adding (A) set that further includes the zeroth-order term and the mixed partial, ; and a removing (R) set that drops both second derivatives, . Results are reported in \Cref{fig: darcy_errs_table}, for the equation-learning error evaluated both in- and out-of-distribution, and for the operator-learning error obtained by solving a new Darcy boundary value problem with the inferred operator.
Across both methods, and for all metrics we observe a consistent ranking:
the exact dictionary is superior followed by the expanded dictionary and then 
the reduced dictionary. Indeed, the expanded dictionary obtains comparable errors 
to the exact choice. However, the reduced dictionary is worse by one or two orders 
of magnitude. 
% E outperforms A, which in turn outperforms R. The gap between E and A is modest, at most a factor of about $2.5$ for the 1-step method and smaller still for the 2-step method, reflecting the fact that the two extra features in A ($u$ and $\partial_{y_1y_2}u$) do not correspond to any true term in the Darcy operator, so their inclusion adds only a mild identifiability cost rather than a qualitative failure. The gap between A and R, by contrast, is severe, removing both second-derivative features degrades every error substantially, by roughly one order of magnitude for the 2-step method and by two to three orders of magnitude for the 1-step method, since $\hat P$ is then structurally unable to represent the diffusive term $\mathrm{div}(a\nabla u)$ at all, regardless of how much data is available.
This is qualitatively different from our experiments in \Cref{subsubsec: exp_shock_ablation}, where the PINN-SR collapsed to essentially the same (incorrect) equation whether or not the diffusion feature was supplied. For Darcy, both the 1-step and 2-step appear to give 
good solutions as long as the dictionary of differential operators  is sufficiently
expressive.

% methods respond directly and monotonically to feature completeness, with neither showing that same insensitivity. As before, the state interpolant used to seed (and, for the 2-step method, to fully determine) the state estimate is built independently of the feature set used for $\hat P$, so the sensitivity to E/A/R is attributable entirely to the dynamics-inference step rather than to any change in how the state itself is reconstructed.

\begin{table}
    \centering
\editstart
\begin{tabular}{cc|cc}
     & & \hspace{1cm}Method &  \\
    \hline
     & Feature set & 1-step & 2-step \\
    \hline
    \multirow{3}{*}{$\mathcal{R}_{\text{eql}}(\rm ID)$} 
     & Exact & \num{5.4e-3} & \num{9.0e-2} \\
     & {Reduced} & \num{1.7e0} & \num{8.1e-1} \\
     & {Expanded} & \num{1.4e-2} & \num{1.2e-1} \\
    \hline
    \multirow{3}{*}{$\mathcal{R}_{\text{eql}}(\rm OOD)$}
     & Exact & \num{3.0e-3} & \num{5.1e-2} \\
     & {Reduced} & \num{8.4e-1} & \num{9.1e-1} \\
     & {Expanded} & \num{7.0e-3} & \num{7.1e-2} \\
    \hline
    \multirow{3}{*}{$\mathcal{R}_{\text{opl}}(\rm ID)$}
     & Exact & \num{5.9e-4} & \num{3.4e-2} \\
     & {Reduced} & \num{6.2e-1} & \num{9.7e-1} \\
     & {Expanded} & \num{1.3e-3} & \num{3.8e-2} \\
    \hline
\end{tabular}
\editend
\caption{
    \editstart
    Relative equation  and operator learning errors for the Darcy flow PDE \eqref{darcy_PDE}, with $M=32$ training functions and $N_{\mathcal{Y}}=2$ observed interior points. Each row  reports the 1-step and 2-step errors under three feature sets: 
    the exact dictionary of operators 
    $\{u, \partial_{y_1}u,\partial_{y_2}u,\partial_{y_1}^2u,\partial_{y_2}^2u\}$ matching the true operator; the reduced dictionary $\{u,\partial_{y_1}u,\partial_{y_2}u\}$
    which misses both second derivatives; 
    and an expanded dictionary $\{u,\partial_{y_1}u,\partial_{y_2}u,\partial_{y_1}^2u,\partial_{y_2}^2u,\partial_{y_1y_2}u\}$ with two extra features.
    \editend
    }
    \label{tab: darcy_errs_table-SM}
\end{table}

\editend

\section{Theoretical details}\label{app:theory-details}
\editstart
In this section we outline additional theoretical details.
We review preliminaries in \Cref{sec:SM-theory-prelims} including 
RKHSs and Sobolev spaces followed by the proof of representer theorems for 2-step and 1-step KEqL 
in \Cref{sm: representer theorems}. We present an error bound for 2-step KEqL in 
\Cref{app:theory:errorbound-2-step}. A convergence result for 1-step KEqL, under milder 
conditions than those in \Cref{thm:1-step-error-estimate}, is presented in
\Cref{app:theory:convergence-proofs}. Finally, we 
present the proof of \Cref{prop:operator-learning-1-step-KEqL} in \Cref{sm:operator-learning-analysis}.
\editend

\subsection{Preliminaries}\label{sec:SM-theory-prelims}

\subsubsection{RKHS spaces and Mercer's theorem}\label{sec:SM-RKHS-review}

\editstart

Consider an open set $\Omega \subset \R^q$ and 
a kernel $\HK: \Omega \times \Omega \to \R$. We say $\HK$ is positive 
definite and symmetric (PDS) if 
 for any $N \in \mathbb{N}$ and set of points 
 $X = \{x_1, \dots, x_N\} \subset \Omega$, the kernel 
 matrix $\HK(X, X) \in \R^{N \times N}$
 with $(i,j)$-entries $\HK(x_i,x_j)$ is PDS. If $\HK(X, X)$ is 
 strictly PDS  then we say $\HK$ is a strictly PDS kernel. 
Every PDS kernel $\HK$ is uniquely identified with a Hilbert space $\mH$,
called its corresponding RKHS, with inner product 
$\langle \cdot, \cdot \rangle_{\mH}$ and norm $\| \cdot \|_\mH$.
The kernel $\HK$ and elements of 
$\mH$ satisfy the so-called reproducing property, i.e., 
$\langle f, \HK(x, \cdot) \rangle_\mH = f(x)$ for all $f \in \mH$.
We are particularly interested in the characterization of $\mH$ via 
Mercer's theorem \cite[Thm.~4.49]{steinwart2008support}:

\begin{proposition}\label{prop:Mercers-thm}
    Suppose $\Omega \subset \R^q$ is bounded and let $\HK:\Omega \times \Omega \to \R$ be a PDS kernel that is continuous in both of its arguments on $\overline{\Omega}$. 
    Then there exists an orthonormal set of continuous eigenfunctions
    $\{ e_i \}_{i=1}^\infty \subset L^2(\Omega)$ and 
    decreasing eigenvalues
     $\{\lambda_i \}_{i=1}^\infty$, $\lambda_1 \ge \lambda_2 \ge \dots$, 
    such that 
       $ \HK(x, x') = \sum_{i=1}^\infty \lambda_i e_i(x) e_i(x')$.
\end{proposition}

\begin{lemma}\label{lem:H-gamma-bound-SM}
    Suppose $f\in \mH^{2\gamma}$ and $f' \in \mH^\gamma$. Then 
        $\langle f', f \rangle_{\mH^\gamma} \le 
        \| f'\|_{L^2(\Omega)} \| f\|_{\mH^{2\gamma}}$.
\end{lemma}

\begin{proof}
For simplicity of notation let us assume that $\lambda_i \neq 0$ so that 
the kernel $\HK$ is strictly PDS (i.e., non-degenerate).
By definition of the $\mH^\gamma$ inner product, and using 
Cauchy-Schwartz, we have that
$\langle f', f \rangle_{\mH^\gamma} 
= \sum_{i=1}^\infty \lambda_i^{-\gamma} c_i(f') c_i(f)  
\le  \left(  \sum_{i = 1}^\infty c_i(f')^2 \right)^{1/2} 
\left(  \sum_{i = 1}^\infty \lambda_i^{-2\gamma} c_i(f)^2 \right)^{1/2} 
= \| f'\|_{L^2(\Omega)} \| f\|_{\mH^{2\gamma}}.$  
\end{proof}

\editend

\subsubsection{The Sobolev embedding theorem}\label{sec:SM-Sobolev-review}

Here we briefly recall the Sobolev embedding theorem 
that is central to our error analysis. We will not present a 
review of Sobolev spaces and instead refer the reader to 
the classic texts \cite{adams2003sobolev, evans2022partial}. 

For an open set $\Omega \subseteq \R^q$  and $\gamma \in \mathbb{N}$ we write 
$H^\gamma(\Omega)$ to denote the $L^2(\Omega)$ based Sobolev space of index $\gamma$
equipped with the norm 
$    \| u\|_{H^\gamma(\Omega)}^2 = \sum_{| \mbf{a} | \le \gamma} 
    \| D^\mbf{a} u \|^2_{L^2(\Omega)},$
where $\mbf{a}$ is a multi-index. We also employ the convention $H^0(\Omega) = L^2(\Omega)$.
We then have the following Sobolev embedding theorem:
\begin{proposition}[Sobolev embedding theorem {\cite[Thm.~4.12]{adams2003sobolev}}]\label{prop:sobolev-embedding}
Suppose $\Omega \subset \R^q$ is a bounded set with Lipschitz boundary and that for 
$p \in \mathbb{N}$ it holds that $\gamma > q/2 + p$. Then $H^\gamma(\Omega)$ 
is continuously embedded in $C^p(\Omega)$ and it holds that 
$\| u \|_{C^p(\Omega)} \le C_\Omega \| u \|_{H^\gamma(\Omega)}$ 
for an embedding constant $C_\Omega \ge 0$ that depends only on $\Omega$.
\end{proposition}

\subsection{Representer theorems for KEqL}\label{sm: representer theorems}
Below we give the proofs of representer theorems for 
both versions of the KEqL.

\subsubsection{Representer theorem for 2-step KEqL}\label{app:theory:representer-theorem-2-step-method}

We will  give a brief  justification for the expressions 
\eqref{u-kernel-interpolant} and \eqref{P-kernel-interpolant}
as direct consequences of \Cref{prop:generalized-rep-theorem}: Consider \eqref{u-optimal-recovery} and apply \Cref{prop:generalized-rep-theorem} with $\mH = \UU$, $\phi_i = \delta_{y_i^m}$, 
and $z_i = u^m(y_i^m)$ for all $i = 1,\dots, N$. This gives 
\eqref{u-kernel-interpolant}. Next consider 
 \eqref{P-optimal-recovery} and apply \Cref{prop:generalized-rep-theorem} with $\mH = \PP$, $\phi_i = \delta_{s_i}$ for $s_i \in S$, and $z_i = f(y_i)$ for $y_i \in Y$ to obtain 
 \eqref{P-kernel-interpolant}.

\subsubsection{Representer theorem for 1-step KEqL}\label{app:theory:representer-theorem-1-step-method}
\editstart
Here, we prove \Cref{thm:one-step-equivalence}).
\editend
To simplify notation, we first prove the theorem with $M=1$, i.e., 
with training data  $(u(Y^1),f)$, and $\overline{P}=0$. At the end 
of the proof we comment on how the argument can be extended to the 
general case. Let us recall the 
corresponding optimal recovery problem for convenience: 
\begin{equation}\label{one-shot-optimal-recover-simplified}
\begin{aligned}
     (\wh{u}, \whP) = & \argmin_{ v \in \UU, G \in \PP} && 
 \| G \|^2_\PP + \lambda_1  \| v\|^2_\UU \\ 
& \text{s.t.} &&   v(Y^1) = u(Y^1), \quad \text{and} \quad  G(S) = f(Y),\\
& \text{where} && S = \Phi(v,Y). 
\end{aligned}
\end{equation}
Recall that $Y$ is the dense set of collocation points 
that we use to impose the PDE constraint.

Then our goal is to show that under the assumptions of 
\Cref{thm:one-step-equivalence},  every minimizing  tuple
        $(\whu, \whP)$ of 
     \eqref{one-shot-optimal-recover-simplified} 
     can be written in the form
    \begin{equation*}
        \whu(y) = \UK(\phi, y)^T\wh{\alpha}, 
        \quad 
        \whP(s) =   \PK(S(\widehat{\alpha}), s)^T \wh{\beta},
    \end{equation*}
    for a tuple  $(\wh{\alpha} , \wh{\beta})$ that solves 
    the equivalent optimization problem 
    \begin{equation}\label{one-step-representer-form-simplified}
\begin{aligned}
     (\wh{\alpha}, \wh{\beta}) = & 
     \argmin_{ \alpha \in \R^{QK}, \: \beta \in \R^{K}} && 
\beta^T \PK(S(\alpha), S(\alpha)) \beta 
+ \lambda_1 \alpha^T \UK(\phi,\phi) \alpha, \\ 
 % & \text{s.t.} && z^{m, q}_k = (\gamma^m)^T \UK(\phi, \phi^q_k), 
 % \quad m=1, \dots, M, \quad k=1, \dots, K, \quad q= 1, \dots, Q, \\
 & \text{s.t.} && \UK(\phi, Y^1)^T \alpha = u(Y^1),\quad 
\text{and} \quad \PK(S(\alpha), S(\alpha))^T \beta = f(Y),
 \\
 & \text{where}          && S(\alpha) = \Phi( \UK(\phi, \cdot)^T \alpha,Y).
\end{aligned}
\end{equation}
Recall our notation 
$\UK(\phi, Y^m) \in \R^{QK \times N}$ for the matrix  with 
columns $\UK(\phi, y_n)$, and $\PK(S(\alpha), S(\alpha)) \in \R^{K \times K}$
for the matrix with columns $\PK(S(\alpha), s_k)$ 
with $s_k= \Phi( \UK(\phi, \cdot)^T \alpha,y_k)$.
\begin{proof}
Let us write \eqref{one-shot-optimal-recover-simplified} in the equivalent form
\begin{equation}\label{one-shot-two-level-optimization-problem}
     \left\{
     \begin{aligned}
       & \argmin_{Z \in \R^{K\times Q}} \: \left\{
       \begin{aligned}
          & \argmin_{v \in \UU, G \in \PP} ~\|G\|_{\PP}^2 + \lambda_1\|v\|_{\UU}^2 \\
          &\text{s.t.} \quad  \phi_k^q(v) = Z_{k,q} \text{ and } G(s_k) = f(y_k),\\
          &\text{for all } k=1,\ldots,K, \text{ and }q=1,\ldots,Q
          \\
          &\text{where} \quad  s_k := (y_k,Z_{k,1},\dots,Z_{k,Q})
       \end{aligned} \right.\\
  &\text{s.t.} \quad Z_{j,1} = u(y_j), \hspace{7ex} \text{for } j \in \{k\in\{1,\dots,K\}: y_k \in Y^1 \}.
\end{aligned}
\right.
\end{equation}
Observe that the $Z$ variable matrix that we introduce acts as a slack variable for the inner problem and the outer constraint is enforcing the observation locations of $v$ at $Y^1 \subset Y$. 

For a fixed $Z \in \R^{K \times Q}$, we can solve the inner optimization problem for $v$ and $G$ explicitly using \Cref{prop:generalized-rep-theorem}, which leads to 
\begin{align*}
    v(y) = \UK(\phi,y)^T \UK(\phi,\phi)^{-1}\operatorname{vec}(Z),\quad \text{and}\quad G(s) = \PK(S,s)^T \PK(S,S)^{-1}f(Y).
\end{align*}
The RKHS norms of these solutions can be computed explicitly as
\begin{align*}
     \|v\|_{\UU}^2 = \operatorname{vec}(Z)^T\UK(\phi,\phi)^{-1} \operatorname{vec}(Z),\quad\text{and}\quad\|G\|_{\PP}^2 = f(Y)^T \PK(S,S)^{-1} f(Y). 
\end{align*}
Thus, we equivalently write \eqref{one-shot-two-level-optimization-problem} as
\begin{equation}\label{one-shot-two-level-optimization-problem_Z_slacked}
     \left.
     \begin{aligned}
       & \argmin_{Z \in \R^{K\times Q}} \: f(Y)^T \PK(S,S)^{-1} f(Y) + \lambda_1\operatorname{vec}(Z)^T\UK(\phi,\phi)^{-1} \operatorname{vec}(Z) \\
  &\text{s.t.} \quad \UK(\phi,Y^1)^T \UK(\phi,\phi)^{-1}\operatorname{vec}(Z) = u(Y^1),
\end{aligned}
\right.
\end{equation}
where we now modified the $s_k$ to be of the form
\begin{equation*}
s_k = (y_k, \UK(\phi,\phi_k^1)^T \UK(\phi,\phi)^{-1}\operatorname{vec}(Z),\dots, \UK(\phi,\phi_k^Q)^T \UK(\phi,\phi)^{-1}\operatorname{vec}(Z)).    
\end{equation*}
Similarly, the constraint in \eqref{one-shot-two-level-optimization-problem_Z_slacked} is valid since $\phi_k^q(v) = Z_{k,q}$ which particularly includes $\phi_j^1(v) = Z_{j,1}$ for $j \in \{\ell\in\{1,\dots,K\}: y_\ell \in Y^1 \}$.

Finally, if we define
$\alpha = \UK(\phi,\phi)^{-1}\operatorname{vec}(Z)$ and $\beta = \PK(S,S)^{-1}f(Y)$ (where in case $\UK(\phi,\phi)$ and $\PK(S,S)$ are not invertible, understand $\alpha$ and $\beta$ in the least squares sense) then the problem \eqref{one-shot-two-level-optimization-problem_Z_slacked} is equivalent to the desired system 
\begin{equation*}\label{one-shot-two-level-optimization-problem_alphabeta_slacked}
     \left.
     \begin{aligned}
       & \argmin_{\alpha \in \R^{QK}, \beta \in \R^K} \: \beta^T \PK(S,S)\beta + \lambda_1\alpha^T\UK(\phi,\phi) \alpha\\
  &\text{s.t.} \quad \UK(\phi,Y^1)^T \alpha = u(Y^1),\quad\text{and}\quad \PK(S,S)^T\beta = f(Y)
\end{aligned}
\right.
\end{equation*}
where $S = \{ s_1, \dots s_K\}$ with each $s_k = \left(y_k, \UK(\phi, \phi^1_k)^T \alpha , \dots, \UK(\phi, \phi^Q_k)^T \alpha \right)$. Notice that the constraint on $\beta$ is trivially satisfied since $G(s_k) = f(y_k)$ for all $k \in \{1,\dots,K\}$, however, it becomes explicit due to the change of variables in terms of $\beta$ in the formulation of $G$.

For the general proof for an arbitrary number of functions $M$ the same idea follows since the equivalent form \eqref{one-shot-two-level-optimization-problem} now contains the slack tensor $Z \in \R^{K\times Q \times M}$ instead of the matrix $Z$ above since 
each $v^m$ will be represented by a matrix $Z^m$ in our proof.
Additionally $S$ will then contain not only $K$ points  but $MK$ points corresponding to each of the $v^m$ and their 
requisite partial derivatives evaluated at the collocation points.
To obtain the result when $\overline{P} \neq 0$ the proof simply
 follows by replacing the data $f(Y)$ in 
 \eqref{one-shot-optimal-recover-simplified} with $\overline{P}(S)$
 and similarly in subsequent arguments in the proof.
{}
\end{proof}

\subsection{Going beyond Dirichlet data}\label{app:beyond-Dirichlet-data}
\editstart
We briefly outline how our 2-step and 1-step KEqL methods can be extended beyond the 
Dirichlet boundary condition settings of \Cref{sec:proposed-method} and \Cref{sec:algorithms}. 
We do not present numerical experiments concerning this setting and simply collect the 
mathematical formulation for interested readers. 

Let us suppose that $\B(u) = B \circ \Phi(u, \cdot)$ with a nonlinear map $B: \mS \to \R$. Note 
that we can always choose $\Phi$ to include all required $L_q$'s for the boundary conditions 
as well as the PDE operator $\P$ in the interior. With this setup we can impose  
boundary conditions via additional constraints or penalties in both 2-step and 1-step KEqL. 
The main difference here is that we should take care to distinguish our observation points $Y^m$,
the boundary collocation points $Y^{\partial \mY}$ and the interior collocation points 
$Y^\mY$. We further assume that $Y^m \subset Y^\mY$.

\subsubsection{2-step KEqL with general boundary data}
Since the learning of $\whu^m$ and $\whP$ is decoupled in 2-step KEqL, we will impose the 
boundary condition only in the first step where we estimate the $\whu^m$. That is, 
we replace \eqref{u-optimal-recovery} with 
\begin{equation*}%\label{u-optimal-recovery}
    \begin{aligned}
    \whu^m & :=  \argmin_{v^m \in \UU} \| v^m \|_\UU  \\ 
    \text{s.t.} &  \quad  v^m(Y^m) = u^m(Y^m) \quad \text{and} \quad  \B(v^m)(Y^{\partial \mY}) = g(Y^{\partial \mY}).
    \end{aligned}
\end{equation*}
and \eqref{u-optimal-recovery-noisy} with 
\begin{equation}\label{u-optimal-recovery-noisy-with-bc}
    \whu^m := \argmin_{v^m \in \UU} \| v^m \|^2_\UU 
    + \frac{1}{2 \sigma_u^2} \| v^m(Y^m) -  u^m(Y^m)\|_2^2 
    + \frac{1}{2 \sigma_{\partial\mY}^2} 
    \| \B(v^m)(Y^{\partial\mY}) - g(Y^{\partial\mY}) \|_2^2.
\end{equation}
We can identify the minimizers of both of these problems using our representer theorem 
\Cref{prop:generalized-rep-theorem}. 
In the case of \eqref{u-optimal-recovery-noisy-with-bc},
 following the same argument as in  \Cref{app:theory:representer-theorem-1-step-method},
 we can show that  $\whu^m  = \UK(\phi, \cdot)^T \wh\alpha^m$ where 
 \begin{equation*}
 \begin{aligned}
 \wh \alpha^m = &\argmin_{\alpha} \quad \alpha^T \UK(\phi, \phi) \alpha
     + \frac{1}{2 \sigma_u^2} \| \UK(\phi, Y^m)^T \alpha -  u^m(Y^m)\|_2^2  \\
   & \qquad + \frac{1}{2 \sigma_{\partial\mY}^2} 
    \| B\big(\Phi(  \UK(\phi, \cdot  )^T \alpha, Y^{\partial \mY}) \big) - g(Y^{\partial\mY}) \|_2^2.
  \end{aligned}
 \end{equation*}
If $B$ is nonlinear then we can solve this problem using our LM algorithm as before. If $B$ is linear 
we can absorb it into the definition of $\Phi$ to obtain a quadratic 
optimization problem for $\wh\alpha^m$. This is analogous to the GP-PDE solver 
of \cite{chen2021solving}.

\subsubsection{1-step KEqL with general boundary data}
Let us write $Y = Y^\mY \cup Y^{\partial \mY}$ and define the $\phi_k^q$ and the subsequent 
kernel vector fields and matrices as in \Cref{sec:algorithms:1-step}.
Then extending 1-step KEqL to arbitrary boundary conditions can be done analogously to the above. 
In the case of the relaxed formulation \eqref{one-step-representer-relaxed}, we obtain the problem 
\begin{equation*}
      \begin{aligned}
     (\wh{\pmb \alpha}, \wh{\beta}) = &
      \argmin_{{ \pmb \alpha},  \: \beta } \;   
      \beta^T \PK(S(\pmb \alpha), S(\pmb \alpha)) \beta 
+   \sum_{m=1}^M 
\frac{\lambda}{M} 
(\alpha^m)^T \UK(\phi,\phi) \alpha^m \\ 
& \qquad 
+    \frac{1}{2 \sigma_u^2} \| \UK(\phi, Y^m)^T \alpha^m - u^m(Y^m) \|_2^2 \\
& \qquad + 
\frac{1}{2 \sigma_{\partial\mY}^2} 
    \| B\big(\Phi(  \UK(\phi, \cdot  )^T \alpha^m, Y^{\partial \mY}) \big) - g(Y^{\partial\mY}) \|_2^2 \\ 
& \qquad + \frac{1}{2 \sigma_P^2} 
\| \PK(S(\pmb \alpha), S^m(\alpha^m))^T \beta +
\overline{P}(S^m(\alpha^m)) - f^m(Y^\mY) \|_2^2. 
  \end{aligned}
\end{equation*}
where we now use the notation $S^m(\alpha^m) = \Phi( \UK(\phi, \cdot)^T \alpha^m, Y^\mY)$ 
and $S(\pmb \alpha) = \cup_{m=1}^M S^m (\alpha^m)$ as before. 
In other words, we simply include a new penalty for the boundary conditions and we need 
take some care to impose the PDE on the interior collocation points $Y^\mY$ 
while the boundary conditions are imposed on the boundary points $Y^{\partial \mY}$. 
The general structure 
of the optimization problem remains the same and our LM algorithm naturally extends to 
this problem.

\editend
\subsection{Error analysis for 2-step KEqL}\label{app:theory:errorbound-2-step}

We present an error estimate for the 2-step KEqL. We follow the same notation as the main theorem, and we follow similar steps.

\begin{theorem}\label{prop:two-step-error-bound}
        Consider the problem in \eqref{u-optimal-recovery} and \eqref{P-optimal-recovery} 
    with $M,N \in \mathbb{N}$. Suppose 
    \Cref{assumption:domain,assumption:UU,assumption:PP} hold,
    $P, \overline{P} \in \PP$, and $u^m \in \UU$. Consider 
    pairs $\{ u^m, f^m\}_{m=1}^M$ 
    satisfying \eqref{generic-DE},
    and a bounded set $B \subset \mS$ with Lipschitz boundary. 
    Then there exist constants $\rho_0, \varrho_0(B) \in (0, 1)$ so that 
    whenever $\rho_{m,N} < \rho_0$ and $\varrho_{M,N}(B) < \varrho_0(B)$
    it holds that:
        \begin{equation*}
        \begin{aligned}
           \| \whu^m_{N} - u^m \|_{H^{\gamma'}(\mY)} 
            & \le C 
              \rho_{m,N}^{(\gamma - \gamma')} 
             \| u^m \|_{\UU}, 
        \end{aligned}
    \end{equation*}
    for $0\le \gamma' < \gamma$ and a 
    constant $C>0$ that depends only on $\mY$. Furthermore, 
    \begin{equation*}
    \begin{aligned}
    &\| P - \whP_{M,N} \|_{L^\infty(B)} 
     \le C \max\{ 1, 2^{\eta - \frac{Q + d}{2}} \} \times \\
     &\hspace{-1ex} \Bigg[  
    \varrho_{M,N}^{\eta - \frac{Q + d}{2}}  
    +  \max \left\{ \left(
    \editstart
    \rho_{M, N}^{\gamma - \gamma'}
    \editend
    \, \| u^m \|_\UU  \right)^{\eta - \frac{Q + d}{2}}, \,
     \editstart
     \rho_{M, N}^{\gamma - \gamma'}
     \editend
     \, \| u^m \|_\UU  \right\} \left(
    \| P \|_\PP + \| \overline{P} \|_\PP \right)
    \Bigg],
\end{aligned}
\end{equation*}
    for $d/2 + \text{order}(\P) < \gamma' < \gamma$ and a 
    constant $C >0$ that depends on $\mY$ and $B$.
\end{theorem}

\begin{proof}
{\it (a)} The first statement is a direct application of the 
Sobolev sampling inequality \Cref{prop:sobolev-sampling-inequality}(a) 
since the $\whu^m_N$ are simply the kernel interpolants of the $u^m$.

{\it (b)} 
Consider the definition of $P$ and observe that 
the interpolation constraint for $\whP_{M,N}$ can be written as 
\begin{equation*}
    \begin{aligned}
    (\whP_{M,N} + \overline{P})(s_k(\whu^m_N)) &= (P + \overline{P}) (s_k(u^m_N)) \\
    &= (P + \overline{P}) (s_k(\whu^m_N)) 
    + (P + \overline{P}) (s_k(u^m)) - (P + \overline{P}) (s_k(\whu^m_N))\\
    & = (P + \overline{P}) (s_k(\whu^m_N)) 
    + (P + \overline{P}) (s_k(u^m - \whu^m_N)), 
    \end{aligned}
\end{equation*}
where once again we used the fact that the $s_k(v)$ is linear in $v$. 
Assuming $s_k(u^m)$ and $s_k(u^m_N)$ belong to the set $B$ we 
can use the local Lipschitz property of $P$ and $\overline{P}$ to 
infer that 
\begin{equation*}
    | \whP_{M,N}(s_k(\whu^m_N)) - P(s_k(\whu^m_N)) | 
    \le C(B) \left( \| P \|_\PP + \|\overline{P}\|_\PP \right) \| s_k(u^m - \whu^m_N) \|_2
\end{equation*}
By the Sobolev sampling inequality \Cref{prop:sobolev-sampling-inequality}(a)
and \Cref{assumption:UU}(ii)
we can further bound 
\begin{equation}\label{disp-8}
    \| s_k(u^m - \whu^m_N) \|_2 \le 
    C \rho_{m, N}^{\gamma - \gamma'} \| u^m\|_\UU,
\end{equation}
where $\gamma > \gamma' > d/2 + \text{order}(\P)$.
This further leads to an error bound on the difference between 
$\whP_{M,N}$ and $P$ on the set of points $s_k(\whu^m_N)$:
\begin{equation*}
    | \whP_{M,N}(s_k(\whu^m_N)) - P(s_k(\whu^m_N)) | 
    \le C
    \rho_{m, N}^{\gamma - \gamma'}
    \left( \| P \|_\PP + \|\overline{P}\|_\PP \right) \| u^m \|_\UU 
\end{equation*}
We can now apply the noisy Sobolev sampling inequality \Cref{prop:sobolev-sampling-inequality}(b), viewing $\whP_{M, N}$ as the noisy interpolant 
on $P$ on the set of points $s_k(\whu^m_N)$, to get the error bound 
\begin{equation*}
    \| P - \whP_{M,N} \|_{L^\infty(B)} 
    \le C \Bigg[ \widehat{\varrho}_{M,N}^{\eta - \frac{Q + d}{2}} 
    \| P \|_\PP + 
    \editstart
    \rho_{M, N}^{\gamma - \gamma'}
    \editend
    \,  \| u^m\|_\UU  \left( \| P \|_\PP + \| \overline{P} \|_\PP \right)
    \Bigg],
\end{equation*}
where we introduced the notation $\widehat{\varrho}_{M,N} 
:= \sup_{s' \in B} \inf_{k, m} \| s' - s_k(\whu^m_N) \|$.
By \eqref{disp-8} we have that 
$\| s' - s_k(\whu^m_N) \| \le \| s' - s_k(u^m) \| + C \sup_m \left(  \rho_{m, N}^{\gamma - \gamma'} \| u^m \|_\UU \right)$ which in turn 
implies $\widehat{\varrho}_{M,N} \le \varrho_{M,N} + C \sup_m \left(  \rho_{m, N}^{\gamma - \gamma'} \| u^m \|_\UU \right)$. Substituting 
back into the bound above yields  
\begin{equation*}
    \begin{aligned}
        \| P - \whP_{M,N} \|_{L^\infty(B)} 
    &\le C \Bigg[ \left( \varrho_{M,N} + \sup_m  \left(  \rho_{m, N}^{\gamma - \gamma'} \| u^m \|_\UU \right) \right)^{\eta - \frac{Q + d}{2}} 
    \| P \|_\PP\\ &\qquad + 
    \editstart
    \rho_{M, N}^{\gamma - \gamma'} 
    \editend
    \, \| u^m\|_\UU  \left( \| P \|_\PP + \| \overline{P} \|_\PP \right)
    \Bigg].
    \end{aligned}
\end{equation*}
Using the inequality $(a + b)^p \le \max\{ 1, 2^{p-1}\} (a^p + b^p)$
for $p \in (0, +\infty]$ we can simplify this bound to take the desired 
form 
\begin{equation*}
\begin{aligned}
    \| P - \whP_{M,N} \|_{L^\infty(B)} 
     &\le C \max\{ 1, 2^{\eta - \frac{Q + d}{2}} \}\times \\ 
    &\hspace{-21ex}\Bigg[  
    \varrho_{M,N}^{\eta - \frac{Q + d}{2}}  
    +  \max \left\{ \left( 
    \editstart
    \rho_{M, N}^{\gamma - \gamma'}
    \editend
    \, \| u^m \|_\UU  \right)^{\eta - \frac{Q + d}{2}},\,
     \editstart
     \rho_{M, N}^{\gamma - \gamma'}
     \editend
     \, \| u^m \|_\UU  \right\} \left(
    \| P \|_\PP + \| \overline{P} \|_\PP \right)
    \Bigg].
\end{aligned}
\end{equation*}
{}
\end{proof}

\subsection{A convergence result for 1-step KEqL}\label{app:theory:convergence-proofs}

In this section we present the proof for the convergence of the 1-step KEqL which requires less stringent conditions such as Lipschitzness of elements in the RKHS $\PP$. Additionally, this proof is enlightening since it has potential to be extended to non-RKHS optimal recovery models. 

\begin{theorem}\label{prop:2-step-convergence}
Consider the problem \eqref{one-shot-optimal-recover}.
    Suppose \Cref{assumption:domain,assumption:UU} and \Cref{assumption:PP}(i-ii) hold. In addition, suppose that the continuous embeddings stated in \Cref{assumption:UU}(ii) and \Cref{assumption:PP}(ii) are compact.
    Finally, let $P \in \PP$, and assume that $\overline{P}$ is continuous. 
    Consider pairs $\{ u^m, f^m\}_{m=1}^\infty$
    satisfying \eqref{generic-DE} and a bounded set $B \subset \mS$ with
    Lipschitz boundary. Finally, suppose
    $\rho_{m,N} \to 0$ as $N \to \infty$  for all $m$
    and $\varrho_{M,N}(B) \to 0$  as $M,N \to \infty$.
    Then the following holds: 
    \begin{enumerate}[label=(\alph*)]
    \item
    Fix $M$.
    If $u^m \in \UU$
    then 
    $\lim_{N\to \infty} \whu^m_{M,N} = u^m$ pointwise in $\mY$ 
    and in $H^{\gamma'}(\mY)$ for all $\gamma' < \gamma$ and $m \in \{1, \dots, M\}$.
    \item  If $u^m \in \UU^2$ then    $\lim_{M \to \infty} \lim_{N \to \infty} \whP_{M,N} = P$ 
    pointwise in $B$ and in $H^{\eta'}(B)$ for all $\eta' < \eta$.
    
    \end{enumerate}

\end{theorem}
% \textcolor{red}{In general, convergence theorems are not compatible with sparse data regime. This theorem in particular is especially incompatible with whatever we observe numerically in experiments that only have one pair of solution and source term functions. Burgers and Duffing oscillator have 1.}
\begin{proof}
    Let us verify that our assumptions are sufficient for the problem to be 
    well defined. Since $\UU$ is an RKHS then the problem for 
    $\whu^m_{M,N}$ is readily well defined. Moreover, since $\UK$ is 
    continuous the space $\UU^2$ is well defined as in \Cref{app:theory:RKHS-review}. Moreover,
    \Cref{assumption:domain} along with \Cref{assumption:UU}(ii)
    and \Cref{prop:sobolev-embedding}
    ensure that $\Phi$ is continuous and bounded, in fact, $\Phi(\cdot, y)$ 
    is a bounded and linear operator on $\UU$; this is simply 
    the statement that $\P(u)$ is defined pointwise. Finally, 
    since $\PP$ is an RKHS and $\overline{P}$ is continuous we ensure that 
    the problem for $\whP_{M,N}$ is also well defined. 

    {\it Proof of (a):}
    Observe that the pair $(\pmb{u}, P)$ (notation: $\pmb{u} =(u^1, u^2, 
    \dots u^M)$) are feasible for \eqref{one-shot-optimal-recover} for all values of $M,N$. Then the optimality of $(\wh{\pmb u}_{M,N}, \whP_{M,N}) $ implies that 
    \begin{equation}\label{optimality-bound-1-step}
        \| \whP_{M,N} \|_{\PP}^2 + 
        \editstart
        \frac 1 M
        \editend
        \sum_{m=1}^M \| \whu_{M,N}^m\|_\UU^2 
        \le \| P \|_\PP^2 + 
        \editstart
        \frac{1}{M}
        \editend
        \sum_{m=1}^M \| u^m \|_\UU^2.
    \end{equation}
    Thus, for fixed $M$ we have that $\| \whu^m_{M,N} \|_\UU \le C(M)$ 
    for all $m$, i.e., $\{ \whu^m_{M,N} \}_{N=1}^\infty$ is bounded
    in $\UU$. Then \Cref{assumption:UU}(ii) implies that $\whu^m_{M,N}$
    has a convergent subsequence in $H^{\gamma}(\mY)$.
    On the other hand, since $\rho_{m,N} \to 
    0$ and thanks to the assumption that $u^m$ are continuous (since they belong to $\UU$) we infer that all accumulation points of $\whu^m_{M,N}$
    coincide with $u^m$. This implies that for fixed $M$ and 
    for any $m \in \{1, \dots, M\}$ we have $\lim_{N \to \infty} \whu^m_{M,N} 
    = u^m$ pointwise and in $H^{\gamma'}(\mY)$ for $\gamma' < \gamma$.

    {\it Proof of (b):}  Fix $M$ and for each $u^m$ define the optimal interpolant 
    \begin{equation}\label{u-bar-def-SM}
        \overline{u}^m_N:= \argmin_{v \in \UU} \| v\|_\UU 
        \quad \text{s.t.} \quad
        v(Y^m_N) = u^m(Y^m_N).
    \end{equation}
    By optimality, we have the bound 
    \begin{equation}\label{disp-0}
        \| \whu^m_{M,N} \|_\UU \ge \| \overline{u}^m_N \|_{\UU}.
    \end{equation}
    At the same time, the representer theorem 
    gives
    \begin{equation*}
        \overline{u}^m_N = \UK(Y^m_N, \cdot)^T \UK(Y^m_N, Y^m_N)^{-1} u^m(Y^m_N)    
    \end{equation*}
    and a direct calculation using 
    the reproducing property implies that 
    $\langle u^m - \overline{u}^m_N, \overline{u}^m_N \rangle_\UU =0$
    which in turn gives the identity 
    \begin{equation}\label{disp-1}
        \| u^m - \overline{u}^m_N \|_\UU^2 = \| u^m\|_{\UU}^2 - \| \overline{u}^m_N\|_\UU^2.
    \end{equation}
    At the same time, since we assumed $u^m \in \UU^2$ we also 
    have, using \Cref{lem:H-gamma-bound-SM}, 
    \begin{equation*}%\label{disp-2}
        \| u^m - \overline{u}^m_N \|_\UU^2 
        = \langle u^m - \overline{u}^m_N, u^m \rangle_{\UU} 
        \le \| u^m - \overline{u}^m_N \|_{L^2(\mY)} \| u^m \|_{\UU^2}.
    \end{equation*}
    But, an identical argument to part (a) shows that $\| u^m - \overline{u}^m_N \|_{L^2(\mY)} \to 0$ as $N \to \infty$ and so we infer 
    that $ \lim_{N \to \infty} \| u^m - \overline{u}^m_N \|_\UU =0.$
    This fact, together with \eqref{disp-0} and \eqref{disp-1} implies that 
    \begin{equation*} %\label{lim-sup-disp}
        \limsup_{N \to \infty} \| \whu^m_N \|_\UU \ge \| u^m \|_\UU.
    \end{equation*}
    Then it follows from \eqref{optimality-bound-1-step} that 
    \begin{equation*}
        \limsup_{N\to \infty} \| \whP_{M,N} \|_\PP \le \| P \|_\PP.
    \end{equation*}
    This implies that the sequence $\{ \whP_{M,N}\}_{N=1}^\infty$ is 
    bounded and so has a convergent subsequence due to 
    \Cref{assumption:PP}(ii).
    From part (a) we also have that 
    $\lim_{N \to \infty} \Phi(\whu^m_{M,N}, y) = \Phi(u^m, y)$
    for all $y \in \mY$ and $m \in \{1, \dots, M\}$.
    {}
\end{proof}

% We  now summarize our proof of quantitative error estimates
% that were summarized in Theorem~3. We will split the proof into two 
% propositions, giving rates for the 1-step and 2-step methods separately. 

\subsection{Error analysis for equation-mediated operator learning}
\label{sm:operator-learning-analysis}

\editstart
We begin by giving an elementary lemma that allows us to 
control the error between the solution of a true PDE and 
that of an approximate equation under sufficient 
regularity assumptions on the forward and inverse differential 
operators.
For this lemma,  we will consider $\P: \UU \to \FF$ and its inverse 
$\P^{-1}: \FF \to \UU$ for generic Banach spaces $\UU, \FF$.
\editend

% \begin{lemma}\label{lem:operator-learning-error}
% Consider $A \subset \UU$ and $B \subset \FF$ for which it holds that: 
% \begin{enumerate}
%     \item $\P^{-1}$ is locally Lipschitz on $B$, i.e., for any 
%     pair $f, f' \in B \subset \FF$, there exists a constant $L(B)$
%     so that 
%     \begin{equation*}
%         \| \P^{-1}(f) - \P^{-1}(f') \|_\UU \le L(B) \| f - f' \|_\FF.
%     \end{equation*}

%     \item $\widehat{\P}$ approximates  $\P$ on $A$ in the sense that  
%     \begin{equation*}
%         \| \widehat{\P}(u) - \P(u) \|_\FF \le  \epsilon(A) \| u\|_\UU,
%     \end{equation*}
%     for some constant $\epsilon(A) >0$.
% \end{enumerate}
% Fix an $f \in B$ for which $\P^{-1}(f) \in A$, and let $\whu \in A$
%  be any element that 
%     solves $\widehat{\P}(\whu) = f$. Then we have the error bound 
%     \begin{equation*}
%         \| \whu - \P^{-1}(f) \|_\UU \le L(B)\, \epsilon(A)\, \| \whu\|_\UU.
%     \end{equation*}
% \end{lemma}

% \begin{proof}
%     Write $u = \P^{-1}(f)$. 
%     Since $\whu$ solves the approximate problem we can write 
%     %\begin{equation*}
%         $\widehat\P (\whu) + \P(\whu) - \P(\whu) = f = \P(u).$
%     %\end{equation*}
%     Combining this with the Lipschitz assumption on $\P^{-1}$
%     and the assumed error bound for $\widehat\P$ 
%     gives the chain of inequalities 
%     \begin{equation*}
%    \| \whu - u \|_\UU \le 
%    L(B)\| \P(\whu) - \P(u) \|_\FF = L(B) \| \widehat \P (\whu) - \P(\whu) \|_\FF
%    \le L(B)\, \epsilon(A)\, \| \whu \|_\UU.
%     \end{equation*}
%     {}
% \end{proof}

\editstart
\begin{lemma}\label{lem:operator-learning-error}
Consider a well defined and invertible differential operator $\P: \UU \to \FF$ 
along with an approximation $\widehat{\P}$.
Fix a set $F \subset \FF$ and consider its pre-image 
$\wh \P^{-1}(F) := \{ \whu \in \UU \mid \wh \P(\whu) = f, \: f \in F \}$.
Suppose there exists $A \supseteq \wh \P^{-1}(F) \cup \P^{-1}(F)$ so that:

\begin{enumerate}
    \item $\P^{-1}$ is uniformly Lipschitz on $\P(A)$, i.e., 
    there exists a constant $L(A) \ge 0$
   so that for any 
    pair $f, f' \in \P(A)$ we have
    \begin{equation*}
         {\| \P^{-1}(f) - \P^{-1}(f') \|_\UU} \le  L(A) {\| f - f' \|_\FF}.
    \end{equation*}

    \item $\widehat{\P}$ approximates  $\P$ on $A$ in the sense that  
    there exists $\epsilon(A) \ge 0$ so that
    \begin{equation*}
      \sup_{u \in A}  \| \widehat{\P}(u) - \P(u) \|_\FF \le  \epsilon(A).
    \end{equation*}
\end{enumerate}
 Then it holds that 
\begin{equation*}
   \sup_{f \in F} 
   \sup_{\whu \in \wh{\P}^{-1}(f)} 
   \| \whu - \P^{-1}(f) \|_{\UU} 
    \le L(A) \epsilon(A).
\end{equation*}
% Fix an $f \in B$ for which $\P^{-1}(f) \in A$, and let $\whu \in A$
%  be any element that 
%     solves $\widehat{\P}(\whu) = f$. Then we have the error bound 
%     \begin{equation*}
%         \| \whu - \P^{-1}(f) \|_\UU \le L(B)\, \epsilon(A)\, \| \whu\|_\UU.
%     \end{equation*}
\end{lemma}

\begin{proof}
    Fix an arbitrary $f \in F$ and write $u = \P^{-1}(f)$.
    Consider any $\whu \in \wh \P^{-1}(f)$. Since this $\whu$ 
    solves the approximate problem 
    $\wh \P (\whu) = f =  \P(u)$ we can write 
    %\begin{equation*}
        $\widehat\P (\whu) + \P(\whu) - \P(\whu) = \P(u).$
    %\end{equation*}
    Noting that by assumption $\whu \in A$ we can utilize properties 1 and 2 in that order to 
    obtain the chain of inequalities 
    \begin{equation*}
   \| \whu - u \|_\UU \le 
   L(A)\| \P(\whu) - \P(u) \|_\FF = L(A) \| \widehat \P (\whu) - \P(\whu) \|_\FF
   \le L(A)\, \epsilon(A).
    \end{equation*}
    Since the above bound holds for any $\whu \in \wh \P^{-1}(f)$ and $f \in F$ was arbitrary
    we obtain the desired result.
    {}
\end{proof}
\editend

Let us now apply this lemma to obtain an error bound for the learned equation 
$\widehat{\P}_{M,N}(u) = (\widehat{P}_{M,N} + \overline{P}) \circ \Phi(u, \cdot) = f$
via 1-step KEqL; a similar result can be shown for the 2-step method by 
a straightforward modification of the proof using our error bound 
for 2-step KEqL in \Cref{prop:two-step-error-bound}.

\editstart
\begin{proof}
[Proof of {\Cref{prop:operator-learning-1-step-KEqL}}]
The result follows from a direct application of \Cref{lem:operator-learning-error} 
once we verify condition 2 since we readily assumed that $\P$ satisfies the first condition.
Since we also assumed that \Cref{thm:1-step-error-estimate} holds and 
thanks to the assumption that $\Phi(u,y) \in B$ for all $u \in A$ and $y \in Y$ 
we have the bound 
    \begin{equation*}
    \begin{aligned}
       & \| \widehat\P_{M,N}(u) - \P(u) \|_{C(\mY)} = 
       \sup_{y \in \mY}
       \left| \whP_{M,N} \circ \Phi(u, y) - P \circ \Phi(u, y) \right|  \\ 
        & \qquad \le C \left( \varrho_{M, N}^{\eta - \frac{Q + d}{2}}
    + {\rho_{M, N}^{\gamma - \gamma'} }\right) 
    \left( \| P \|_\PP^2 + \| \overline{P} \|_\PP^2 + 
     {\frac 1 M}\sum_{m=1}^M \| u^m \|_{\UU^2}^2 \right)^{1/2}.
    \end{aligned}
    \end{equation*}
    The result follows by a straightforward application of \Cref{lem:operator-learning-error}.
    {}
\end{proof}
\editend

\bibliographystyle{amsplain}
\bibliography{SIREV_REFs}

\end{document}